\documentclass[conference]{IEEEtran}
\usepackage{times}
\usepackage[numbers]{natbib}
\usepackage{multicol}
\usepackage{multirow}
\usepackage[bookmarks=true]{hyperref}
\usepackage{graphicx}	
\usepackage{amsmath}
\usepackage{mathtools}
\usepackage{paralist}
\usepackage{tabularx}	
\usepackage{algpseudocode}
\usepackage{color}
\usepackage{fancyhdr}
\usepackage{bm,bbm}
\usepackage{tikz}
\usepackage{amsthm}
\usepackage{epstopdf}
\usepackage{amsfonts}
\usepackage{dsfont}
\usepackage{mathrsfs}
\usepackage{amssymb}
\usepackage{enumitem}
\usepackage{mdframed}
\usepackage{algorithm}
\usepackage{epsfig}
\usepackage[english]{babel}
\usepackage[all]{xy}
\usepackage[latin1]{inputenc}
\usepackage{tikz}
\usepackage[T1]{fontenc}
\usepackage{caption}
\usepackage[labelformat=simple]{subcaption}

\newtheorem{theorem}{Theorem}
\newtheorem{definition}{Definition}

\newtheorem{lemma}{Lemma}

\newtheorem{remark}{Remark}
\newtheorem{proposition}{Proposition}

\DeclareMathOperator{\col}{col}
\newcommand{\twoto}{%
    \mathrel{\substack{\to \\ \to}}%
}

\newcommand{\SEthree}{%
    \mathrel{\mathrm{SE}(3) }%
}
\newcommand{\cent}{%
    \mathrel{\mathrm{cent} }
    }

\def\tr{\mathop{\rm tr}}
\title{Provably Guaranteed Polytopic Uncertainty Quantification for SLAM}
\author{Guangyang Zeng$^{1,2}$, Yulong Gao$^3$, Yuan Shen$^1$, Lingpeng Chen$^1$, Haoying Li$^1$, Guodong Shi$^4$, and Junfeng Wu$^{1,2}$ \\
  $^1$School of Data Science, The Chinese University of Hong Kong, Shenzhen \\ $^2$School of Artificial Intelligence, The Chinese University of Hong Kong, Shenzhen \\ $^3$Department of Electrical and Electronic Engineering, Imperial College London  \\ $^4$School of Aerospace, Mechanical and Mechatronic Engineering, The University of Sydney
}

\begin{document}

\twocolumn[{
    \renewcommand\twocolumn[1][]{#1}
    \maketitle 

    \begin{center}
        \includegraphics[width=\textwidth]{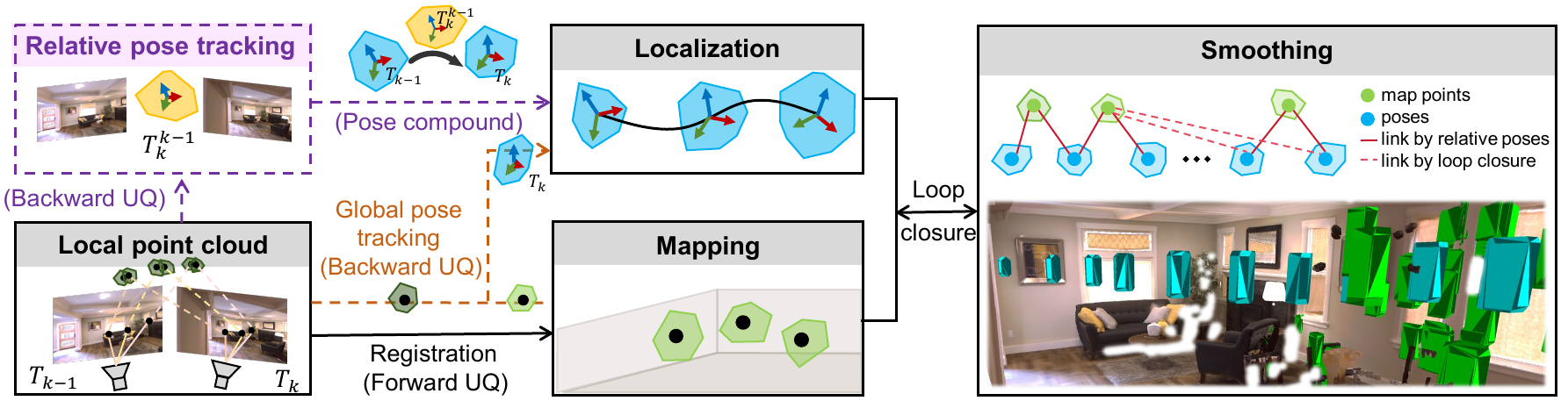} 

        \captionof{figure}{Our guaranteed polytopic SLAM uncertainty quantification (UQ) framework. For global robot localization, there are two alternatives: (i) relative pose tracking followed by pose compound (marked as dashed purple lines), and (ii) global pose tracking (marked as dashed orange lines).}
        \label{fig:guaranteed_SLAM}
    \end{center}
}]

% avoiding spaces at the end of the author lines is not a problem with
% conference papers because we don't use \thanks or \IEEEmembership

%Homer Simpson\authorrefmark{2},
%Montgomery Scott\authorrefmark{3} and
%Eldon Tyrell\authorrefmark{4}}
%Telephone: (800) 555--1212, Fax: (888) 555--1212}
%\authorblockA{\authorrefmark{4}Tyrell Inc., 123 Replicant Street, Los Angeles, California 90210--4321}}

% \maketitle

%   \begin{figure*}[!htbp]
%     \centering
%     \includegraphics[width=0.98\linewidth]{figures/USLAM.pdf}
%     \caption{Our guaranteed polytopic SLAM UQ framework. For global robot localization, there are two alternatives: (i) relative pose tracking followed by pose compound (marked as dashed purple lines), and (ii) global pose tracking (marked as dashed orange lines).}
%     \label{fig:guaranteed_SLAM}
% \end{figure*}

\begin{abstract}
In safety-critical robotics applications, guaranteed and practical uncertainty quantification (UQ) in perception is vital. Many existing works either offer no formal containment guarantee, rely on restrictive modeling assumptions, or focus only on pose estimation rather than a complete SLAM pipeline. This paper presents provably guaranteed UQ algorithms for 3D-3D landmark-based SLAM. The algorithms consist of three basic UQ modules: forward UQ for mapping, backward UQ for pose tracking, and pose compound. Each module produces a certified uncertainty set; when the input uncertainty bounds are deterministic, the output sets inherit deterministic guarantees, i.e., they provably contain the true poses and landmarks. Specifically, we use polytopes to represent uncertainty sets, enabling tractable computations and a unified treatment of pose uncertainty. To enhance algorithms' practical usability, we incorporate conformal prediction to calibrate measurement uncertainty from data with prescribed probability. Simulations and experiments demonstrate that the proposed algorithms provide both strong theoretical guarantees and practical usability. The code is open-sourced at \url{https://github.com/LIAS-CUHKSZ/Polytopic-SLAM-Uncertainty-Quantification}.
\end{abstract}

% Keywords: UQ, Set-membership estimation, SLAM, Polytopic uncertainty, Conformal prediction

\IEEEpeerreviewmaketitle

\section{Introduction} \label{sec:introduction}
Simultaneous localization and mapping (SLAM) has found widespread applications in robotics. In recent years, substantial progress has been made, with research primarily improving accuracy~\cite{cvivsic2022soft2}, robustness~\cite{yang2020teaser}, and efficiency~\cite{han2025building}, among other aspects.  
Most SLAM algorithms calculate optimal point estimates for the robot state and the environment map. However, for \emph{safety-critical} applications such as autonomous driving and unmanned aerial vehicle navigation, point estimates alone are insufficient. Reliable \emph{uncertainty quantification} (UQ) is  needed so that downstream planning and control can explicitly account for perception risk and make safety-aware decisions.

\textbf{Probabilistic UQ.} This line of work typically assumes that the measurement uncertainty follows a distribution, most commonly Gaussian. Under such assumptions, one often computes a point estimate via maximum likelihood or maximum a posteriori inference, and then approximates the associated uncertainty using the inverse Fisher information matrix (FIM) or Hessian evaluated at the optimal solution to obtain a covariance~\cite{engel2014lsd,rosen2019se,mangelson2020characterizing}.
Despite its popularity, this approach has several limitations. (i) Measurement uncertainty often deviates from a Gaussian distribution and may vary in dynamic scenes~\cite{gao2024closure,agrawal2024gatekeeper}. (ii) Using FIM or Hessian again assumes a Gaussian distribution and is prone to underestimating the undergoing uncertainty~\cite{szeliski2022computer}. (iii) The resulting covariance ellipsoid provides no formal guarantee of containing the true parameter; it is an approximation rather than a certified confidence region.

\textbf{Deterministic UQ.} This line of work is commonly referred to as \emph{set-membership estimation} (SME). SME assumes measurement uncertainty is modeled by hard bounds rather than probability distributions and computes the set of all parameters consistent with measurements.
SME offers two advantages: (i) it requires only knowledge of the uncertainty support, without  committing to a particular distribution; and (ii) the resulting parameter uncertainty set comes with a \emph{provable guarantee} of containing the true parameter. 
SME has been widely studied in control theory~\cite{wang2018set,de2024set,wang2018zonotopic} and in signal processing~\cite{belforte1990parameter,jaulin1993set}. In recent years, it has gained traction in robotics perception and estimation problems, with many works favoring interval-based formulations~\cite{mustafa2018guaranteed,voges2021interval,ehambram2022interval}.
More recently, a stream of work has introduced \emph{conformal prediction} (CP)~\cite{lindemann2024formal,angelopoulos2023conformal} to construct statistically valid measurement uncertainty bounds using a calibration data set, and then to perform pose estimation. 
Specifically, CP is used to calibrate uncertainty sets for the residuals $g(T,x_i,y_i):=(y_i-f(T,x_i) )\in \mathcal G_i, i=1,\ldots,N$, where $T \in {\rm SE}(3)$ denotes the unknown pose, $x_i$ is the regressor, and $y_i$ is the
model output.
The resulting feasible pose set is defined as
\begin{equation} \label{eqn:parameter_set}
    %\mathcal S=
    \{T \in {\rm SE}(3) \mid g(T,x_i,y_i) \in \mathcal G_i,i=1,\ldots,N\}.
\end{equation}
Although the pose uncertainty set admits the concise description in~\eqref{eqn:parameter_set}, it is  difficult to calculate, analyze or manipulate, since $f$ may be non-convex, and the pose $T$ is subject to non-convex constraints.  
Tractable approximations are adopted, including outer approximations~\cite{tang2024uncertainty,yang2023object,shaikewitz2025uncertainty} and inner approximations~\cite{gao2024closure}. 
However, these approaches are limited to pose estimation and do not yet extend to a complete SLAM pipeline.

\textbf{Contributions.}
 In our work, we develop a novel framework that bridges distribution-free calibration with SME to enable tractable, \emph{provably guaranteed} UQ for 3D-3D landmark-based SLAM. Unlike heuristic approaches, our method guarantees that estimated sets contain the true parameter with a prescribed probability. A key novelty is the framework's duality and modularity: it seamlessly transitions between probabilistic confidence sets (via CP) and strict deterministic enclosures by simply interchanging the bounding module.

Technically, as illustrated in Figure~\ref{fig:guaranteed_SLAM}, we realize this via three primitives: forward UQ (mapping), backward UQ (pose tracking), and pose compound. 
Each primitive presents distinct challenges, and the associated uncertainties entangle mutually due to the propagation between landmarks and poses.
We resolve them with tractable algorithms and then integrate them for full SLAM UQ.
In particular, by implementing these primitives using polytopic sets, we leverage mature tools to ensure computational tractability via linear constraints, while enabling a unified representation of rotation and translation uncertainty. 
Simulations and experiments demonstrate that our algorithms contain the true parameters with certified guarantees, achieve tighter uncertainty sets than recent guaranteed pose UQ methods, and apply to complete SLAM pipelines.

{\bf Notations:} Let $\mathcal X$ and $\mathcal Y$ be two sets, $\mathcal X+\mathcal Y$ and $\mathcal X \cdot \mathcal Y$ denote their Minkowski sum and Minkowski product, respectively. The closed Euclidean ball centered at point $p_0 \in \mathbb R^n$ with radius $r\geq 0$ is denoted as $\mathcal B(p_0,r)$. For a set $\mathcal X$  in a Euclidean space, its convex hull is denoted by ${\rm conv}(\mathcal X)$.  
Let $2^{\mathcal X}$ be the collection of all subsets of $\mathcal X$. A set-valued function 
$f: \mathcal X \to 2^{\mathcal Y}$ will be rewritten as 
$f: \mathcal  X\twoto Y$ to emphasize its set-valuedness. 
The image of $\mathcal C\subset X$ under $f$ is defined as 
$f(\mathcal C)=\bigcup_{x\in \mathcal C}f(x)$. The preimage of $\mathcal C$
is defined as 
$f^{-1}(\mathcal C)=\{x: f(x)\cap \mathcal C\not = \emptyset\}$.
For two vectors $ x$ and $ y$, $\langle{ x}, { y} \rangle$ represents their inner product, $\col(x,y)=[x^\top,y^\top]^\top$ for short, and $x \leq y$ means the element-wise inequality. 

\section{Preliminaries} \label{sec:preliminary}

\subsection{Rigid Transformation} \label{subsec:rigid_transform}
In 3D Euclidean space, (proper) rigid transformations include rotations and translations. 
They preserve the distance between any two points of and the handedness (orientation) of rigid bodies.
A translation is depicted by a vector in $ \mathbb R^3$. A rotation can be characterized by a rotation matrix. Specifically, rotation matrices belong to the special orthogonal group
\begin{equation*}
{\rm SO}(3)=\{{ R}\in \mathbb R^{3 \times 3} \mid { R}^\top { R}={ I}_3, {\rm det}({ R})=1\}.
\end{equation*}
The group operation is the usual matrix multiplication, and the inverse is the matrix transpose. The group ${\rm SO}(3)$ forms a smooth manifold, and the tangent space at the identity is the Lie algebra, denoted as $\mathfrak{so}(3)$, which is comprised of all $3 \times 3$ skew-symmetric matrices. Every $3 \times 3$ skew-symmetric matrix can be generated from a vector---say ${ s} \in \mathbb R^3$---via the \emph{hat} operator ${ s}^\wedge \in \mathfrak {so} (3)$,  which maps \(s\) into  \(\mathfrak{so}(3)\).

The exponential map ${\rm exp}:\mathfrak{so}(3) \rightarrow {\rm SO}(3)$ coincides with standard matrix exponential. Specifically, Rodrigues' rotation formula reads that
\begin{equation*}
    {\rm exp}\left({ s}^\wedge \right)= { I}_3 + \frac{\sin(\|{ s}\|)}{\|{ s}\|} { s}^\wedge + \frac{1-\cos(\|{ s}\|)}{\|{ s}\|^2} \left({ s}^\wedge \right)^2.
\end{equation*}
We let ${\rm Exp}:\mathbb R^3 \rightarrow {\rm SO}(3)$ 
be the composition of the hat operator followed by $\exp$, that is, ${\rm Exp}({ s})={\rm exp}({ s}^\wedge)$. The inverse operation of ${\rm Exp}$ is denoted as ${\rm Log}: {\rm SO}(3) \rightarrow \{s \in \mathbb R^3 \mid \|s\| \in [0,\pi)\}$.
The geodesic distance between two rotations is defined as
\begin{equation*}
    {\rm dis}({ R}_1,{ R}_2)=\cos^{-1} \left( \frac{\tr({ R}_1{ R}_2^\top)-1}{2}\right) \in [0,\pi].
\end{equation*}

\subsection{Conformal Prediction} \label{subsec:conformal_prediction}
Conformal prediction (CP) is a lightweight statistical tool for uncertainty quantification~\cite{shafer2008tutorial,lindemann2024formal}. Unlike standard machine learning models, which typically provide a single point prediction, it outputs a prediction set that is guaranteed to contain the true value with a user-specified probability. Specifically, let $s_0, s_1, \ldots, s_L$ be $L+1$ \emph{exchangeable} random variables\footnote{The random variables $s_0,  \ldots, s_L$ are said to be exchangeable if the joint distribution of $s_0 \ldots, s_L$ is the same to that of $s_{\sigma(0)},  \ldots, s_{\sigma(L)}$ for any permutation $\sigma$ on indices.}. It is straightforward that independent and identically distributed random variables are exchangeable. The goal of CP is to estimate an upper bound for $s_0$ using the calibration set ${s_1, \ldots, s_L}$, ensuring coverage with probability $1-\delta$. Based on the quantile function, the upper bound $C: \mathbb R^L \rightarrow \mathbb R$ is given by $C(s_1,\ldots,s_L)={\rm Quantile}_{1-\delta}(s_1,\ldots,s_L,\infty)$,
which is the $(1-\delta)$-quantile over the empirical distribution of $s_1,\ldots,s_L,\infty$. Let $s_{L+1}=\infty$, and without loss of generality, assume $s_1,\ldots,s_L$ are sorted in non-decreasing order. Then, we have
\begin{equation} \label{eqn:conformal_value}
    C(s_1,\ldots,s_L)=s
    _p, ~p=\lceil(1-\delta)(L+1) \rceil,
\end{equation}
where $\lceil \cdot \rceil$ is the ceiling function.
\begin{lemma}[{\bf Marginal coverage~\cite{shafer2008tutorial,lindemann2024formal}}] \label{lemma:CP}
    With the choice of $C$ in~\eqref{eqn:conformal_value}, it holds that 
    \begin{equation} \label{eqn:CP_lemma}
        {\rm Prob}(s_0 \leq C(s_1,\ldots,s_L)) \geq 1-\delta.
    \end{equation}
\end{lemma}
The variable $s
_l,l \in \{0,\ldots,L\}$ is typically called the \emph{nonconformity score}. It quantifies the deviation between a predictor's output $\hat y_l$ and the true value $y_l^o$, e.g., for vector-valued predictions, it may be defined as the prediction error $s_l=\|\hat y_l-y_l^o\|$. A large nonconformity score indicates a poor prediction. For a comprehensive review of CP, see~\cite{angelopoulos2023conformal,fontana2023conformal}.

\subsection{Polytopes} \label{subsec:polytope}
Given $A=[a_1,\ldots,a_M]^\top$ \footnote{The rows of $A$ are assumed to be independent for minimal representation.} and $b=[b_1,\ldots,b_M]^\top$,  define the set 
\begin{equation}\label{eqn:polytope_H}
 \mathcal P(A,b):=\{x  \in \mathbb{R}^n \mid A x \leq b\}.
 \end{equation}If $\mathcal P(A,b)$ is bounded, we call it
 a polytope. The one in form of~\eqref{eqn:polytope_H} is 
 an $H$-polytope, due to  its
half-space representation. Without loss of any generalization, we assume that each facet $a_m$ is a unit normal vector. 

%A polyhedron is a set of the intersection of finite half-spaces:
%\begin{equation} \label{eqn:H_representation}
 %   \mathcal P=\{x\in \mathbb{R}^n \mid a_m^\top x \leq b_m,m=1,\ldots,M\}.
%\end{equation}
%A bounded polyhedron is called a \emph{polytope}. We call the polytope  in Eq.~\eqref{eqn:H_representation} the $H$-polytope,  due to its
%half-space representation. It can be written as the compact form $$\mathcal P(A,b)=\{x  \in \mathbb{R}^n \mid A x \leq b\}$$ with $A=[a_1,\ldots,a_M]^\top$ and $b=[b_1,\ldots,b_M]^\top$. 
%Without loss of any generalization, we assume that each $a_m$ is a facet normal vector. 

According to the Minkowski-Weyl's Theorem, any $H$-polytope $\mathcal{P}(A,b)$ has an equivalent $V$-polytope:  
\begin{align*}
   \mathcal P(\{v_1,\ldots,v_N\})&:={\rm conv}(v_1,\ldots,v_N)\\&:=\{x \mid x=\sum_{i=1}^N \lambda_i v_i, \lambda_i\geq 0, \nonumber \sum_{i=1}^N\lambda_i=1 \} ,
\end{align*}
 where $v_i\in\mathbb{R}^n,i=1,\ldots,N$, are the vertices of the polytope. Both the vertex
enumeration  (from $H$-rep. to $V$-rep.) and facet enumeration (from $V$-rep. to $H$-rep.) can be performed using the libraries \texttt{cddlib} or \texttt{lrslib}~\cite{avis2002canonical}. Next, we  introduce some definitions and polytope operations.

%Another common representation of a polytope is the V-representation, which expresses $\mathcal P={\rm conv}(v_1,\ldots,v_N)$, where $v_i,i=1,\ldots,N$ are the vertices of the polytope. We can switch between the H-representation and the V-representation via vertex/facet enumeration, e.g., using the libraries \texttt{cddlib} or \texttt{lrslib}~\cite{avis2002canonical}. In the worst case, the conversion can take time exponential in the polytope dimension.  
%Next, we will introduce some definitions and polytope operations.

%With a little abuse of notation, we often denote a polytope directly as $Ax\leq b$ for convenience throughout this paper. Each $a_m$ is a facet normal vector; when $A$ is fixed, we call the polytope has a fixed template, i.e., fixed facet directions.  

%Another common representation of a polytope is the V-representation, which expresses $\mathcal P={\rm conv}(v_1,\ldots,v_N)$, where $v_i,i=1,\ldots,N$ are the vertices of the polytope. We can switch between the H-representation and the V-representation via vertex/facet enumeration, e.g., using the libraries \texttt{cddlib} or \texttt{lrslib}~\cite{avis2002canonical}. In the worst case, the conversion can take time exponential in the polytope dimension.   Next, we will introduce some definitions and polytope operations.

\begin{definition}[Polytope diameter]\label{def:polytope_diameter}
For a polytope $\mathcal P\subset \mathbb R^n$, its diameter $d(\mathcal P)$ is defined as 
    \begin{equation*}
        d(\mathcal P)= \max_{{ p}_1, { p}_2 \in \mathcal P} \|{ p}_1-{ p}_2\|.
    \end{equation*}
    \end{definition}
\begin{definition}[Chebyshev ball and Chebyshev center]\label{def:Chebyshev_ball_center}
For a polytope $\mathcal P$, the unique ball of minimal radius that encloses the entire $\mathcal P$ is called its Chebyshev ball. Its radius, known as the Chebyshev radius, is given by, $$ \min _{{x_0},r}\{r:\left\|{x_0}-x\right\|\leq r,\forall x\in \mathcal P\}.$$
The Chebyshev center is denoted by $\cent(\mathcal P)$ and the Chebyshev radius by 
$r(\mathcal P)$.
\end{definition}

These two definitions are still valid for a general set. In the case of polytopes, the diameter is attained by a pair of vertices. Hence, it can be computed by checking all vertex pairs, which takes $O(N^2)$ time in the number $N$ of vertices:
    \begin{equation} \label{eqn:cal_polytope_diameter}
        d(\mathcal P)= \max_{{ p}_1, { p}_2 \in \mathcal V} \|{ p}_1-{ p}_2\|,
    \end{equation}
    where $\mathcal V=\{v_i\}_{i=1}^N$ is the vertex set of $\mathcal P$.
The Chebyshev center and radius of $\mathcal P$ can be found by solving a convex optimization problem:
\begin{subequations} \label{eqn:cal_chebyshev_ball}
\begin{align}
        \min _{{x_0},r} & ~~r \\
        {\rm s.t.} & ~~\left\|{x_0}-v_i\right\|\leq r, i=1,\ldots,N.
\end{align}
\end{subequations}
The obtained optimal $x_0^*$ is the Chebyshev center, and $r^*$ is the Chebyshev radius. Solving~\eqref{eqn:cal_chebyshev_ball} using interior-point methods leads to a time complexity of $O(N^{1.5})$.
% According to \jf{Jung's theorem~\cite{xxx}, 
% $$
% \frac{d(\mathcal P)}{2}\leq r(\mathcal P)\leq d(\mathcal P)\sqrt{\frac{n}{2(n+1)}}.
% $$
% }

\begin{definition}
Given three polytopes $\mathcal{P}, \mathcal{P}_1, \mathcal{P}_2\subset \mathbb{R}^n$, a linear map $H\in \mathbb{R}^{m\times n}$, and a vector $q \in \mathbb{R}^m$, define following set operations: 
    \begin{itemize}
        \item (Affine Map) $H \mathcal{P} + q :=\{Hx + q \mid x\in \mathcal{P}\}$;
         \item (Minkowski Sum) $\mathcal{P}_1 + \mathcal{P}_2 :=\{x+y \mid x\in \mathcal{P}_1, y\in \mathcal{P}_2\}$;
         \item (Intersection) $\mathcal{P}_1 \cap\mathcal{P}_2 :=\{x\mid x\in  \mathcal{P}_1 \ {\rm and} \ x\in \mathcal{P}_2 \}$;
         \item (Set projection to subspace $\mathbb R^d$) $\pi_{\mathbb R^d}(\mathcal P)=\{x\in\mathbb R^d\mid \exists y\in \mathbb R^{n-d}, \col(x,y)\in \mathcal P\}$.    
    \end{itemize}
\end{definition}
These set operations can be defined for arbitrary sets, however, within the context of polytopes, they are closed, i.e., the result is again a polytope. The set projection can be performed by the Fourier-Motzkin elimination algorithm~\cite{zhen2018adjustable}.

%\begin{operation}[Polytope projection] 
%   For a polytope $A[x_1^\top,x_2^\top]^\top \leq b$, we can project it to the subspaces of $x_1$ and $x_2$, obtaining $A_1 x_1 \leq b_1$ and $A_2 x_2 \leq b_2$. This can be completed based on the Fourier-Motzkin elimination algorithm~\cite{zhen2018adjustable}.
%\end{operation}

%\begin{operation}[Minkowski sum of polytopes] 
 %   For two polytopes $\mathcal P_1$ and $\mathcal P_2$, their Minkowski sum $\mathcal P_1+\mathcal P_2=\{p_1+p_2 \mid p_1 \in \mathcal P_1, p_2 \in \mathcal P_2\}$ is also a polytope~\cite{schneider2013convex}. Denote the vertex sets of $\mathcal P_1$ and $\mathcal P_2$ as $\mathcal V_{\mathcal P_1}$ and $\mathcal V_{\mathcal P_2}$, respectively. The Minkowski sum can be obtained by $\mathcal P_1+\mathcal P_2={\rm conv}\left(\{{ v}_1+{ v}_2 \mid { v}_1 \in \mathcal V_{\mathcal P_1}, { v}_2 \in \mathcal V_{\mathcal P_2}\}\right)$.
%\end{operation}

\section{Basic Uncertainty Quantifications} \label{sec:basic_uncertainty_propaga}
We study three fundamental UQ primitives
that form the core building blocks for Sec.IV.

Consider a rigid transformation ${ T}=\begin{bmatrix}
    { R}~  t \\
    { 0}_3^\top~ 1
\end{bmatrix}$ of ${\rm SE}(3)$, which acts on points in $\mathbb R^3$ via a left group action $\phi:{\rm SE}(3)\times \mathbb R^3\to \mathbb R^3$.
We denote $x(T):={\rm col}({\rm vec}(R),t) \in \mathbb R^{12}$ as the vectorization of $T$. In robotics, if $T$ maps a local frame to a reference frame and $p$ is a point in the local frame, then $\phi(T,p)=Rp+t$ gives the point in the reference frame.
Define the map  $\phi_p : \mathrm{SE}(3) \to \mathbb{R}^3, \quad T \mapsto \phi(T,p)$.
The basic problems to be discussed involve the forward and backward transfer of uncertainty between the argument of \(\phi_p\) (elements of \(\mathrm{SE}(3)\)) and its image in \(\mathbb{R}^3\).
The forward UQ aims to propagate uncertainty from $\mathrm{SE}(3)$ to $\mathbb R^3$, whereas the backward UQ infers the uncertainty on $\mathrm{SE}(3)$ from that on $\mathbb R^3$. We will also investigate the fundamental UQ problem for pose compound. 
Before proceeding, we introduce the following definition. 
For a given set $\mathcal S$ of points, define a set-valued function $\phi_{\mathcal S}: \mathrm{SE}(3) \substack{\to \\ \to}
{\mathbb R^3}$  by 
\begin{equation}\label{eqn:def_phi_S}
\phi_{\mathcal S}(T)=
\{\phi_{p}(T): p\in\mathcal S\}.
\end{equation}
%For a subset $\mathcal T\in \mathrm{SE}(3) $, the image under $\phi_{\mathcal S}$ follows the standard definition for set-valued functions: $$\phi_{\mathcal S}(\mathcal T)=\bigcup_{T\in\mathcal T}\phi_{\mathcal S}(T).$$

\subsection{Forward Uncertainty Quantification } \label{subsec:pose_to_point}
The forward UQ problem is interpreted as follows: given a fixed, known and bounded set $\mathcal S$, for a set $\mathcal T \subset {\rm SE}(3)$, find the image $\phi_{\mathcal S}(\mathcal T)$. 
For tractability and computational issues, we restrict $\mathcal S$ to a polytope $\mathcal P(A_1,b_1)=\{p\in \mathbb{R}^3 \mid A_1 p\leq b_1\}$ and focus on the rigid transformation uncertainty form as $\mathcal T=\{T \in {\rm SE}(3) \mid H x(T) \leq d\}$. 
In the SLAM context in Sec.\ref{sec:SLAM_oometry}, $p$ denotes a landmark's local coordinates, with uncertainty $\mathcal P(A_1,b_1)$ provided by CP, and the set $\mathcal T$ for a robot pose can be obtained by backward UQ in each iteration.

The forward UQ derived from $\mathcal T$ and $\mathcal P(A_1,b_1)$ is to compute the set \begin{equation}\label{eqn:uncertainty_set_Q}
\mathcal Q:= \{q= \phi(T,p) \mid T\in \mathcal T, p\in \mathcal P(A_1,b_1)\}.
\end{equation}
%under $\phi_{\mathcal P(A_1,b_1)}$ is given by $\phi_{\mathcal P(A_1,b_1)}(\mathcal T)$. 
However, due to the bilinear property of $\phi$ and non-convex ${\rm SO}(3)$ constraint, the set $\mathcal Q$ is hard to characterize. In what follows, we instead compute a tight and computationally efficient polytopic approximation $\mathcal P(A_2,b_2)$ of  $\mathcal Q$.

Fo computational efficiency, we fix the matrix $A_2=[a_1,\ldots,a_M]^\top$ \emph{a prior}, where each $a_m$ is a \emph{unit normal}. One way to select $A_2$ is from the regular 3D  polytope with $M$ sides. Now the problem of computing $\mathcal P(A_2,b_2)$ boils down to computing the smallest $b_2$ giving the tightest approximation. We need to solve $M$ optimization problems in form of 
    \begin{subequations} \label{eqn:QCQP_pose_to_point}
      \begin{align}
             [b_2]_m= \max_{q, T, p} &~~{ a}_m^\top q \\
    {\rm s.t.} &~~ H { x}(T) \leq d, \ { A}_{1} p \leq { b}_{1}, \\
    &~~ q=\phi(T,p), \  T \in {\rm SE}(3).
      \end{align}
    \end{subequations}
This problem can be simplified by only considering the vertices of $\mathcal P(A_1,b_1)$ for $p$, as shown in the following lemma.
\begin{lemma} \label{lemma:enumerate_vertex}
Let $\mathcal{V}=\{v_i\}_{i=1}^{N}$ be the vertex set of $\mathcal P(A_1,b_1)$, where $N$ is the vertex number. Problem~\eqref{eqn:QCQP_pose_to_point} is equivalent to 
 \begin{subequations} \label{eqn:QCQP_pose_to_point2}
      \begin{align}
             [b_2]_m= \max_{v_i \in \mathcal V} ~\max_{q, T} &~~{ a}_m^\top q \\
    {\rm s.t.} &~~ H { x}(T) \leq d, \ q=\phi(T,v_i), \\
    &~~ T \in {\rm SE}(3).
      \end{align}
    \end{subequations}
\end{lemma}
\begin{proof}
     Suppose the optimum is achieved at $q^*,T^*,p^*$ where $p^* \notin \mathcal V$. First, assume that $p^*$ strictly lies inside the $\mathcal P(A_1,b_1)$. Since there exists a $\Delta q \in \mathbb R^3$ such that $a_m^\top \Delta q>0$, we can choose a $\lambda>0$ and a $\Delta p=R^{*\top} \Delta q$ such that $p^*+\lambda \Delta p \in \mathcal P(A_1,b_1)$ and $a_m^\top \phi(T^*,p^*+\lambda \Delta p)=a_m^\top(q^*+\lambda \Delta q)>a_m^\top q^*$, which leads to a contradiction. Second, assume that $p^*$ lies on a facet (or edge) of $\mathcal P(A_1,b_1)$. Then, it is straightforward that there is a vertex $v^* \in \mathcal V$ on this facet (or edge) such that $a_m^\top \phi(T^*,v^*)\geq a_m^\top q^*$, which completes the proof. 
\end{proof}

In Lemma~\ref{lemma:enumerate_vertex}, we enumerate all vertices $v_i \in \mathcal V$ and solve the inner problem in~\eqref{eqn:QCQP_pose_to_point2} at each iteration. The maximum objective value over $\mathcal V$ is selected as $[b_2]_m$.
However, the inner problem is still non-convex due to the ${\rm SE}(3)$ constraint. We can relax it to an SDP and obtain a certifiable upper bound. Let $y={\rm col}(x(T),1)$ and $X=y y^\top$. Then, the relaxed SDP is
        \begin{subequations} \label{eqn:SDP_relaxation}
      \begin{align}
           [\hat{b}_2]_m=  \max_{v_i \in \mathcal V}~\max_{X \in \mathbb S^{13}} &~~ \tr (Q_{m,0}X) \\
    {\rm s.t.} &~~ X \succeq 0,   \ [X]_{13,13}=1, \\
    &~~ \tr(Q_iX)=0,i=1,\ldots,6, \\
    &~~ \tr(F_iX)\leq0,i=1,\ldots,d_T,
      \end{align}
    \end{subequations}
    where $d_T$ is the dimension of $d$. The data matrices $Q_{m,0}, Q_i,F_i$ have the following forms:
    \begin{equation}\label{eq:Qm0forward}
   Q_{m,0} = \begin{bmatrix}
0_{12 \times 12} & \begin{matrix} \frac{1}{2}(v_i \otimes I_3)a_m \\ \frac{1}{2}a_m \end{matrix} \\
\begin{matrix}  \frac{1}{2}a_m^\top(v_i^\top \otimes I_3) & \frac{1}{2}a_m^\top \end{matrix} & 0
\end{bmatrix},
\end{equation}

\begin{eqnarray}\label{eq:Qtildeforward}
\begin{cases}
     \tilde Q_1  =\begin{bmatrix}
        I_3 & 0_{3 \times 3} & 0_{3 \times 3} \\
        0_{3 \times 3} & 0_{3 \times 3} & 0_{3 \times 3}  \\
        0_{3 \times 3} & 0_{3 \times 3} & 0_{3 \times 3}
    \end{bmatrix}, \\  \tilde Q_2=\begin{bmatrix}
        0_{3 \times 3} & 0_{3 \times 3} & 0_{3 \times 3} \\
        0_{3 \times 3} & I_3 & 0_{3 \times 3} \\
        0_{3 \times 3} & 0_{3 \times 3} & 0_{3 \times 3}
    \end{bmatrix}, \\
     \tilde Q_3  =\begin{bmatrix}
        0_{3 \times 3} & 0_{3 \times 3} & 0_{3 \times 3} \\
        0_{3 \times 3} & 0_{3 \times 3} & 0_{3 \times 3} \\
        0_{3 \times 3} & 0_{3 \times 3} & I_3
    \end{bmatrix},
    \\ \tilde Q_4=\begin{bmatrix}
        0_{3 \times 3} & \frac{1}{2}I_3 & 0_{3 \times 3} \\
        \frac{1}{2}I_3 & 0_{3 \times 3} & 0_{3 \times 3} \\
        0_{3 \times 3} & 0_{3 \times 3} & 0_{3 \times 3}
    \end{bmatrix}, \\
         \tilde Q_5  =\begin{bmatrix}
        0_{3 \times 3} & 0_{3 \times 3} & \frac{1}{2}I_3 \\
        0_{3 \times 3} & 0_{3 \times 3} & 0_{3 \times 3} \\
        \frac{1}{2}I_3 & 0_{3 \times 3} & 0_{3 \times 3}
    \end{bmatrix}, \\
    Q_6=\begin{bmatrix}
        0_{3 \times 3} & 0_{3 \times 3} & 0_{3 \times 3} \\
       0_{3 \times 3} & 0_{3 \times 3} & \frac{1}{2}I_3 \\
        0_{3 \times 3} & \frac{1}{2}I_3 & 0_{3 \times 3}
    \end{bmatrix}.
\end{cases}
\end{eqnarray}

\begin{align}
    Q_i & =\begin{bmatrix}
        \tilde Q_i & 0_{9 \times 3} & 0_9 \\
        0_{3 \times 9} & 0_{3 \times 3} & 0_3 \\
        0_9^\top & 0_3^\top & -1
    \end{bmatrix}, i=1,2,3,  \label{eq:Qiforward1} \\
    Q_i & =\begin{bmatrix}
        \tilde Q_i & 0_{9 \times 3} & 0_9 \\
        0_{3 \times 9} & 0_{3 \times 3} & 0_3 \\
        0_9^\top & 0_3^\top & 0
    \end{bmatrix}, i=4,5,6, \label{eq:Qiforward2}
\end{align}

\begin{equation}\label{eq:Fiforward}
    F_i=
\begin{bmatrix}
0_{12\times 12} &\tfrac12[H]_i^\top\\
\tfrac12[H]_i & -[d]_i
\end{bmatrix},
\end{equation}
where $[H]_i$ denotes the $i$-th row of $H$.
  The time complexity to solve~\eqref{eqn:SDP_relaxation} using interior-point algorithms is $O(n^{3.5}Nd_T+
n^{2.5}Nd_T^{2}+n^{0.5}Nd_T^{3})$~\cite{nesterov1994interior} ($n=13$ in our problem). Since the data matrices are sparse and the constraint numbers in polytopic description is typically moderate in our problem, the SDP can be efficiently solved by off-the-shelf solvers such as \texttt{MOSEK}. To tighten the SDP relaxation, the automatic method in~\cite{dumbgen2024toward} can also be adopted, which augments the formulation with automatically generated redundant constraints.

Finally, we conclude with the following  \emph{guaranteed} result of the whole procedure.   The forward propagation algorithm is summarized in Algorithm~\ref{algrithm:forward_propagation}.
\begin{algorithm}
	\caption{Forward UQ under $\phi_{\mathcal S}$}
\label{algrithm:forward_propagation}
	\begin{algorithmic}[1]
		\Statex \textbf{Input}: $\mathcal P(A_1,b_1)$ for $\mathcal S$,
        $\mathcal P(H,d)$ for 
        $\mathcal T$, unit normal matrix $A_2$.
        \Statex \textbf{Output}:  Polytopic enclosure $\mathcal P(A_2,\hat b_2)$ for $\phi_{\mathcal S}(\mathcal T)$.
       % \State Fix $A_2=[a_1,\ldots,a_M]^\top$, where each $a_m$ is a unit normal; 
        \State Calculate $\{Q_i\}_{i=1}^6$ in \eqref{eq:Qiforward1}--\eqref{eq:Qiforward2} and $\{F_i\}_{i=1}^{d_T}$ in \eqref{eq:Fiforward};
        \For{$m=1:M$} 
\State Calculate $Q_{m,0}$ according to \eqref{eq:Qm0forward};
\State Solve SDP~\eqref{eqn:SDP_relaxation} for  $[\hat b_2]_m$;
        \EndFor
	\end{algorithmic}
\end{algorithm}
\begin{theorem} \label{the:forward_conservatism}
    The set $\mathcal P(A_2,\hat{b}_2)$ is a guaranteed approximation for $\mathcal Q$ in~\eqref{eqn:uncertainty_set_Q}, that is, $\mathcal Q \subseteq \mathcal P(A_2,\hat{b}_2)$.
\end{theorem} 
\begin{proof}
    For any $q \in \mathcal Q$, from~\eqref{eqn:QCQP_pose_to_point}, we have that $a_m^\top q\leq [b_2]_m, m\in \{1,\ldots,M\}$, which implies that $\mathcal Q \subseteq \mathcal P(A_2,{b}_2)$. Moreover, since~\eqref{eqn:SDP_relaxation} is a relaxation of~\eqref{eqn:QCQP_pose_to_point2}, it holds that $[b_2]_m \leq [\hat b_2]_m, m\in \{1,\ldots,M\}$. Therefore,  $\mathcal Q \subseteq \mathcal P(A_2,{b}_2) \subseteq \mathcal P(A_2,\hat{b}_2)$, which completes the proof.
\end{proof}

% The problem in this subsection is to derive UQ under the group action $\phi: {\rm SE}(3) \times \mathbb R^3 \rightarrow \mathbb R^3$: 
% \begin{mdframed}
% \begin{problem} [{\bf UQ from ${\rm SE}(3) \times \mathbb R^3$ to $\mathbb R^3$}]\label{problem:SLAM_triangulation}
%     Given the uncertainty sets ${ A}_x { x} \leq { b}_x$ and ${ A}_p { p} \leq { b}_p$, compute the uncertainty set ${ A}_{q} { q} \leq { b}_q$ for $q$.
% \end{problem}
% \end{mdframed}

% \zgy{The above problem is a bilinear program, which is hard to solve, and we can 
% \begin{itemize}
%     \item approximately solve it by enumerating all vertices of the polytope ${ A}_{p} p \leq { b}_{p}$, denoted by the set $V_{\mathcal P_p}$;
%         \begin{subequations} \label{eqn:LP_pose_to_point}
%       \begin{align}
%            \max_{v_p \in \mathcal V_{\mathcal P_p}}~~  \max_{q, x} &~~{ a}_m^\top q \\
%     {\rm s.t.} &~~ { A}_{x} { x} \leq { b}_{x} \\
%     &~~ q={ R} v_p+{ t}.
%       \end{align}
%     \end{subequations}
% \end{itemize}}

% Denote the obtained offset as $c_m$. Then we have $a_m^\top q \leq c_m$, which finally gives $A_q q \leq b_q$ with $A_q=[a_1,\ldots,a_M]^\top$ and $b_q=[c_1,\ldots,c_M]^\top$.

\subsection{Backward Uncertainty Quantification}\label{subsec:point_to_pose}

The backward UQ problem can be interpreted as follows: given a fixed bounded
set $\mathcal S$, for a target subset $\mathcal Q\subset \mathbb R^3$
find the preimage 
$\phi_{\mathcal S}^{-1}(\mathcal Q)$ in the manifold of $\SEthree$.
The following theorem provides an upper bound on this preimage by relating it to the preimage of a single-valued map 
 $\phi_{s}$ for some $s\in\mathcal S$.
\begin{theorem} \label{the:point_to_pose}
 For any $s\in \mathcal S$, $\phi_{\mathcal S}^{-1}(\mathcal Q)\subseteq\phi_{s}^{-1}(\mathcal Q^{+})$ for all
 $\mathcal Q^{+}\supseteq \mathcal Q+ \mathcal B(0,d(\mathcal S))$.  In particular, when $s={\rm cent}(\mathcal S)$, 
$\phi_{\mathcal S}^{-1}(\mathcal Q)\subseteq\phi_{s}^{-1}(\mathcal Q^{+})$ for all
$\mathcal Q^{+}\supseteq \mathcal Q+ \mathcal B(0,r(\mathcal S))$.
\end{theorem}
\begin{proof}
    From the definition of $\phi_{\mathcal S}^{-1}(\mathcal Q)$, for any $T\in \phi_{\mathcal S}^{-1}(\mathcal Q)$, there exist $s \in \mathcal S$ and $q \in \mathcal Q$ such that $q= \phi_{s}(T)$. Fix an arbitrary $s'\in\mathcal S$. Since $\|s-s'\|\leq d(\mathcal S)$, we have 
$$
\phi_{s'}(T)=Rs'+t=q+R(s'-s)
\in 
\mathcal Q+\mathcal B(0,d(\mathcal S)),
$$
which proves the first assertion. For the second assertion, take $s'=\cent(S)$. Then, $\|s-s'\|\leq r(S)$ by Definition~\ref{def:Chebyshev_ball_center}, and thus
$\phi_{s'}(T)\in 
\mathcal Q+\mathcal B(0,r(\mathcal S))$,
which completes the proof. 
\end{proof}

When 
$\mathcal S$ and $\mathcal Q$ are arbitrary sets, the preimage is generally difficult to characterize. 
Although Theorem~\ref{the:point_to_pose} provides an upper bound on this preimage, obtaining a closed-form expression remains impractical. For further analysis, we restrict $\mathcal Q$  to be a polytope.
A useful lemma is 
\begin{lemma}\label{lemma:minimal_enclosing_polytope}
Give a polytope $\mathcal P(A,b)$, for any real number $\epsilon>0$, it holds that
$\mathcal P(A,b)+\mathcal B(0,\epsilon)\subset \mathcal P(A,b+\epsilon \mathbf{1})$.
\end{lemma}
\begin{proof}
    For any $p_{\epsilon}\in \mathcal P(A,b)+\mathcal B(0,\epsilon)$, one can find $p\in\mathcal P(A,b)$ and $v\in\mathcal B(0,\epsilon)$ such that 
$p_{\epsilon}=p+v$.
Then, $Ap_{\epsilon}=Ap+Av\leq b+Av$. Since $\|v\|\leq \epsilon$ (due to $v\in \mathcal B(0,\epsilon)$) and each row $a_m^\top$ in $A$ is a unit vector, we have $a_m^\top v\leq \epsilon$ for all $m$. Thus $Ap_{\epsilon}\leq b+\epsilon \mathbf 1$, which shows $p_{\epsilon} \in\mathcal P(A,b+\epsilon \mathbf 1)$.
To prove the inclusion is strict,
take a point $p_{\epsilon}'$ satisfying $Ap_{\epsilon}'=b+\epsilon \mathbf 1$. For any $p\in\mathcal P(A,b)$, 
we have $A(p_{\epsilon}'-p)=
b+\epsilon \mathbf 1-Ap\geq \epsilon \mathbf 1$. 
The Cauchy-Schwarz inequality implies
$a_m^\top (p_{\epsilon}'-p)\leq \|p_{\epsilon}'-p\|$ for each $m$ with equality only if 
$p_{\epsilon}'-p$ is parallel to $a_m$.
On the other hand, since $A$ is the data matrix for a polytope, at least one inequality is strict. Thus,$\|p_{\epsilon}'-p\|>\epsilon$, which concludes
$p_{\epsilon}'\not \in \mathcal P(A,b)+\mathcal B(0,\epsilon)$. This completes the proof.
\end{proof}

An immediate benefit of restricting $\mathcal Q$ to a polytope represented in $H$-form as $\mathcal P(A_2,b_2)$ is that the set $\mathcal Q+\mathcal B(0,r(\mathcal S))$ can be bounded by a polytope $\mathcal P(A_2, b_2+r(\mathcal S) \mathbf 1)$ as established in Lemma~\ref{lemma:minimal_enclosing_polytope} by considering the Chebyshev ball of \(\mathcal{S}\).
Moreover, by Theorem~\ref{the:point_to_pose}, fixing \(s = \operatorname{cent}(\mathcal{S})\) renders \(\phi_s\) linear, ensuring that the preimage of a polytope under \(\phi_s\) remains a polytope. In addition, if \(\mathcal{S}\) itself is constrained to be a polytope $\mathcal P(A_1,b_1)$, locating its Chebyshev ball is computationally efficient.
In the SLAM context in Sec.\ref{sec:SLAM_oometry}, $\mathcal P(A_1,b_1)$ for local point $p$ can be provided by CP, and $\mathcal P(A_2,b_2)$ for global map coordinate $q$ can be obtained by previous forward UQ rounds. The backward UQ from $\mathcal P(A_1,b_1)$ and $\mathcal P(A_2,b_2)$ is to compute the set 
\begin{align}\label{eqn:inverse_set_T}
    \mathcal T:=\{T \in {\rm SE}(3) \mid \exists &p\in\mathcal P(A_1,b_1), q \in\mathcal P(A_2,b_2) \notag\\  
    &{\rm such ~that } ~q=\phi(T,p)\}.
\end{align}

From Theorem~\ref{the:point_to_pose} and Lemma~\ref{lemma:minimal_enclosing_polytope}, by setting $s= \cent (P(A_1, b_1))$, 
%the preimage 
%$ \phi_s^{-1}\left(\mathcal %P(A_2, b_2+r(\mathcal P(A_1, b_1) \mathbf 1)\right)$ encloses the original set $\mathcal T$. By definition,
for any $T\in \phi_s^{-1}\left(\mathcal P(A_2, b_2+r(\mathcal P(A_1, b_1) \mathbf 1)\right)$, we have $A_2 \phi(T,s)\leq b_2+r(\mathcal P(A_1, b_1) \mathbf 1$.  It follows that
\begin{equation} \label{eqn:backward_from1}
    \mathcal T_{r}= \{T \in {\rm SE}(3)\mid Hx(T)\leq d\},
\end{equation}
where $H=[{ A}_{2} (s^\top \otimes { I}_3), { A}_{2}]$ and $d={ b}_{2}+r(\mathcal P(A_1, b_1) \mathbf 1$.

Theorem~\ref{the:point_to_pose} also allows for a looser bound based on the polytope diameter, which is computationally simpler as it bypasses the computation of the Chebyshev center. 
Choosing arbitrary
$s \in \mathcal P(A_1,b_1)$, and the preimage $ \phi_s^{-1}(\mathcal P(A_2, b_2+d(\mathcal P(A_1, b_1) \mathbf 1))$ is given by 
\begin{equation} \label{eqn:backward_from2}
    \mathcal T_{d}= \{T \in {\rm SE}(3)\mid Hx(T)\leq d\},
\end{equation}
where $H=[{ A}_{2} (s^\top \otimes { I}_3), { A}_{2}]$ and $d={ b}_{2}+d(\mathcal P(A_1, b_1) \mathbf 1$.

In summary, we conclude with the following result on the \emph{guarantee} of the whole procedure.
\begin{theorem} \label{the:backward_conservatism}
Both $\mathcal T_r$ in~\eqref{eqn:backward_from1} and $\mathcal T_d$ in~\eqref{eqn:backward_from2} are guaranteed approximations for $\mathcal T$ in~\eqref{eqn:inverse_set_T}, that is, $\mathcal T\subset \mathcal T_r, \mathcal T_d$.
Moreover, $\mathcal T_r$ is the unique minimum set that 
\begin{enumerate}
    \item can be written as ${\rm SE}(3) \cap \mathcal P(H,\cdot)$ with template $H=[{ A}_{2} (s^\top \otimes { I}_3), { A}_{2}], s=\cent (\mathcal P(A_1, b_1))$ \label{T1_tempalte_property};
\item  encloses $\mathcal T$\label{T1_enclsoingT};
\item is enclosed by any other set satisfying~\ref{T1_tempalte_property}) and~\ref{T1_enclsoingT}).
\end{enumerate}
\end{theorem}
\begin{proof}
    The conservatism of $\mathcal T_r$ and $\mathcal T_d$ are straightforward by Theorem~\ref{the:point_to_pose} and Lemma~\ref{lemma:minimal_enclosing_polytope}. To prove the second claim, suppose that there exists $d'$ such that $\mathcal T'= \{T \in {\rm SE}(3)\mid Hx(T)\leq d'\}$ encloses $\mathcal T$, and there exists at least an $m$ such that $[d']_m<[d]_m$. Recall 
    $s={\rm cent}(\mathcal P(A_1,b_1))$. Take $s'$ in the set $\mathcal P(A_1,b_1)$ to achieve the Chebyshev radius, that is,
    \begin{equation}\label{eqn:def_coner_point_s'}
    \|s-s'\|=r(\mathcal P(A_1,b_1)).
    \end{equation}
    Denote $A_2=[a_1,\ldots,a_M]^\top$. Let $q^*$ from $\mathcal P(A_2,b_2)$ and $q_m^*$ from the facet $a_m^\top q=[d]_m$ so that they achieve the distance between the two sets, which combining the unity of
$a_m^\top$ leads to 
$$a_m^\top(q^*_m-q^*)=\|q^*-q_m^*\|=r(\mathcal P(A_1,b_1)).$$
   Due to~\eqref{eqn:def_coner_point_s'}, we can find a rotation $R$ satisfying $R(s-s')=q_m^*-q^*$.
   Define $t=q_m^*-Rs$ and we have $Rs+t=q_m^*$. 
   Hence, $Rs'+t=Rs+t-R(s-s')=q_m^*+q^*-q_m^*=q^*$ and 
$\begin{bmatrix}
    { R}~  t \\
    { 0}_3^\top~ 1
\end{bmatrix}=:T\in\mathcal {T}$ by~\eqref{eqn:inverse_set_T} since $s'\in \mathcal P(A_1,b_1)$ and 
$q^*\in\mathcal P(A_2,b_2)$.
On the other hand, by~\eqref{eqn:backward_from1},
$$[Hx(T)]_m =a_m^\top (Rs+t)=a_m^\top q_m^*=[d]_m,$$
which leads to $[Hx(T)]_m>[d']_m$. As a result, $T\not \in \mathcal T'$. It implies that $\mathcal T'$ does not enclose $\mathcal T$,
 introducing a contradiction and completing the proof.
\end{proof}

Our backward UQ is reported in Algorithm~\ref{algrithm:backward_propagation}.
\begin{algorithm}
	\caption{Backward UQ under $\phi_{\mathcal S}$ }
	\label{algrithm:backward_propagation}
	\begin{algorithmic}[1]
		\Statex \textbf{Input}:  $\mathcal P(A_1,b_1)$ for $\mathcal S$ and $\mathcal P(A_2,b_2)$ for $\mathcal Q$.
        \Statex \textbf{Output}:Polytopic enclosure $\mathcal T$ for $\phi^{-1}_{\mathcal S}(\mathcal Q)$.
       \If{prioritize tightness}
\State Calculate $s=\cent(\mathcal P(A_1,b_1))$ and $r(\mathcal P(A_1,b_1))$ by solving~\eqref{eqn:cal_chebyshev_ball};
\State Set $H=[{ A}_{2} (s^\top \otimes { I}_3), { A}_{2}]$ and $d={ b}_{2}+r(\mathcal P(A_1, b_1) \mathbf 1$;
       \ElsIf{prioritize efficiency}
\State Calculate $d(\mathcal P(A_1,b_1))$ 
%by solving~\eqref{eqn:cal_polytope_diameter} 
and set  $d={ b}_{2}+d(\mathcal P(A_1, b_1) \mathbf 1$;
\State Choose $s \in \mathcal P(A_1,b_1)$ and set $H=[{ A}_{2} (s^\top \otimes { I}_3), { A}_{2}]$;
       \EndIf
       \State Set $\mathcal T= \{T \in {\rm SE}(3)\mid Hx(T)\leq d\}$.
	\end{algorithmic}
\end{algorithm}

\subsection{Rigid Transformation Uncertainty Compound} \label{subsec:pose_compound}

The rigid transformation uncertainty compound problem is formulated as follows:
find a set $\mathcal T_{3}$, if there exists one, of ${\rm SE}(3)$ for uncertainty sets $\mathcal T_1,\mathcal T_2 \subset {\rm SE}(3)$
such that 
\begin{equation}\label{eqn:T_coumpound}
\phi_{\mathcal S}(\mathcal T_{3})=\phi_{\phi_{\mathcal S}(\mathcal T_2)}(\mathcal T_1)
\end{equation}
holds for any bounded point set
$\mathcal S$.
The problem is well-defined only if $\mathcal T_3$ exists and is unique.
The composition of the group action $\phi$ on $\mathbb R^3$ is compatible with the group multiplication in $\mathrm{SE}(3)$. Specifically, for $T_1, T_2 \in \mathrm{SE}(3)$, $\phi_{\phi_p(T_2)}(T_1)=\phi_{p}(T_1T_2)$. Consequently, the compound of these uncertainties can be fully characterized using the Minkowski product, establishing that the problem is well-defined. 
\begin{lemma}\label{T_compound_iff_set_product}
For two sets $\mathcal T_1,\mathcal T_2\subset {\rm SE}(3)$, 
$\mathcal T_3\subset {\rm SE}(3)$ satisfies~\eqref{eqn:T_coumpound} for any bounded set $\mathcal S$ if and only if 
$\mathcal T_3=\mathcal T_1\cdot \mathcal T_2$.
\end{lemma}
\begin{proof}
    For $\mathcal T_1\cdot \mathcal T_2$, we have
\begin{align}\label{eqn:equality_compound}
\phi_{\mathcal S}(\mathcal T_1\cdot \mathcal T_2)&= \cup_{T\in \mathcal T_1\cdot \mathcal T_2}\phi_{\mathcal S}(T)\notag\\
&= \{\phi_p(T_1T_2)\mid p\in\mathcal S, T_1\in\mathcal T_1,T_2\in\mathcal T_2
\}\notag\\
& =\{\phi_{q}(T_2)\mid q\in \phi_{\mathcal S}(\mathcal T_1),T_2\in\mathcal T_2\}\notag\\
& =\phi_{\phi_{\mathcal S}(\mathcal T_1)}(\mathcal T_2).
\end{align}
Moreover, since~\eqref{eqn:T_coumpound} is satisfied for any bounded $\mathcal S$, one can reuse argument in~\eqref{eqn:equality_compound} and show by contradiction that $\mathcal T_3$ cannot be any set other than $\mathcal T_1 \cdot \mathcal T_2$, which completes the proof.
\end{proof}

By Lemma~\ref{T_compound_iff_set_product}, the problem of compounding uncertainties in rigid transformations can be conveniently reformulated as a set product problem.
To facilitate further analysis, let us restrict our attention to the polytopic uncertainty sets $\mathcal T_1=\{T\in {\rm SE}(3)\mid { A}_{1} x(T) \leq { b}_{1}\}$  on $T_1$ and $\mathcal T_2=\{T\in {\rm SE}(3)\mid { A}_{2} x(T) \leq { b}_{2}\}$  on $T_2$. Assume that $\mathcal P(A_1, b_1)$ and $\mathcal P(A_2, b_2)$ are given. The polytopic structure is not generally preserved under the group product. We propose two methods---one direct and one indirect---to compute enclosures for $\mathcal T_1\cdot \mathcal T_2$.

\subsubsection{Direct Computation Method}
In this part, we are interested in developing an algorithm to compute a tight polytopic set $\mathcal P(A_3, b_3)=\{x\in \mathbb{R}^{12}\mid { A}_{3} x \leq { b}_{3}\}$  on $x(T_3)$. 
Similar to the forward UQ in Sec.\ref{subsec:pose_to_point}, we fix the template $A_3=[a_1,\ldots,a_M]^\top$ \emph{a prior} to improve the computational efficiency. Here each $a_m$ is a \emph{unit normal}. Then, the problem to be solved is    
\begin{subequations} \label{eqn:QCQP_odometry}
    \begin{align}
              [b_3]_m = \max_{{ T}_1, { T}_2, { T}_3} &~~{ a}_m^\top { x}(T_3) \\
    {\rm s.t.} &~~ { T}_3= { T}_1 { T}_2, \ { A}_{1} x(T_1) \leq { b}_{1}, \\
    &~~ { A}_{2} x(T_2) \leq { b}_{2}, \ T_1,T_2 \in {\rm SE}(3).
      \end{align}
    \end{subequations}

This bilinear optimization problem is hard to be globally solved. Let $y={\rm col}(x(T_1),x(T_2),1)$ and $X=y y^\top$.  We have the following relaxed SDP giving a certifiable upper bound.
        \begin{subequations} \label{eqn:SDP_relaxation2}
      \begin{align}
           [\hat{b}_3]_m=  \max_{X \in \mathbb S^{25}} &~~ \tr (Q_{m,0}X) \\
    {\rm s.t.} &~~ X \succeq 0, \   [X]_{25,25}=1, \\
    &~~ \tr(Q_iX)=0,i=1,\ldots,12,\\
    &~~ \tr(F_iX)\leq0,i=1,\ldots,b_{1,T},\\
    &~~ \tr(G_iX)\leq0,i=1,\ldots,b_{2,T},
      \end{align}
    \end{subequations}
    where $b_{1,T}$ and $b_{2,T}$ are the dimensions of $b_1$ and $b_2$, respectively, and the data matrices $Q_{m,0}, Q_i,F_i,G_i$ have the following forms:
    \begin{eqnarray}\label{eq:Qicompound}
\begin{cases}
    Q_i & =\begin{bmatrix}
        \tilde Q_i & 0_{9 \times 15} & 0_9 \\
        0_{15 \times 9} & 0_{15 \times 15} & 0_{15} \\
        0_9^\top & 0_{15}^\top & -1
    \end{bmatrix}, i=1,2,3,  \\
    Q_i & =\begin{bmatrix}
        \tilde Q_i & 0_{9 \times 15} & 0_9 \\
        0_{15 \times 9} & 0_{15 \times 15} & 0_{15} \\
        0_9^\top & 0_{15}^\top & 0
    \end{bmatrix}, i=4,5,6, \\
    Q_i & =\begin{bmatrix}
        0_{12 \times 12} & 0_{12 \times 9} & 0_{12 \times 3} & 0_{12} \\
        0_{9 \times 12} & \tilde Q_{i-6} & 0_{9 \times 3} & 0_9 \\
        0_{3 \times 12} & 0_{3 \times 9} & 0_{3 \times 3} & 0_{3} \\
        0_{12}^\top & 0_{9}^\top & 0_{3}^\top & -1
    \end{bmatrix}, i=7,8,9, \\
    Q_i & =\begin{bmatrix}
        0_{12 \times 12} & 0_{12 \times 9} & 0_{12 \times 3} & 0_{12} \\
        0_{9 \times 12} & \tilde Q_{i-6} & 0_{9 \times 3} & 0_9 \\
        0_{3 \times 12} & 0_{3 \times 9} & 0_{3 \times 3} & 0_{3} \\
        0_{12}^\top & 0_{9}^\top & 0_{3}^\top & 0
    \end{bmatrix}, i=10,11,12,
    \end{cases}
\end{eqnarray}
with the same $\tilde Q_i,i=1,\ldots,6$ in \eqref{eq:Qtildeforward},
\begin{equation}\label{eq:Qm0compound}
    Q_{m,0}=\begin{bmatrix}
        0_{9 \times 9} & 0_{9 \times 3} & \frac{1}{2}\Phi & \frac{1}{2} \Psi & 0_9 \\
        0_{3 \times 9} & 0_{3 \times 3} & 0_{3 \times 9} & 0_{3 \times 3} & \frac{1}{2} a_{m,t} \\
        \frac{1}{2}\Phi^\top & 0_{9 \times 3} & 0_{9 \times 9} & 0_{9 \times 3} & 0_9 \\
        \frac{1}{2}\Psi^\top & 0_{3 \times 3} & 0_{3 \times 9} & 0_{3 \times 3} & 0_3 \\
        0_9^\top & \frac{1}{2} a_{m,t}^\top & 0_9^\top & 0_3^\top & 0
    \end{bmatrix},
\end{equation}
with $a_{m,t}=[0_{3 \times 9}~I_3]a_m$, $\Phi=\sum_{k=1}^3 U_k^\top A S_k$, and $\Psi=[U_1^\top a_{m,t}~U_2^\top a_{m,t}~U_3^\top a_{m,t}]$, with $A=[[a_m]_{1:3}~[a_m]_{4:6}~[a_m]_{7:9}]$ and 
\begin{equation*}
    U_1=[I_3~0_{3 \times 6}],~ U_2=[0_{3 \times 3}~I_3~0_{3 \times 3}],~U_3=[0_{3 \times 6}~I_3],
\end{equation*}
\begin{eqnarray}\label{eq:Sicompound}
\begin{cases}
    S_1 & =\begin{bmatrix}
        1 & 0 & 0_7^\top \\
        0_3^\top & 1 & 0_5^\top \\
        0_6^\top & 1 & 0_2^\top
    \end{bmatrix},  
    S_2  =\begin{bmatrix}
        0 & 1 & 0_7^\top \\
        0_4^\top & 1 & 0_4^\top \\
        0_7^\top & 1 & 0
    \end{bmatrix},  \\
    S_3 & =\begin{bmatrix}
        0_2^\top & 1 & 0_6^\top \\
        0_5^\top & 1 & 0_3^\top \\
        0_7^\top & 0 & 1
    \end{bmatrix}, 
    \end{cases}
\end{eqnarray}
and 
\begin{equation}\label{eq:FiGicompound}
  \begin{split}
        F_i &=
\begin{bmatrix}
0_{24\times 24} & \begin{matrix}\tfrac12[A_1]_i^\top\\[2pt] 0_{12}\end{matrix}\\
\begin{matrix}\tfrac12[A_1]_i & 0_{12}^\top\end{matrix} & -[b_1]_i 
\end{bmatrix}, \\
G_i &=
\begin{bmatrix}
0_{24\times 24} &
\begin{matrix}
0_{12}\\
\tfrac12 [A_2]_i^\top
\end{matrix}
\\[8pt]
\begin{matrix}
0_{12}^\top & \tfrac12 [A_2]_i
\end{matrix}
& -[b_2]_i
\end{bmatrix}.
  \end{split}
\end{equation}
The time complexity to solve the SDP using interior-point algorithms is $O(n^{3.5}b_{1,T}+
n^{2.5}b_{1,T}^{2}+n^{0.5}b_{1,T}^{3}+n^{3.5}b_{2,T}+
n^{2.5}b_{2,T}^{2}+n^{0.5}b_{2,T}^{3})$~\cite{nesterov1994interior} ($n=25$ in this problem). Similarly, this SDP can be tightened using the method in~\cite{dumbgen2024toward} and solved by off-the-shelf solvers like \texttt{MOSEK}.
\begin{theorem} \label{the:direct_compound_conservatism}
    The set $\{T_3\in {\rm SE}(3)\mid { A}_{3} x(T_3) \leq \hat{ b}_{3}\}$ is a guaranteed approximation for the original set $\mathcal T_3$.
\end{theorem} 
The proof of Theorem~\ref{the:direct_compound_conservatism} is similar to that of Theorem~\ref{the:forward_conservatism} and is omitted here.
 
% Since $x(T_3) \in \mathbb R^{12}$, the facet number $M$ of $\mathcal{P}(A_3,b_3)$ is usually large and solving $M$ SDPs can be computationally inefficient. Here, we also provide a more efficient method to approximate $b_3$. Specifically, we solve~\eqref{eqn:QCQP_odometry} by using the vertices of $\mathcal P(A_1,b_1)$.  One can use the reverse-search algorithm~\cite{avis1992pivoting}
% for the vertex enumeration of the set $\mathcal P(A_1,b_1)\subset \mathbb{R}^{12}$. Denote by $\mathcal{V}=\{v_i\}_{i=1}^{N}$ the set of vertices $\mathcal P(A_1,b_1)$, where $N$ is the number of its vertices. Then, we solve the following problem
% \begin{subequations} \label{eqn:sequential_LPs}
%     \begin{align}
%               [\hat{b}_3]_m = \max_{i \in \{1,\ldots,N\}}\max_{{ T}_2, { T}_3 \in \mathbb R^{3 \times 3}} &~~{ a}_m^\top { x}(T_3) \\
%     {\rm s.t.} &~~ { T}_3= { T}_1 { T}_2 \\
%     &~~ { A}_{2} x(T_2) \leq { b}_{2} \\
%     &~~ { v}_i={{\rm vec}}(T_1) \\
%     &~~ { x}_i={{\rm vec}}(T_i), \ i=2,3.
%       \end{align}
%     \end{subequations}
% The inner problem in~\eqref{eqn:sequential_LPs} is an LP and be solved efficiently, with time complexity $O(n^{2.37}\log n)$ using the algorithm~\cite{cohen2021solving}, where $n$ is the dimension of the decision vector ($n=24$ in our problem). The total computational complexity for solving problem~\eqref{eqn:sequential_LPs} is $O(N n^{2.37}\log n )$.
Our direct pose uncertainty compound algorithm is summarized in Algorithm~\ref{algrithm:direct_compound}.
\begin{algorithm}
	\caption{Direct pose compound}
	\label{algrithm:direct_compound}
	\begin{algorithmic}[1]
		\Statex \textbf{Input}:  $\mathcal P(A_1,b_1)$, $\mathcal P(A_2,b_2)$ for $x(T_1)$, $x(T_2)$ and unit normal matrix $A_3$
        \Statex \textbf{Output}: Polytopic enclosure $\mathcal P(A_3,\hat b_3)$ for $x(T_1T_2)$.
     %   \State Fix $A_3=[a_1,\ldots,a_M]^\top$, where each $a_m$ is a unit normal;
\State Calculate $\{Q_i\}_{i=1}^{12}, \{F_i\}_{i=1}^{b_{1,T}}, \{G_i\}_{i=1}^{b_{2,T}}$ according to \eqref{eq:Qicompound}--\eqref{eq:FiGicompound};
        \For{$m=1:M$} 
\State Calculate $Q_{m,0}$ according to \eqref{eq:Qm0compound};
\State Solve the SDP~\eqref{eqn:SDP_relaxation2} for $[\hat b_3]_m$;
        \EndFor
% \State Obtain vertex set $\mathcal{V}=\{v_i\}_{i=1}^{N}$ of $\mathcal P(A_1,b_1)$ via vertex enumeration;
% \For{$m=1:M$}
% \State Solve problem~\eqref{eqn:sequential_LPs} and obtain $[\hat b_3]_m$;
% \EndFor
	\end{algorithmic}
\end{algorithm}

\subsubsection{Indirect Computation Method}
Since $x(T_3) \in \mathbb R^{12}$, its polytopic template $A_3$ can be complex (i.e., $M$ is large), and solving $M$ SDPs~\eqref{eqn:SDP_relaxation2} may become prohibitive. We instead propose an efficient indirect method. This approach decomposes the polytope $\mathcal P(A,b)$ via set projection into translation $\mathcal P(A_t,b_t)$ and rotation $\mathcal P(A_r,b_r)$ (where $r=\operatorname{vec}(R)$), and then further relax the rotation uncertainty.

%When the template for the polytopic enclosure of $x(T_3)$ is complex (i.e.,~when $M$ is large), solving $M$ SDPs via the direct method~\eqref{eqn:SDP_relaxation2} may become computationally prohibitive for real-time applications. This motivates an efficient indirect compound method. The core idea is to decompose uncertainties for rotation and translation, and compress the rotation uncertainty to one dimension. Specifically, for a polytope $\mathcal P(A,b)$ for ${ x}(T)$, by set projection, we can decompose uncertainties for rotation and translation and obtain $\mathcal P(A_r,b_r)$ for $r={\rm vec}({ R})$ and $\mathcal P(A_t,b_t)$ for $t$. 

%Instead of studying the polytope $\mathcal P(A_r,b_r)$, we introduce the set  $\mathcal X :=\{{\rm Exp}(\theta { w}) \mid { w} \in \mathbb S^2, \theta \in [0, \theta']\}$, which is defined by one-dimensional uncertainty $\theta$. The following proposition gives a characterization of $\mathcal X$ that over-approximates the set $\mathcal P(A_r,b_r)$. 

We begin with over-approximating $\mathcal P(A_r,b_r)$ using the polytopic set $\mathcal X :=\{{\rm Exp}(\theta { w}) \mid { w} \in \mathbb S^2, \theta \in [0, \theta']\}$ (defined by one-dimensional uncertainty $\theta$). The following proposition characterizes this set.

\begin{proposition} \label{lemma:single_theta_conservatism}
   Consider the polytope $\mathcal P(A_r,b_r)$ where $A_r=[a_1, \ldots,a_M]^\top$ and $b_r=[b_1, \ldots,b_M]^\top$.  Then, the set $\bar R \mathcal X$ is a guaranteed approximation for the set $\{R \in {\rm SO}(3) \mid {\rm vec}(R) \in \mathcal P(A_r,b_r)\}$, where $\mathcal X :=\{{\rm Exp}(\theta { w}) \mid { w} \in \mathbb S^2, \theta \in [0, \theta']\}$. Here, $\bar R \in {\rm SO}(3)$ is obtained from  $\bar r$ via inverse vectorization and singular value decomposition-based ${\rm SO}(3)$ projection~\cite{zeng2025bias}  and  
   $\theta'=\cos^{-1}\left(\max\{\frac{c^*-1}{2},-1\}\right)$. In particular, $\bar r$ is the Chebyshev center of $\mathcal P(A_r,b_r)$,
and $c^*$ is obtained by solving the problem 
 \begin{subequations} \label{eqn:SDP_relaxation3}
      \begin{align}
           c^*=  \max_{X \in \mathbb S^{10}} &~~ \tr (Q_{0}X) \\
    {\rm s.t.} &~~ X \succeq 0, \     [X]_{10,10}=1, \\
    &~~ \tr(Q_iX)=0,i=1,\ldots,6, \\
    &~~ \tr(F_iX)\leq0,i=1,\ldots,M.
      \end{align}
    \end{subequations}
    The data matrices $Q_0, Q_i, F_i$ are provided as
    \begin{equation} \label{eqn:matrix1_sdp3}
    Q_0=\begin{bmatrix}
    0_{9 \times 9} & \frac{1}{2} \bar r \\
    \frac{1}{2} \bar r^\top & 0
\end{bmatrix},~F_i=\begin{bmatrix}
    0_{9 \times 9} & \frac{1}{2} [A_r]_i^\top \\
    \frac{1}{2} [A_r]_i & -[b_r]_i
\end{bmatrix},
\end{equation}
\begin{equation} \label{eqn:matrix2_sdp3}
    Q_i=\begin{bmatrix}
    \tilde Q_i & 0_9 \\
    0_9^\top & -1
\end{bmatrix}, i=1,2,3, ~Q_i=\begin{bmatrix}
    \tilde Q_i & 0_9 \\
    0_9^\top & 0
\end{bmatrix}, i=4,5,6,
\end{equation}
where $\tilde Q_i,i=1,\ldots,6$ are the same as in~\eqref{eq:Qtildeforward}.
\end{proposition}
\begin{proof}
    Since $\bar r$ is the Chebyshev center of $\mathcal P(A_r,b_r)$, $\bar R$ can be interpreted as an approximate center of $\{R \in {\rm SO}(3) \mid {\rm vec}(R) \in \mathcal P(A_r,b_r)\}$.
On the other hand, let $R^*$ be the optimal solution to the problem 
\begin{subequations} \label{eqn:max_theta'}
        \begin{align}
             \min_{R \in {\rm SO}(3)} &~~\tr(R \bar R^\top) \\
    {\rm s.t.} &~~ A_r {\rm vec}({ R}) \leq b_r.
      \end{align}
\end{subequations}
For any $R \in \{R \in {\rm SO}(3) \mid {\rm vec}(R) \in \mathcal P(A_r,b_r)\}$, we have $\tr(R^*\bar R^\top)\leq \tr(R\bar R^\top)$. That is ${\rm dis}(R,\bar R) \leq {\rm dis}(R^*,\bar R)$. Moreover, Problem~\eqref{eqn:SDP_relaxation3} is actually an SDP relaxation of~\eqref{eqn:max_theta'}. Hence, $c^*\leq \tr(R^* \bar R^\top)$, which implies $\theta'=\cos^{-1}\left(\max\{\frac{c^*-1}{2},-1\}\right) \geq {\rm dis}(R^*,\bar R)\geq {\rm dis}(R,\bar R)$.
This means that $R \in \bar R \mathcal X$ and completes the proof.
\end{proof}

Now let us consider the pose compound $T_3=T_1 T_2$, as defined before. Suppose that the uncertainties on the pose $T_i$ are characterized by the  set $\bar{ R}_i \mathcal X_i$ on the rotation $R_i$
and the  set $\mathcal P(A_{t_i}, b_{t_i})$ on the translation $t_i$, where $\mathcal X_i =\{{\rm Exp}(\theta_i { w}) \mid { w} \in \mathbb S^2, \theta_i \in [0, \theta_i']\}$.  According to the pose compound operation, we have that $R_3=R_1R_2$ and ${ t}_3={ R}_1{ t}_2+{ t}_1$. From a set-valued perspective, we have
\begin{align*}
    { R}_3  \in (\bar { R}_1 \mathcal X_1 \bar { R}_2) \cdot \mathcal X_2, \ 
    { t}_3  \in \bar { R}_1 \mathcal X_1 \cdot \mathcal P(A_{t_2},  b_{t_2})+\mathcal P(A_{t_1}, b_{t_1}).
\end{align*}

Next we develop an efficient method to compute tight uncertainty sets $\bar { R}_3 \mathcal X_3=\{\bar R_3{\rm Exp}(\theta_3 { w}) \mid { w} \in \mathbb S^2, \theta_3 \in [0, \theta_3']\}$ on $R_3$ and $\mathcal P(A_{t_3}, b_{t_3})=\{t \in \mathbb R^3 \mid A_{t_3} t \leq b_{t_3}\}$ on $t_3$.
The following lemma gives the exact characterization of $\bar { R}_3 \mathcal X_3$. 
\begin{lemma} \label{the:R_propagation}
   Let  $\theta_3'=\theta_1'+\theta_2'$
   and $\bar { R}_3=\bar { R}_1 \bar { R}_2$. Then, the uncertainty set $\bar { R}_3  \mathcal X_{3}$ is equivalent to $(\bar { R}_1 \mathcal X_1 \bar { R}_2) \cdot \mathcal X_2$, i.e., $\bar { R}_3  \mathcal X_{3}=(\bar { R}_1 \mathcal X_1 \bar { R}_2) \cdot \mathcal X_2$.
\end{lemma}
\begin{proof}
    The first supporting lemma is 
\begin{lemma} \label{lemma:rotation_commute}
    Let $\mathcal X =\{{\rm Exp}(\theta { w}) \mid { w} \in \mathbb S^2, \theta \in [0, \theta']\}$, then for any ${ R} \in {\rm SO}(3)$, we have ${ R}\mathcal X=\mathcal X{ R}$.
\end{lemma}
\begin{proof}
    It is equivalent to prove ${ R}\mathcal X{ R}^\top=\mathcal X$. For any ${ X} \in \mathcal X$, let ${ X}'={ R}^\top{ X} { R}$. We have 
    \begin{equation*}
        {\rm dis}({ X}',{ I})={\rm dis}({ X},{ I})\leq \theta',
    \end{equation*}
    which implies ${ X}' \in \mathcal X$. Hence, ${ X}={ R}{ X}'{ R}^\top \in { R}\mathcal X{ R}^\top$, that is, $\mathcal X \subseteq { R}\mathcal X{ R}^\top$. 

    For any ${ X} \in { R}\mathcal X{ R}^\top$, there exists an ${ X}' \in \mathcal X$ such that ${ X}={ R}{ X}'{ R}^\top$. We have 
     \begin{equation*}
        {\rm dis}({ X},{ I})={\rm dis}({ X}',{ I})\leq \theta',
    \end{equation*}
    which implies ${ X} \in \mathcal X$. Therefore, ${ R}\mathcal X{ R}^\top \subseteq \mathcal X$, which completes the proof. 
\end{proof}

Based on Lemma~\ref{lemma:rotation_commute}, we can rewrite $(\bar { R}_1 \mathcal X_1 \bar { R}_2) \cdot \mathcal X_2$ as $\bar { R}_3 \mathcal X_1  \cdot \mathcal X_2$. The second supporting lemma is

\begin{lemma} \label{lemma:minkowski_product}
    It holds that $\mathcal X_{3}=\mathcal X_1 \cdot \mathcal X_2$.
\end{lemma}
\begin{proof}
    For any ${ X}_{3} \in \mathcal X_{3}$, we can express it as ${ X}_{3}={\rm Exp}(\theta_{3} { w})$ for some ${ w} \in \mathbb S^2, \theta_{3} \in [0, \theta_1'+\theta_2']$. Then, there exist $\theta_{1} \in [0, \theta_1']$ and $\theta_{2} \in [0, \theta_2']$ such that $\theta_{1}+\theta_{2}=\theta_{3}$. Let ${ X}_{1}={\rm Exp}(\theta_{1} { w})$ and ${ X}_{2}={\rm Exp}(\theta_{2} { w})$, we have 
    \begin{equation*}
        { X}_{1}{ X}_{2}={\rm Exp}((\theta_{1}+\theta_{2}) { w})={\rm Exp}(\theta_{3} { w})={ X}_{3},
    \end{equation*}
    which implies ${ X}_{3} \in \mathcal X_{1} \cdot \mathcal X_{2}$, and thus $\mathcal X_{3} \subseteq \mathcal X_{1} \cdot \mathcal X_{2}$.

    For any ${ X}_{3} \in \mathcal X_{1} \cdot \mathcal X_{2}$, there exist ${ X}_{1} \in \mathcal X_{1} $ and ${ X}_{2} \in \mathcal X_{2} $ such that ${ X}_{3}={ X}_{1}{ X}_{2}$. We have 
    \begin{align*}
        {\rm dis}({ X}_{1}{ X}_{2},{ I}) & = {\rm dis}({ X}_{1},{ X}_{2}^\top) \\
        & \leq {\rm dis}({ X}_{1},{ I}) + {\rm dis}({ I},{ X}_{2}^\top) \\
        & \leq \theta_1'+\theta_2',
    \end{align*}
    where the first inequality is based on the triangle inequality~\cite{lee2018introduction}. Therefore, ${ X}_{3}={ X}_{1}{ X}_{2} \in \mathcal X_{3}$, indicating $\mathcal X_{1} \cdot \mathcal X_{2} \subseteq \mathcal X_{3}$ and completing the proof.
\end{proof}
Based on Lemma~\ref{lemma:minkowski_product}, we have $\bar { R}_3  \mathcal X_{1} \cdot \mathcal X_{2}=\bar { R}_3  \mathcal X_{3}$, which completes the proof of Lemma~\ref{the:R_propagation}.
\end{proof}

The following lemma gives a polytopic overapproximation of the uncertainty set $\mathcal P(A_{t_3}, b_{t_3})$ over $t_3$. 

\begin{lemma} \label{lemma:indirect_t3}
   Given a matrix $A=[a_1,\ldots,a_M]^\top$ where each  $a_m$ is a normal vector, the  set $\mathcal P(A_{t_3}, b_{t_3})=\mathcal P(A\bar R_1^\top,b)+\mathcal P(A_{t_1},b_{t_1})$ is a guaranteed approximation for the uncertainty set on $t_3$, i.e., $\bar { R}_1 \mathcal X_1 \cdot \mathcal P(A_{t_2}, b_{t_2})+\mathcal P(A_{t_1}, b_{t_1})$. Here, 
\begin{equation} \label{eqn:max_t_opt2}
    [b]_m=\max_{{ v} \in \mathcal V_{2}} \left\{\max_{\theta_1 \in[0, \theta_1'], { w}\in \mathbb S^2} \langle {\rm Exp}(\theta_1 { w}) { a}_m, { v}\rangle\right\},
\end{equation}
where $\mathcal V_{2}$ is  the vertex set of $\mathcal P(A_{t_2}, b_{t_2})$. 
\end{lemma}
\begin{proof}
    It is straightforward that the polytope $\mathcal P(A,b)$ encloses the set $\mathcal X_1 \cdot \mathcal P(A_{t_2},b_{t_2})$ if $[b]_m$ is solving by 
    \begin{subequations} \label{eqn:max_t_opt}
      \begin{align}
             [b]_m=\max_{\theta_1, { w},  { t}_2} &~~{ a}_m^\top {\rm Exp}(\theta_1 { w}) { t}_2 \\
    {\rm s.t.} &~~ \theta_1 \in [0,\theta_1'], { w} \in \mathbb S^2, { t}_2 \in \mathcal P(A_{t_2}, b_{t_2}).
      \end{align}
    \end{subequations}
    Since $ w$ spans over the whole $\mathbb S^2$, we can rewrite the objective in~\eqref{eqn:max_t_opt} as ${ a}_m^\top {\rm Exp}(\theta_1 { w})^\top { t}_2=\langle {\rm Exp}(\theta_1 { w}) { a}_m, { t}_2\rangle$. In addition, $\langle {\rm Exp}(\theta_1 { w}) { a}_m, { t}_2\rangle$ achieves its maximum value at one of the vertices of $\mathcal P(A_{t_2}, b_{t_2})$. Therefore, Problem~\eqref{eqn:max_t_opt} is equivalent to~\eqref{eqn:max_t_opt2}. As a result, the uncertainty set $\mathcal P(A_{t_3}, b_{t_3})=\mathcal P(A\bar R_1^\top,b)+\mathcal P(A_{t_1},b_{t_1})$ is a guaranteed approximation for the uncertainty set $\bar { R}_1 \mathcal X_1 \cdot \mathcal P(A_{t_2}, b_{t_2})+\mathcal P(A_{t_1}, b_{t_1})$, which completes the proof.
\end{proof}

As illustrated in Figure~\ref{fig:maximum_inner_product}, the inner maximization in~\eqref{eqn:max_t_opt2} has a closed-form solution:
\begin{equation} \label{eqn:inner_value}
\begin{split}
    & \max_{\theta_1 \in[0, \theta_1'], { w}\in \mathbb S^2} \langle {\rm Exp}(\theta_1 { w}) { a}_m, { v}\rangle \\
    & =\begin{cases}
        \|{ v}\|, ~ \text{if} ~~\angle({ a}_m, { v}) \leq \theta_1',  \\
        \|{ v}\|\cos(\angle({ a}_m, { v})-\theta_1'), ~ \text{if} ~~\angle({ a}_m, { v}) > \theta_1',
    \end{cases}
\end{split}
\end{equation}
where $\angle({ a}_m, { v})$ denotes the angle between $a_m$ and $v$. This enables efficient set computation. 

The above result ensures a \emph{guaranteed} result for the uncertainty set on $T_3$ via pose compound. 
\begin{theorem} \label{the:indirect_conservatism}
    The set $\{T_3 \in {\rm SE}(3) \mid R_3 \in \bar R_3 \mathcal X_3,t_3 \in \mathcal P(A_{t_3}, b_{t_3})\}$ is a guaranteed approximation for the set $\mathcal T_3$.
\end{theorem}
\begin{proof}
    Based on the property of set projection and Proposition 1, we have $\mathcal T_1 \subseteq \{T_1 \in {\rm SE}(3) \mid R_1 \in \bar R_1 \mathcal X_1, t_1 \in \mathcal P(A_{t_1}, b_{t_1})\}$ and $\mathcal T_2 \subseteq \{T_2 \in {\rm SE}(3) \mid R_2 \in \bar R_2 \mathcal X_2, t_2 \in \mathcal P(A_{t_2}, b_{t_2})\}$. Therefore, it holds from $T_3=T_1T_2$ that 
    \begin{align*}
        \mathcal T_3 \subseteq \{T_3 \in {\rm SE}(3) \mid & R_3 \in (\bar { R}_1 \mathcal X_1 \bar { R}_2) \cdot \mathcal X_2, \\
        & ~t_3 \in \bar { R}_1 \mathcal X_1 \cdot \mathcal P(A_{t_2}, b_{t_2})+\mathcal P(A_{t_1}, b_{t_1})\}.
    \end{align*}
    Finally, according to Lemmas~\ref{the:R_propagation} and~\ref{lemma:indirect_t3}, we have $\mathcal T_3 \subseteq \{T_3 \in {\rm SE}(3) \mid  R_3 \in \bar R_3 \mathcal X_3, 
         t_3 \in \mathcal P(A_{t_3}, b_{t_3})\}$, which completes the proof.
\end{proof}

Our indirect pose uncertainty compound algorithm is summarized in Algorithm~\ref{algrithm:indirect_compound}.
\begin{algorithm}
	\caption{Indirect pose compound}
	\label{algrithm:indirect_compound}
	\begin{algorithmic}[1]
		\Statex \textbf{Input}: Uncertainty $\bar R_1 \mathcal X_1$ for $R_1$, $\mathcal P(A_{t_1}, b_{t_1})$ for $t_1$, $\bar R_2 \mathcal X_2$ for $R_2$, and $\mathcal P(A_{t_2}, b_{t_2})$ for $t_2$.
        \Statex \textbf{Output}: Uncertainty $\bar R_3 \mathcal X_3$ for $R_3$ and $\mathcal P(A_{t_3}, b_{t_3})$ for $t_3$.
       %  \State Obtain $\mathcal P(A_{r_1}, b_{r_1})$, $\mathcal P(A_{t_1}, b_{t_1})$ from $\mathcal P(A_1,b_1)$ and $\mathcal P(A_{r_2}, b_{r_2})$, $\mathcal P(A_{t_2}, b_{t_2})$ from $\mathcal P(A_2,b_2)$ by set projection;
       % \State Obtain $\bar R_1 \mathcal X_1$ from $\mathcal P(A_{r_1}, b_{r_1})$ and $\bar R_2 \mathcal X_2$ from $\mathcal P(A_{r_2}, b_{r_2})$ by solving~\eqref{eqn:max_ellipsoid_in_polytope} and~\eqref{eqn:max_theta'};
       \State Set $\bar R_3=\bar R_1 \bar R_2$ and $\mathcal X_{3} =\mathcal X_{1} \cdot \mathcal X_{2}=\{{\rm Exp}(\theta_{3} { w}) \mid { w} \in \mathbb S^2, \theta_{3} \in [0, \theta_1'+\theta_2']\}$;
       \State Fix $A=[a_1,\ldots,a_M]^\top$, where each $a_m$ is a unit normal;
       \For{$m=1:M$}
\State Solve Problem~\eqref{eqn:max_t_opt2} and obtain $[b]_m$;
       \EndFor
       \State Set $\mathcal P(A_{t_3}, b_{t_3})=\mathcal P(A\bar R_1^\top,b)+\mathcal P(A_{t_1}, b_{t_1})$.
	\end{algorithmic}
\end{algorithm}

% Since the Minkowski sum of two convex polytopes is also a convex polytope~\cite{schneider2013convex}, the uncertainty set for ${ t}_3$ is a convex polytope. Specifically, it can be efficiently calculated as follows. Denote the vertex sets of $\mathcal P_1$ and $\mathcal P_3$ as $\mathcal V_{\mathcal P_1}$ and $\mathcal V_{\mathcal P_3}$, respectively. The Minkowski sum of the two polytopes is given by 
% \begin{equation*}
%     \mathcal P_1+\mathcal P_3={\rm conv}\left(\{{ v}_1+{ v}_3 \mid { v}_1 \in \mathcal V_{\mathcal P_1}, { v}_3 \in \mathcal V_{\mathcal P_3}\}\right).
% \end{equation*}

\begin{figure}[!htbp]
    \centering
    \includegraphics[width=0.7\linewidth]{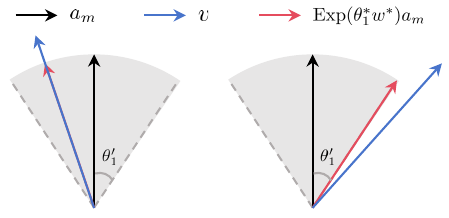}
    \caption{Illustration of Eq.~\eqref{eqn:inner_value}. When $\angle({ a}_m,{ v})\leq \theta_1'$, the optimal rotation ${\rm Exp}(\theta_1^* { w}^*)$ aligns ${ a}_m$ to the direction of $ v$; when $\angle({ a}_m,{ v})> \theta_1'$, the optimal rotation ${\rm Exp}(\theta_1^* { w}^*)$ rotates ${ a}_m$ along the optimal axis ${ w}^*$ by $\theta_1^*=\theta_1'$ such that the rotated vector has an angle of $\angle({ a}_m,{ v})- \theta_1'$ relative to $ v$.}
    \label{fig:maximum_inner_product}
\end{figure}

\section{Uncertainty Quantification for SLAM} \label{sec:SLAM_oometry}

In the section, we construct a guaranteed SLAM UQ pipeline based on the basic UQ primitives in Sec.\ref{sec:basic_uncertainty_propaga}. Here we focus on 3D-3D landmark-based SLAM (joint pose-landmark estimation)~\cite{zhang2014loam,campos2021orb} as a representative formulation.

%our theoretical framework remains applicable to broader SLAM systems.

%We approach a provably guaranteed SLAM UQ pipeline based on the basic UQ in Sec.\ref{sec:basic_uncertainty_propaga}. In this work, we focus on landmark-based SLAM, where robot poses and landmark positions are jointly estimated, as widely studied in the literature~\cite{zhang2014loam,campos2021orb}. It serves as a representative formulation, while the theoretical modules developed herein may, in principle, be applicable to a broader class of SLAM systems.
\subsection{Algorithm Overview} \label{subsec:algorithm_overview}
%The SLAM problem admits numerous map representations and solutions.

An overview of our algorithm is shown in Figure~\ref{fig:guaranteed_SLAM}.
The robot is assumed to possess metric-scale 3D perception capabilities, enabled by sensors such as stereo cameras, RGB-D cameras, or LiDARs, for local point cloud generation.
The proposed algorithm consists of the following key components:
\begin{enumerate}
    \item \emph{Local point cloud generation}: extracting point clouds in the local frame from the raw sensor data.
    \item \emph{Localization}: estimating the global poses of the robot.
    \item \emph{Mapping}: estimating the global landmark positions.
    \item \emph{Smoothing}: leveraging loop closures to correct the estimated robot's trajectory and map landmarks.
\end{enumerate}
As illustrated in Figure~\ref{fig:guaranteed_SLAM}, the localization module can be implemented as a \textbf{relative framework}, which estimates relative poses between consecutive frames and compound it with the previous global localization, or a \textbf{global framework}, which executes direct pose tracking by establishing correspondences between the local point cloud and the global map.

%The localization module can be implemented in two ways.   First, one can estimate the relative pose between consecutive frames and then compound it with the previous global localization result. Second, one can directly execute global pose tracking by establishing point correspondences between the local point cloud and the global map. SLAM frameworks following these two approaches are referred to as the \textbf{relative framework} and \textbf{global framework}, respectively.

Next, we will utilize CP and the basic UQ in Sec.\ref{sec:basic_uncertainty_propaga} to illustrate how to quantify uncertainty in each component.

\subsection{Uncertainty Quantification for Local Point Cloud} \label{subsec:CP_for_local_points}
We employ CP to provide polytopic uncertainty sets for local point clouds.
When a new frame arrives, the robot generates a local point cloud using onboard sensors. 
Let ${ p}_i$ be the coordinates of the $i$-th local 3D point.
% We need to give a polytopic uncertainty set $\mathcal P(A_i,b_i)=\{p_i \mid { A}_i { p}_i \leq { b}_i\}$ for ${ p}_i$.
The objective is to characterize its uncertainty by constructing a polytopic uncertainty set $\mathcal P(A_i,b_i)=\{p_i \mid { A}_i { p}_i \leq { b}_i\}$. 
The derivation of ellipsoidal uncertainty sets for points directly measured by LiDARs or RGB-D cameras using CP has been studied in~\cite{gao2024closure}, which can be easily adapted for 
polytopic uncertainties.
% Here, we will derive an ellipsoidal uncertainty set for \emph{stereo camera}-generated points with CP, followed by polytopic uncertainty approximation.
Here, we further investigate the uncertainty modeling of 3D points reconstructed from stereo image measurements and derive their ellipsoidal uncertainty sets via CP, followed by polytopic uncertainty approximation.

Let $\left\{{ u}_l,{ v}_l,{ p}_l^o \right\}_{l=1}^L$ be the calibration set of CP, where ${ u}_l$ and ${ v}_l$ are the normalized image coordinates of a pair of matched features in the left and right images, respectively, and ${ p}_l^o$ is the true 3D coordinates of the point in the left camera frame. Let $({ R}_{\mathcal L}^{\mathcal R}, { t}_{\mathcal L}^{\mathcal R})$ be the stereo baseline. Following~\cite{zeng2025bias}, the triangulated 3D point is given by  
\begin{equation*}
 \hat { p}_l=({ H}_l^\top{ H}_l)^{-1} { H}_l^\top { y}_l, \   { H}_l=\begin{bmatrix}
        { u}_l^{h \wedge} \\
        { v}_l^{h \wedge} { R}_{\mathcal L}^{\mathcal R}
    \end{bmatrix},~~ y_l=\begin{bmatrix}
            { 0}_3 \\
            -{ v}_l^{h \wedge} { t}_{\mathcal L}^{\mathcal R}
        \end{bmatrix},
\end{equation*}
and ${ u}_l^{h}$ and ${ v}_l^{h}$ are the homogeneous coordinates. In addition, a scaled covariance ${ \Sigma}_l$ of $\hat { p}_l$ can be given as ${ \Sigma}_l={ J}_l  { J}_l^\top$, where ${ J}_l$ is the Jocabian matrix of $\hat p_l$ with respect to ${\rm col}(u_l,v_l)$.
Then, we can construct the \emph{nonconformity score} as 
\begin{equation} \label{eqn:nonconfomity_score}
    s_l=\left\|\hat { p}_l-{ p}_l^o\right\|_{{ \Sigma}_l}.
\end{equation}
\begin{remark} \label{remark:scaled_covariance}
    It is well known that the uncertainty of a triangulated point can vary substantially across points, largely due to differences in depth~\cite{zeng2025bias,cvivsic2022soft2}. We therefore introduce a scaled covariance $\Sigma_l$ to normalize the uncertainty of each point, so that the resulting quantities \(s_l\), \(l=1,\ldots,L\), better satisfy the exchangeability assumption under which CP is valid. 
\end{remark}
For the tested data $({ u}_i,{ v}_i,{ p}_i^o )$, it holds from Lemma~\ref{lemma:CP} that 
    \begin{equation*} 
        {\rm Prob}(\left\|\hat { p}_i-{ p}_i^o\right\|_{{ \Sigma}_i} \leq C(s_1,\ldots,s_L)) \geq 1-\delta.
    \end{equation*}
Hence, the ellipsoidal uncertainty set for $p_i$ is $\left\|\hat { p}_i-{ p}_i\right\|_{{ \Sigma}_i} \leq C(s_1,\ldots,s_L)$. 

Given the ellipsoid $\left\|\hat { p}_i-{ p}_i\right\|_{{ \Sigma}_i} \leq C(s_1,\ldots,s_L)$, by fixing the polytope unit normals, we can analytically give the minimum polytope that encloses the ellipsoid~\cite{boyd2004convex}. 
Specifically, we fix the unit normals $A_i=[a_1,\ldots,a_M]^\top$ \emph{a prior}. For each $a_m$, the tightest offset $[b_i]_m$ is given by 
\begin{equation*}
    [b_i]_m=a_m^\top \hat { p}_i+C(s_1,\ldots,s_L)\sqrt{a_m^\top \Sigma_i a_m}.
\end{equation*}
Finally, the polytopic uncertainty set for $p_i$ is $\mathcal P(A_i,b_i)$.
Since the polytope encloses the ellipsoid, it is straightforward that 
 \begin{equation*} 
        {\rm Prob}({ p}_i^o \in \mathcal P(A_i,b_i)) \geq 1-\delta.
    \end{equation*}

\subsection{Localization} \label{subsec:localization}
\subsubsection{Relative pose tracking followed by pose compound} 
When a new frame arrives, we perform relative pose tracking, i.e., estimating the relative pose between the $k$-th frame and the previous frame, denoted by $T_{k}^{k-1}$. 
After feature matching, we have point correspondences $\{p_{k,i},q_{k,i}\}_{i=1}^{n_k}$, where $p_{k,i}$ and $q_{k,i}$ are coordinates in the $k$-th frame and the previous frame, respectively. They have the relationship $q_{k,i}=\phi(T_k^{k-1},p_{k,i})$. 
The uncertainty sets $\mathcal P({ A}_{p_{k,i}},{ b}_{p_{k,i}})$ for $p_{k,i}$ and $\mathcal P({ A}_{q_{k,i}},{ b}_{q_{k,i}})$ for $q_{k,i}$ are provided by CP. 
Applying the backward UQ in Sec.\ref{subsec:point_to_pose} gives  
%T_k^{k-1}$ Quantifying the uncertainty of $T_k^{k-1}$ involves backward UQ. Hence, we can apply Algorithm~\ref{algrithm:backward_propagation} and obtain the uncertainty set 
$\mathcal T_k^{k-1}=\{T_k^{k-1}\in {\rm SE}(3) \mid H_k^{k-1} {\rm vec}(T_k^{k-1})\leq d_k^{k-1}\}$.
Note that $T_{k}=T_{k-1} T_{k}^{k-1}$. 
The uncertainty set $\mathcal T_{k-1}=\{T_{k-1} \in {\rm SE}(3) \mid H_{k-1} {\rm vec}(T_{k-1})\leq d_{k-1}\}$ for $T_{k-1}$ is stored by the previous frame localization. 
% Quantifying the uncertainty of $T_{k}$ involves pose compound. 
Quantifying the uncertainty of $T_k$ requires UQ through pose compound. We can apply either direct or indirect rigid transformation uncertainty compound in Sec.\ref{subsec:pose_compound} to compute the  set $\mathcal T_{k}=\{T_{k} \in {\rm SE}(3) \mid H_{k} {\rm vec}(T_{k})\leq d_{k}\}$ for $T_{k}$.

\subsubsection{Global pose tracking}
As an alternative, we can directly estimate the global pose by tracking points in the global map. After feature matching between consecutive frames, we can obtain point correspondences $\{p_{k,i},q_{k,i}\}_{i=1}^{n_k}$, where $p_{k,i}$ and $q_{k,i}$ are coordinates in the $k$-th frame and the global frame, respectively. They have the relationship $q_{k,i}=\phi(T_k,p_{k,i})$. The uncertainty sets $\mathcal P({ A}_{p_{k,i}},{ b}_{p_{k,i}})$ for $p_{k,i}$ and $\mathcal P({ A}_{q_{k,i}},{ b}_{q_{k,i}})$ for $q_{k,i}$ are provided by CP and the mapping function, respectively. 
We can utilize our backward UQ in Sec.\ref{subsec:point_to_pose} for global pose estimation and compute $\mathcal T_k=\{T_k\in {\rm SE}(3) \mid H_k {\rm vec}(T_k)\leq d_k\}$.

\subsection{Mapping} \label{subsec:mapping}
Once an arriving frame has been localized, the mapping module is invoked to incorporate the newly observed local points into the global map.
% When the frame is localized, we also activate the mapping component, i.e., inserting the newly observed local points to the global map. 
Let ${ p}_{k,i}^{\rm Ne}$ and ${ q}_{k,i}^{\rm Ne}$ be the coordinates of the $i$-th newly observed point in the $k$-th robot frame and the global frame, respectively, which is related by ${ q}_{k,i}^{\rm Ne}=\phi(T_{k},{ p}_{k,i}^{\rm Ne})$. 
The uncertainty set $\mathcal P({ A}_{p_{k,i}}^{\rm Ne},{ b}_{p_{k,i}}^{\rm Ne})$ for ${ p}_{k,i}^{\rm Ne}$ is provided by CP and the uncertainty set $\mathcal T_{k}=\{T_{k} \in {\rm SE}(3) \mid H_{k} {\rm vec}(T_{k})\leq d_{k}\}$ for $T_{k}$ is provided by the localization function. 
%Quantifying the uncertainty of ${ q}_{k,i}^{\rm Ne}$ involves forward UQ. 
We can apply the forward propagation in Sec.\ref{subsec:pose_to_point} to compute the uncertainty set $\mathcal P({ A}_{q_{k,i}}^{\rm Ne},{ b}_{q_{k,i}}^{\rm Ne})$ for ${ q}_{k,i}^{\rm Ne}$.

\subsection{Smoothing} \label{subsec:loop_closure}
Upon loop closure detection, the trajectory and map are jointly optimized to mitigate accumulated drift and improve global consistency.
% If a loop closure is detected, the trajectory and map drift can be corrected to enhance global consistency. 
As shown in Figure~\ref{fig:guaranteed_SLAM}, the frames that involve in the loop closures may additionally link to more map points, denoted by dashed factors in the figure. Then, we jointly optimize all frame poses and map points within the loop closures.
Specifically, we adopt an alternating optimization scheme: at each iteration, we first update the map points and subsequently refine the frame poses. For each map point, we incorporate all associated frames and execute the forward UQ in Sec.\ref{subsec:pose_to_point} to refine the map uncertainty set. To update each frame pose, we utilize all visible points and apply the backward UQ in Sec.\ref{subsec:point_to_pose}. Because both the pose and map point sets decrease monotonically and are bounded below by the empty set, the iteration is guaranteed to converge.
As detailed in the supplementary pseudocode, we employ a maximum of three iterations or stop when the uncertainty set refinement becomes negligible.
It is noteworthy that if an incorrect loop closure is introduced, the uncertainty sets may shrink incorrectly, possibly losing containment or even becoming empty. Note that empty uncertainty sets can indicate an incorrect loop closure and thus support rejection.

The pseudo code for the whole pipeline is provided in Algorithm~\ref{algrithm:SLAM}.
\begin{algorithm}
	\caption{Guaranteed SLAM UQ}
	\label{algrithm:SLAM}
	\begin{algorithmic}[1]
		\State  \emph{\# Initialization}
        \State $k=1, H_{k}=[1,-1]^\top \otimes I_{12},$  $d_{k}=[1,-1]^\top \otimes[{\rm vec}(I_3)^\top, 0_3^\top]^\top$, $\mathcal T_{k}=\{T \in {\rm SE}(3) \mid H_{k} x(T)\leq d_{k}\}$;
       \State Use CP to obtain $\mathcal P({ A}_{p_{k,i}}^{\rm Ne},{ b}_{p_{k,i}}^{\rm Ne})$ for $p_{k,i}$;
        \State Apply Algorithm~\ref{algrithm:forward_propagation} on $\mathcal T_{k}$ and $\mathcal P({ A}_{p_{k,i}}^{\rm Ne},{ b}_{p_{k,i}}^{\rm Ne})$ to obtain $\mathcal P({ A}_{q_{k,i}}^{\rm Ne},{ b}_{q_{k,i}}^{\rm Ne})$;
        \While{not stopped}
        \State $k=k+1$;
\State \emph{\# Localization}
\State \emph{\# Option 1: relative framework}
\State Execute feature matching to obtain point correspondences $\{p_{k,i},q_{k,i}\}_{i=1}^{n_k}$;
\State Use CP to obtain $\mathcal P({ A}_{p_{k,i}},{ b}_{p_{k,i}})$ for $p_{k,i}$ and $\mathcal P({ A}_{q_{k,i}},{ b}_{q_{k,i}})$ for $q_{k,i}$;
\State Apply Algorithm~\ref{algrithm:backward_propagation} on $\mathcal P({ A}_{p_{k,i}},{ b}_{p_{k,i}})$ and $\mathcal P({ A}_{q_{k,i}},{ b}_{q_{k,i}})$ to obtain $\mathcal T_k^{k-1}$;
\State Apply Algorithm~\ref{algrithm:direct_compound} or Algorithm~\ref{algrithm:indirect_compound} on $\mathcal T_{k}^{k-1}$ and $\mathcal T_{k-1}$ to obtain $\mathcal T_{k}$;
\State \emph{\# Option 2: global framework}
\State Inquire the global map to obtain $\mathcal P({ A}_{q_{k,i}},{ b}_{q_{k,i}})$;
\State Apply Algorithm~\ref{algrithm:backward_propagation} on $\mathcal P({ A}_{p_{k,i}},{ b}_{p_{k,i}})$ and $\mathcal P({ A}_{q_{k,i}},{ b}_{q_{k,i}})$ to obtain $\mathcal T_k$;
\State \emph{\# Mapping}
\State Apply Algorithm~\ref{algrithm:forward_propagation} on $\mathcal T_{k}$ and $\mathcal P({ A}_{p_{k,i}}^{\rm Ne},{ b}_{p_{k,i}}^{\rm Ne})$ to obtain $\mathcal P({ A}_{q_{k,i}}^{\rm Ne},{ b}_{q_{k,i}}^{\rm Ne})$;
\State \emph{\# Loop-closure smoothing}
\If{a loop closure is detected}
\For{$l=1:{\rm max\_iter}$}
\State Apply Algorithm~\ref{algrithm:forward_propagation} to update the uncertainty of the map within the loop closure;
\State Apply Algorithm~\ref{algrithm:backward_propagation} to update the uncertainty of poses within the loop closure;
\EndFor
\EndIf
        \EndWhile
	\end{algorithmic}
\end{algorithm}

\begin{remark}
    Our algorithm relies on the exchangeability assumption of calibration and test data to construct statistically valid uncertainty sets via CP. In SLAM, this assumption holds when both datasets originate from the same underlying process, encompassing the sensor model, front-end pipeline, and environmental distribution. Outliers do not inherently violate exchangeability, provided the contamination mechanism remains consistent across both datasets; however, higher outlier rates typically yield looser uncertainty sets. As the core building blocks in our SLAM UQ pipeline, three primitives (forward UQ, backward UQ, and pose compound) require no specialized assumptions.  Our theoretical derivation rigorously shows that if the input uncertainty sets are valid, then the outputs are guaranteed to contain the true values. While we integrate these primitives with CP here, they remain agnostic to the specific local point cloud calibration method used.
\end{remark}

\section{Simulation and Experiment} \label{sec:sim_and_exp}

\subsection{Basic Uncertainty Quantification Simulation} \label{subsec:basic_simulation}

\subsubsection{Forward uncertainty quantification} To test our forward UQ  in Sec.\ref{subsec:pose_to_point}, six random trials are performed. In each trial, we randomly generate a pose $T \in {\rm SE}(3)$ and a point $p \in \mathbb R^3$. In addition, we impose a box uncertainty on $p$, whose size is drawn from $[0.08,0.2]$m. We also endow $x(T)$ with a box uncertainty, where the first 9 dimensions and the last 3 dimensions take values from $[0.01,0.02]$ and $[0.05,0.1]$m, respectively.
We generate the polytope matrix $A_2$ for $q$ by enforcing a \(45^\circ\) dihedral angle between each pair of adjacent facets.
To verify the conservatism of our algorithm, we randomly sample $1000$ points from the original set $\mathcal Q$ for $q=\phi(T,p)$ and test whether they fall inside our estimated set $\mathcal P(A_2,\hat b_2)$. 
The result is that all samples belong to the estimated set, and Figure~\ref{fig:sim_forward_propagation} visualizes one of the trials, which validates our guaranteed UQ.

\begin{figure}[!htbp]
    \centering
     \begin{subfigure}[t]{0.49\linewidth}
        \centering
        \includegraphics[width=\linewidth]{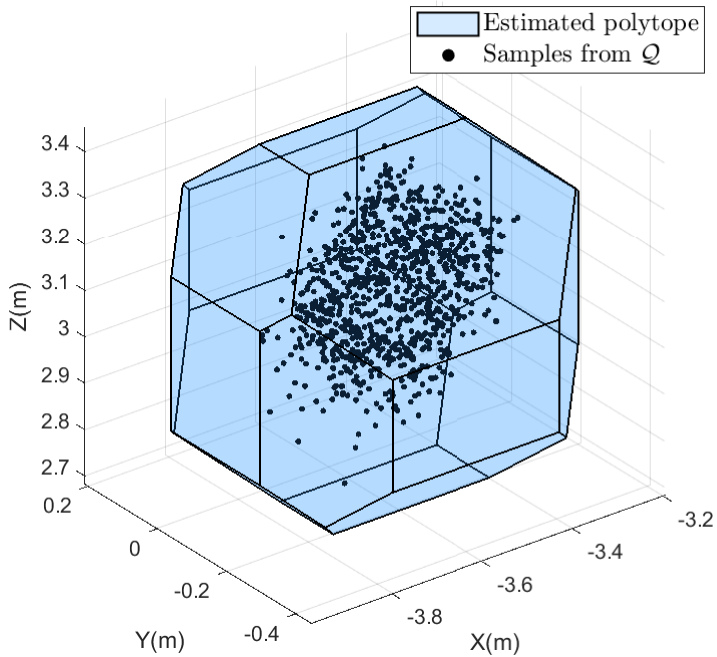}
        \caption{Trial 1}
        \label{fig:sim_forward_propagation1}
    \end{subfigure}
         \begin{subfigure}[t]{0.49\linewidth}
        \centering
        \includegraphics[width=\linewidth]{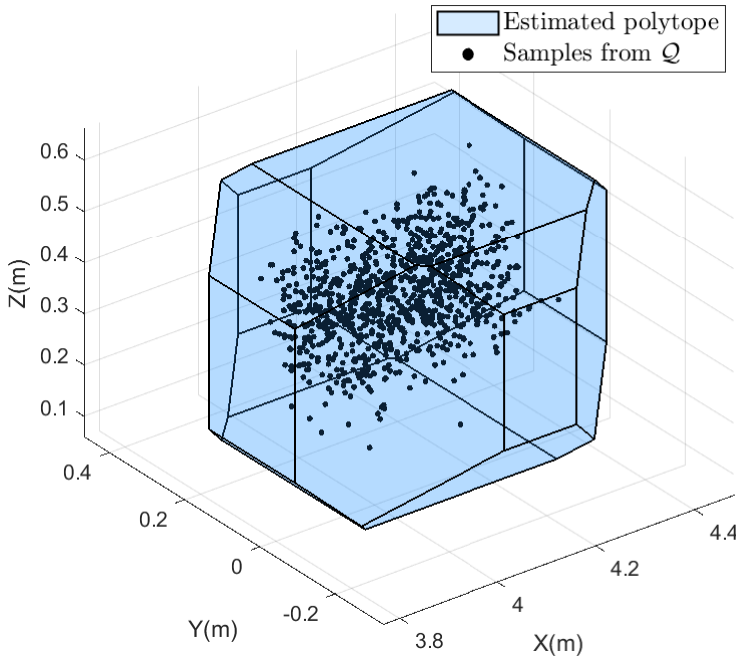}
        \caption{Trial 2}
        \label{fig:sim_forward_propagation2}
    \end{subfigure}
    \caption{Forward propagation results. Samples from the original set $\mathcal Q$ all locate inside our estimated polytope by Algorithm~\ref{algrithm:forward_propagation}.}
    \label{fig:sim_forward_propagation}
\end{figure}

\subsubsection{Backward uncertainty quantification} In this part, we evaluate our backward UQ in Sec.\ref{subsec:point_to_pose} and compare it with the pose UQ methods \texttt{CLOSURE}~\cite{gao2024closure} and \texttt{GRCC}~\cite{tang2024uncertainty}. \texttt{CLOSURE} produces an inner approximation of the original set, whereas \texttt{GRCC} gives an outer approximation. Consider point correspondences $\{p_i,q_i\}_{i=1}^n$ satisfying $q_i=Rp_i+t$. The compared methods model uncertainty directly on the residuals $q_i - Rp_i - t$ using ellipsoidal sets, rather than imposing separate uncertainties on both $p_i$ and $q_i$. To match their setting, we treat $p_i$ as exact and assign an ellipsoidal uncertainty set to each $q_i$. 
For our method, we first generate a polytope matrix  by enforcing a \(45^\circ\) dihedral angle between each pair of adjacent facets and compute a minimum-volume enclosing polytope for each ellipsoid. Then, we perform pose uncertainty quantification using these more conservative polytopes. 
We have conducted six trials with randomly sampled poses and ellipsoid semi-axis lengths ranging from $0.01$m to $0.05$m. For \texttt{CLOSURE}, we adopt the default reasonable parameter configuration provided in~\cite{gao2024closure}, without tuning its hyperparameters in the random walk sampling.
The conservatism test results are presented in Table~\ref{tab:backward_propagation}. All pose samples drawn from the original set $\mathcal T$ lie within the sets given by our algorithm and \texttt{GRCC}, validating the guaranteed property of the two methods. However, since \texttt{CLOSURE} only gives an inner approximation of $\mathcal T$, some samples fall outside its estimated set. In addition, the conservatism of \texttt{CLOSURE} varies a lot in different random trails, resulting in inconsistent containment rates.

We note that \texttt{CLOSURE} and \texttt{GRCC} quantify translation and rotation uncertainties separately. The translation uncertainty is represented by a Euclidean ball $\{t \in \mathbb R^3 \mid \|t-\bar t\|\leq r_t\}$, and the rotation uncertainty by a geodesic ball $\{R \in {\rm SO}(3) \mid {\rm dis}(R,\bar R) \leq r_R\}$. 
To enable a direct comparison, we follow the convention used by the baselines by projecting our estimated polytopes $\mathcal{P}(H, d)$ onto the rotation and translation polytopes. We then enclose each  rotation polytope within a geodesic ball; note that this representation tends to favor the baseline methods by changing our polytopic bounds.
For visualization, we map rotation matrices to $\mathbb R^3$ via the ${\rm Log}$ operation and plot the rotation uncertainty as the Euclidean ball $\{s \in \mathbb R^3 \mid \|s-{\rm Log(\bar R)}\|\leq r_R\}$.
The translation and rotation uncertainties in the first trial are visualized in Figure~\ref{fig:sim_backward_propagation_pro}.
Our projected results are more conservative than \texttt{CLOSURE} for both translation and rotation and more conservative than \texttt{GRCC} for rotation. 
For complete pose uncertainty visualization, we adopt a sampling-based method. Specifically, we draw $1000$ poses from the estimated set, transform a ball of radius $0.01$m using these poses, and visualize the union of the transformed balls. The ball center is set as $[1~0~0]^\top, [0~1~0]^\top, [0~0~1]^\top, [1~1~1]^\top$, respectively, and the results in the first trial are shown in Figure~\ref{fig:sim_backward_propagation}. Interestingly, unlike Figure~\ref{fig:sim_backward_propagation_pro}, the pose uncertainty of our method is comparable to \texttt{CLOSURE} and tighter than \texttt{GRCC}. 
The reason  may be attributed to the conservatism (in Figure~\ref{fig:sim_backward_propagation_pro}) introduced by the comparison adaptation to ensure a direct comparison.

\begin{figure}[!htbp]
    \centering
     \begin{subfigure}[t]{0.49\linewidth}
        \centering
        \includegraphics[width=\linewidth]{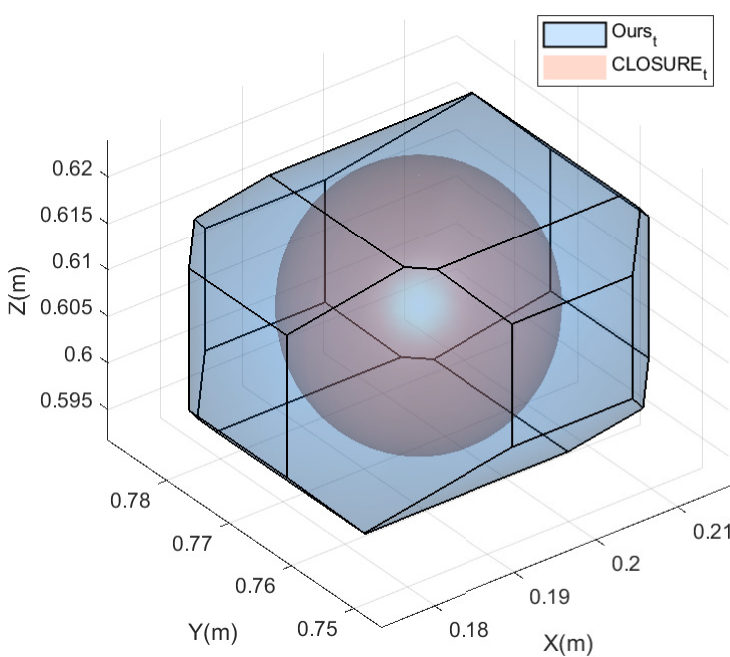}
        \label{fig:sim_backward_propagation_t}
    \end{subfigure}
         \begin{subfigure}[t]{0.49\linewidth}
        \centering
        \includegraphics[width=\linewidth]{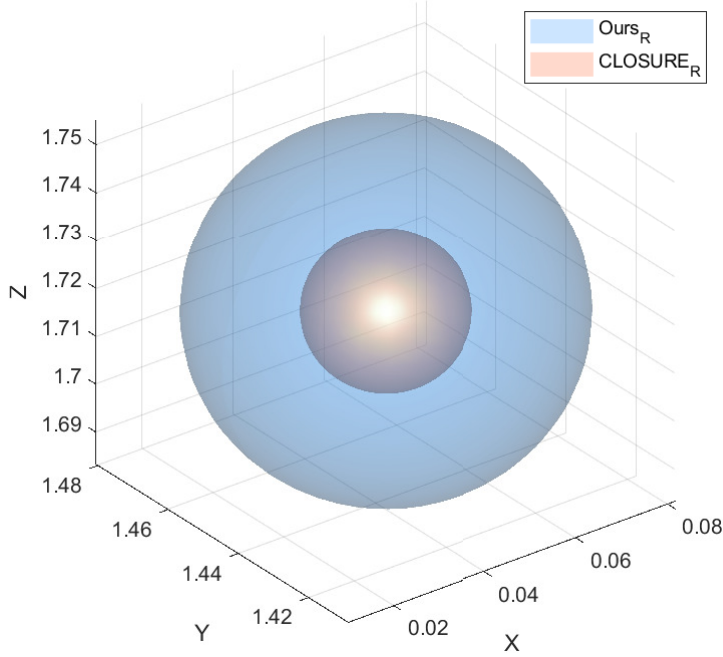}
        \label{fig:sim_backward_propagation_R}
    \end{subfigure}
     \begin{subfigure}[t]{0.49\linewidth}
        \centering
        \includegraphics[width=\linewidth]{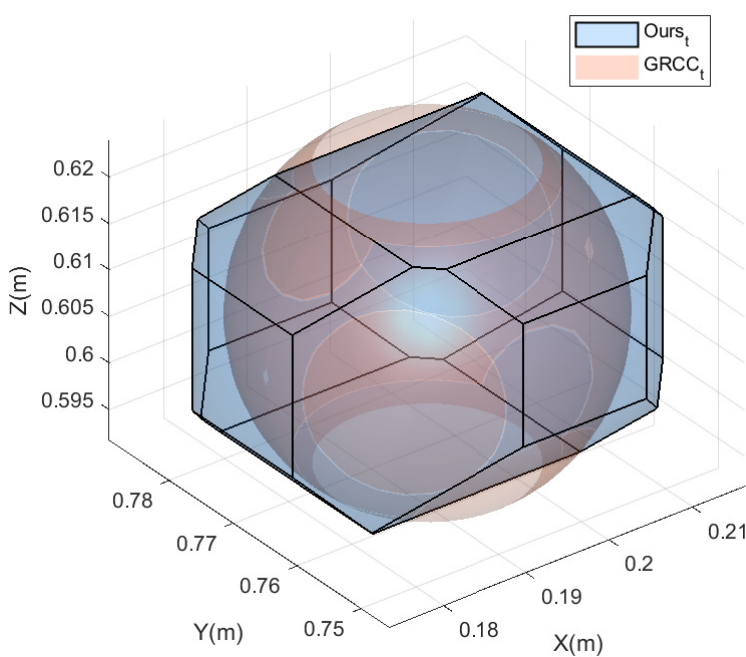}
        \label{fig:sim_backward_propagation_t2}
    \end{subfigure}
         \begin{subfigure}[t]{0.49\linewidth}
        \centering
        \includegraphics[width=\linewidth]{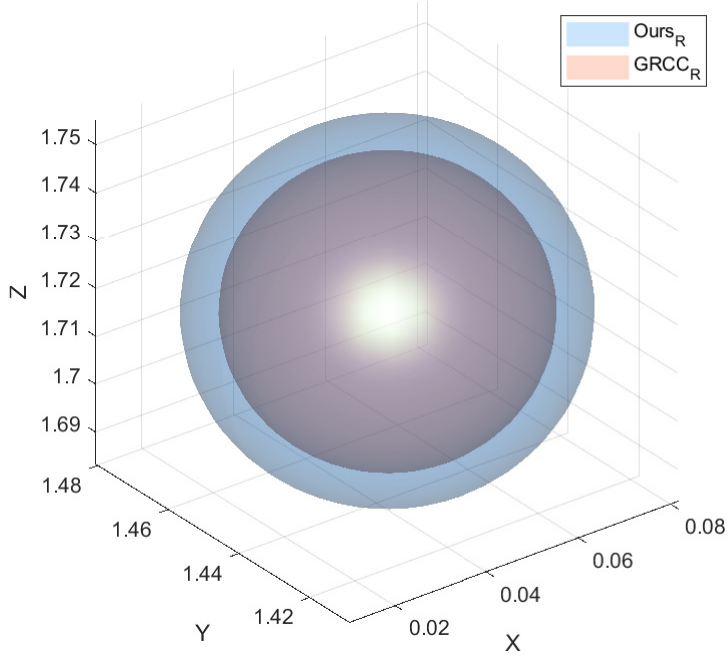}
        \label{fig:sim_backward_propagation_R2}
    \end{subfigure}
    \caption{Visualization of uncertainties for translation  (left) and rotation (right) in backward UQ.}
    \label{fig:sim_backward_propagation_pro}
\end{figure}

\begin{figure*}[!htbp]
    \centering
    \begin{subfigure}[t]{0.24\linewidth}
        \centering
        \includegraphics[width=\linewidth]{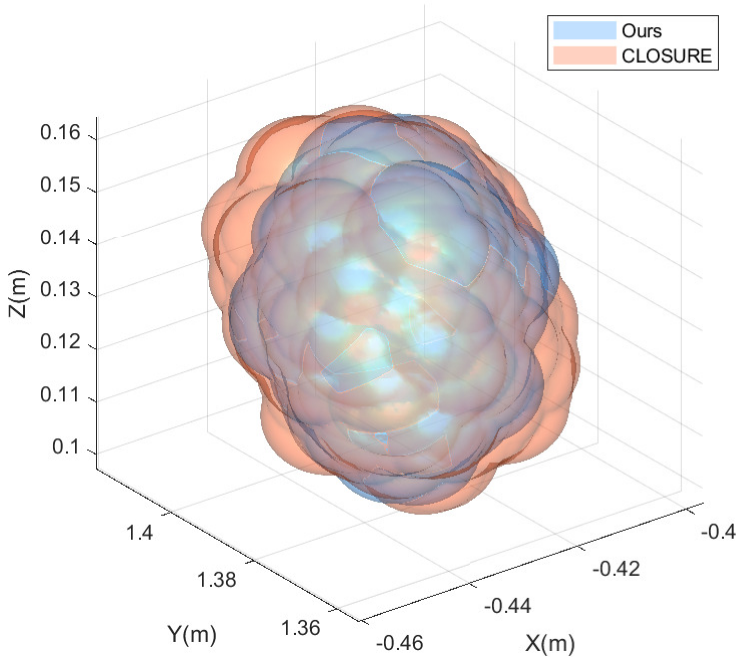}
        \caption{Ball center: $[1 ~0 ~0]^\top$}
        \label{fig:sim_backward_propagation_a}
    \end{subfigure}\hfill
    \begin{subfigure}[t]{0.24\linewidth}
        \centering
        \includegraphics[width=\linewidth]{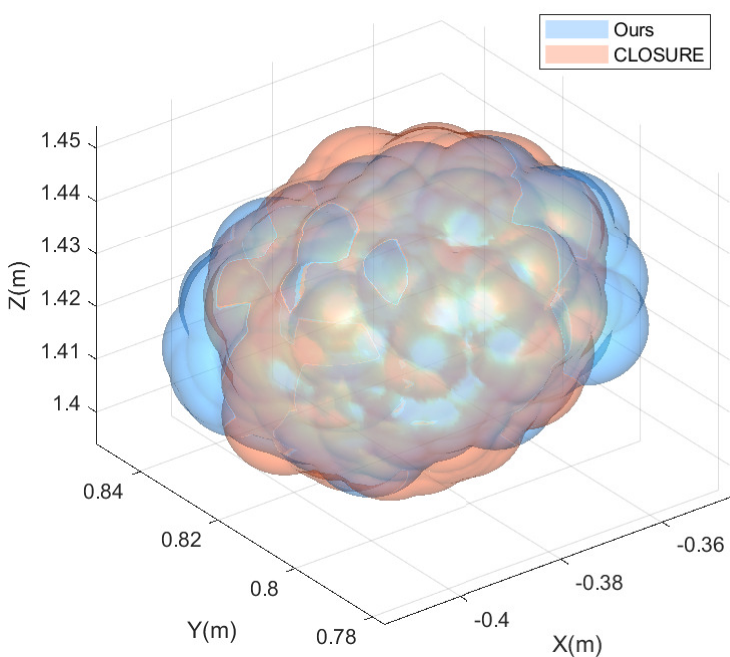}
        \caption{Ball center: $[0 ~1 ~0]^\top$}
        \label{fig:sim_backward_propagation_b}
    \end{subfigure}\hfill
    \begin{subfigure}[t]{0.24\linewidth}
        \centering
        \includegraphics[width=\linewidth]{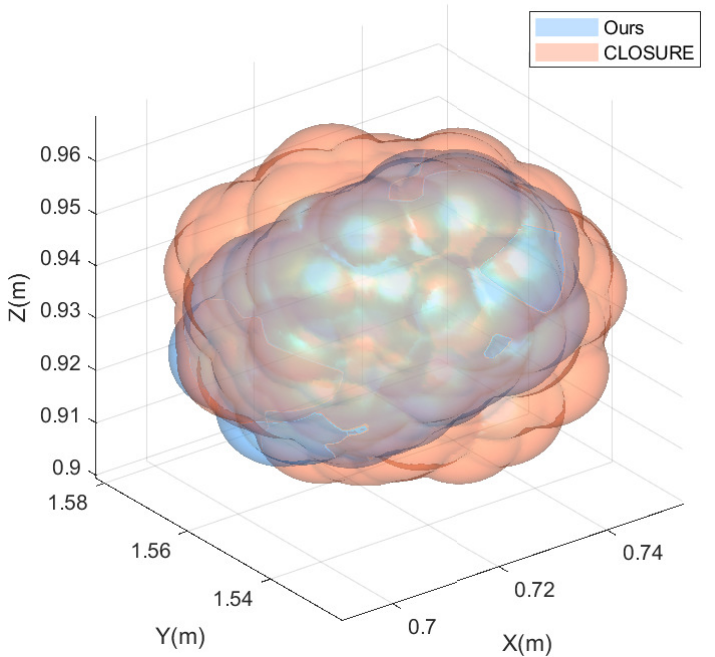}
        \caption{Ball center: $[0 ~0 ~1]^\top$}
        \label{fig:sim_backward_propagation_c}
    \end{subfigure}
     \begin{subfigure}[t]{0.24\linewidth}
        \centering
        \includegraphics[width=\linewidth]{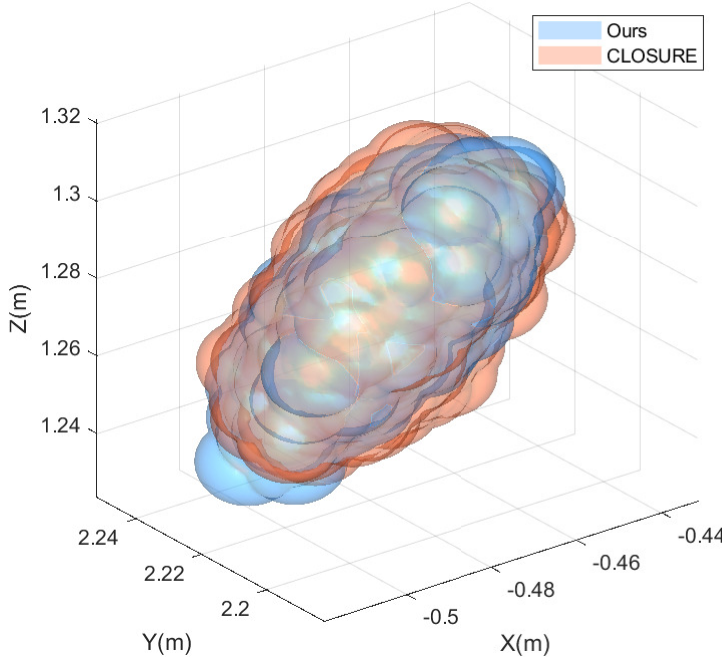}
        \caption{Ball center: $[1 ~1 ~1]^\top$}
        \label{fig:sim_backward_propagation_d}
    \end{subfigure}
     \begin{subfigure}[t]{0.24\linewidth}
        \centering
        \includegraphics[width=\linewidth]{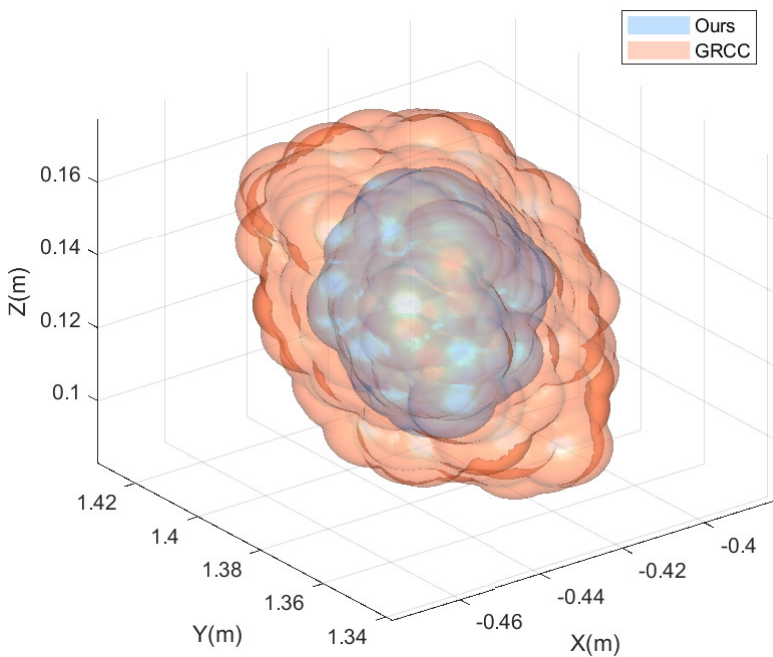}
        \caption{Ball center: $[1 ~0 ~0]^\top$}
        \label{fig:sim_backward_propagation_a2}
    \end{subfigure}\hfill
    \begin{subfigure}[t]{0.24\linewidth}
        \centering
        \includegraphics[width=\linewidth]{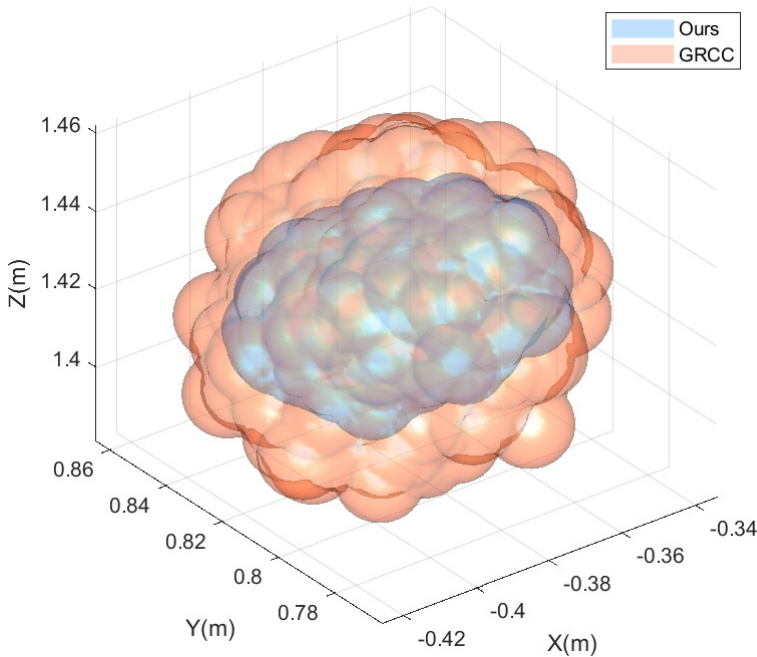}
        \caption{Ball center: $[0 ~1 ~0]^\top$}
        \label{fig:sim_backward_propagation_b2}
    \end{subfigure}\hfill
    \begin{subfigure}[t]{0.24\linewidth}
        \centering
        \includegraphics[width=\linewidth]{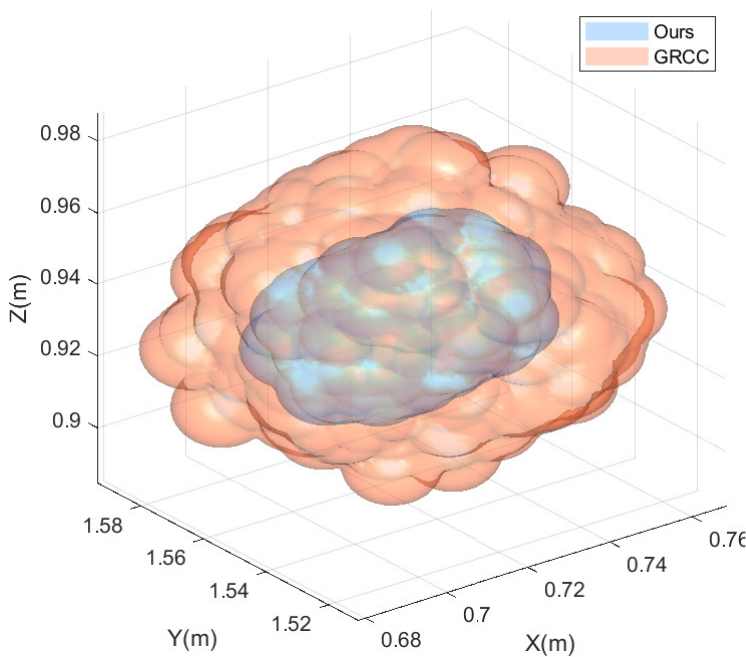}
        \caption{Ball center: $[0 ~0 ~1]^\top$}
        \label{fig:sim_backward_propagation_c2}
    \end{subfigure}
     \begin{subfigure}[t]{0.24\linewidth}
        \centering
        \includegraphics[width=\linewidth]{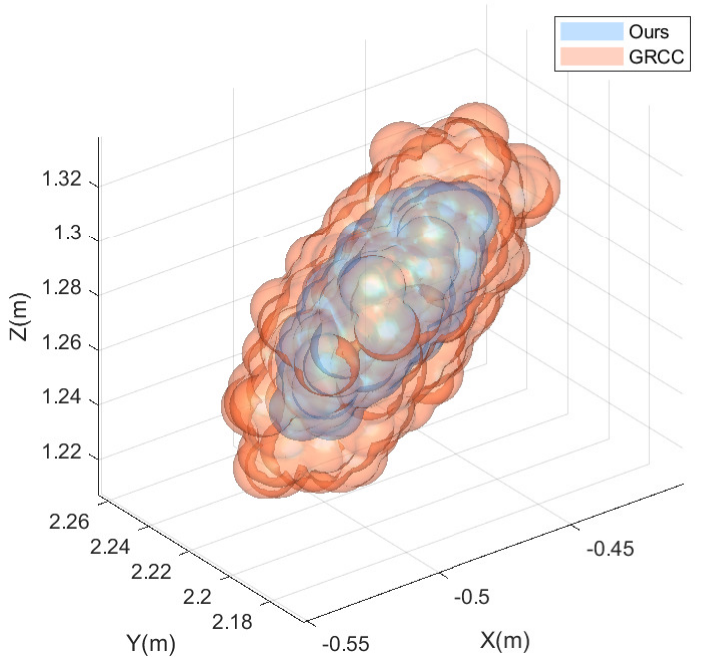}
        \caption{Ball center: $[1 ~1 ~1]^\top$}
        \label{fig:sim_backward_propagation_d2}
    \end{subfigure}

    \caption{Sampling-based pose uncertainty visualization in backward UQ. A ball with radius $0.01$m is transformed by $1000$ poses sampled from the estimated pose set, and the union of the transformed balls are plotted.}
    \label{fig:sim_backward_propagation}
\end{figure*}

\begin{figure*}[!htbp]
    \centering
    \begin{subfigure}[t]{0.24\linewidth}
        \centering
        \includegraphics[width=\linewidth]{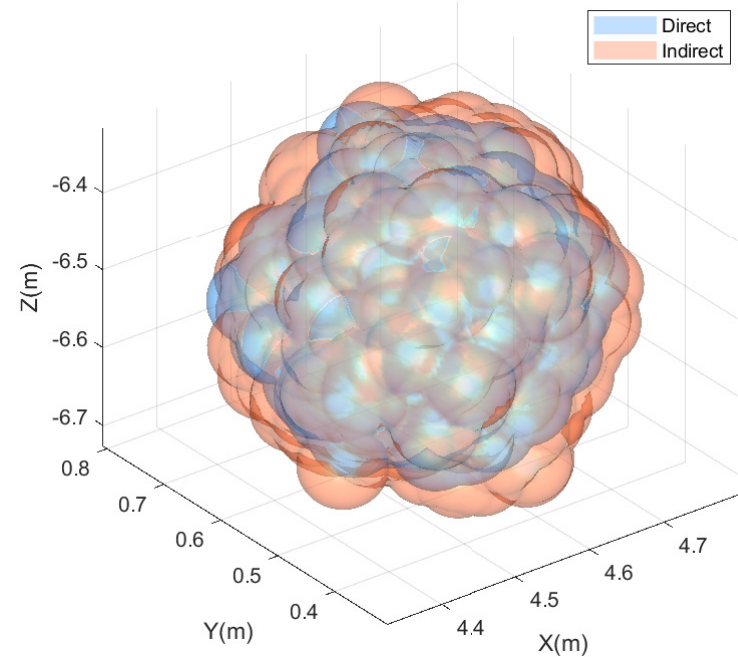}
        \caption{Ball center: $[10~0~0]^\top$}
        \label{fig:sim_pose_compound_a}
    \end{subfigure}
    \begin{subfigure}[t]{0.24\linewidth}
        \centering
        \includegraphics[width=\linewidth]{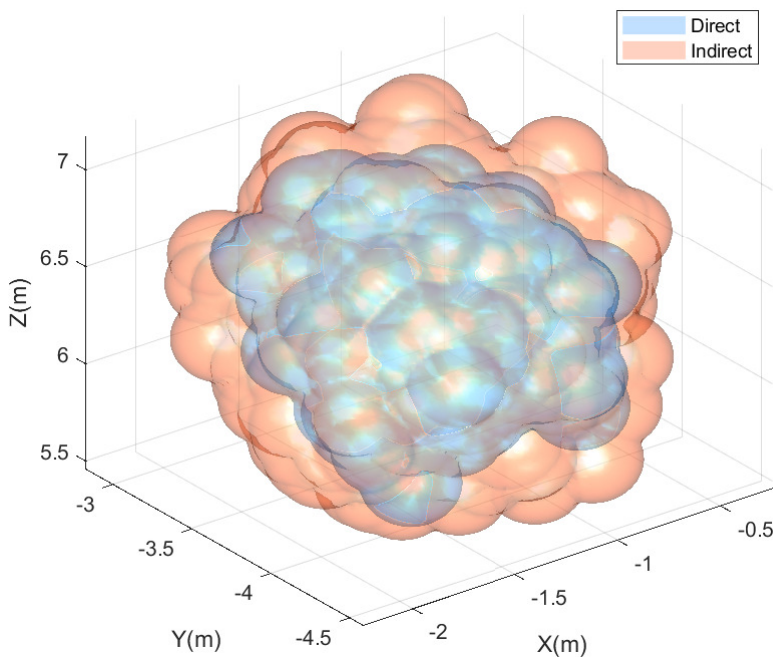}
        \caption{Ball center: $[0~10~0]^\top$}
        \label{fig:sim_pose_compound_b}
    \end{subfigure}
    \begin{subfigure}[t]{0.24\linewidth}
        \centering
        \includegraphics[width=\linewidth]{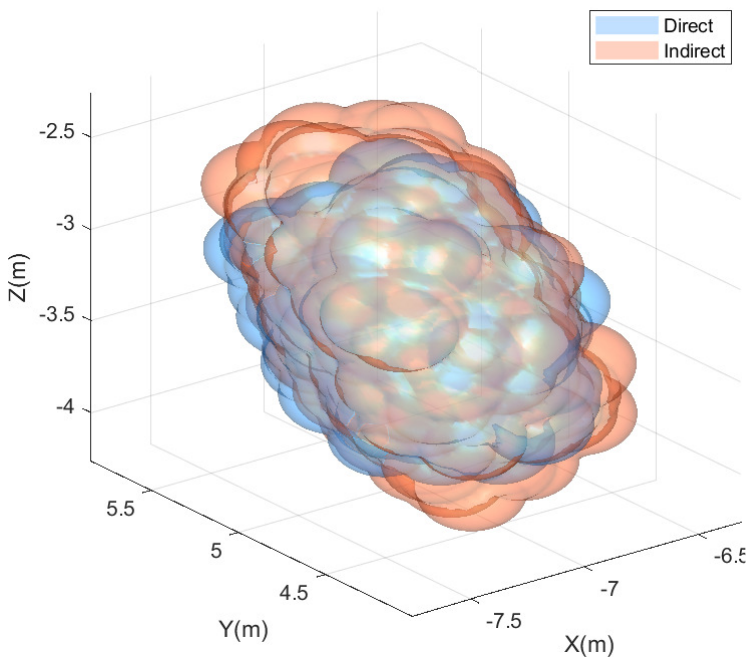}
        \caption{Ball center: $[0~0~10]^\top$}
        \label{fig:sim_pose_compound_c}
    \end{subfigure}
     \begin{subfigure}[t]{0.24\linewidth}
        \centering
        \includegraphics[width=\linewidth]{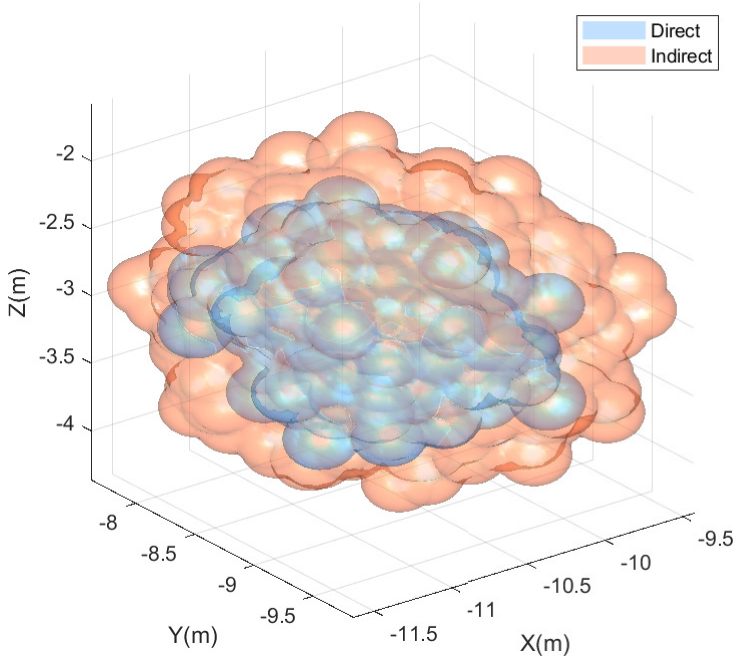}
        \caption{Ball center: $[10~10~10]^\top$}
        \label{fig:sim_pose_compound_d}
    \end{subfigure}

    \caption{Sampling-based pose uncertainty visualization in pose compound. A ball of radius $0.2$m is transformed by $1000$ poses sampled from the estimated set, and the union of the transformed balls are plotted. }
    \label{fig:sim_pose_compound}
\end{figure*}

\begin{table}[t]
\centering
\caption{Conservatism test for backward UQ: the percentage of 1000 samples from the original set $\mathcal T$ that lie within the estimated set.}
\label{tab:backward_propagation}
\begin{tabular}{lcccccc}
\hline
Trial & 1 & 2 & 3 & 4 & 5 & 6 \\
\hline
CLOSURE & 87.1\% & 59.6\% & 32.8\% & 71.3\% & 73.2\% & 10.4\%\\
GRCC & 100\% & 100\% & 100\% & 100\% & 100\% & 100\%\\
Ours & 100\% & 100\% & 100\% & 100\% & 100\% & 100\%\\
\hline
\end{tabular}
\end{table}

\begin{figure*}[!htbp]
    \centering
    \begin{subfigure}[t]{0.235\linewidth}
        \centering
        \includegraphics[width=\linewidth]{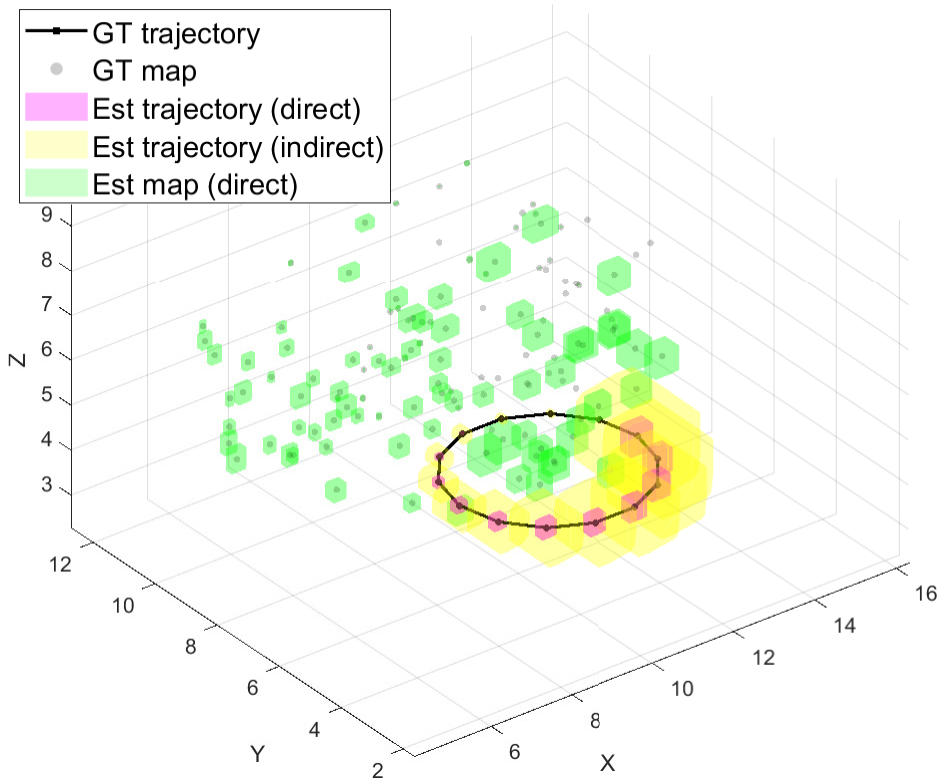}
        \caption{Relative framework}
        \label{fig:toy_example_a}
    \end{subfigure}
    \begin{subfigure}[t]{0.235\linewidth}
        \centering
        \includegraphics[width=\linewidth]{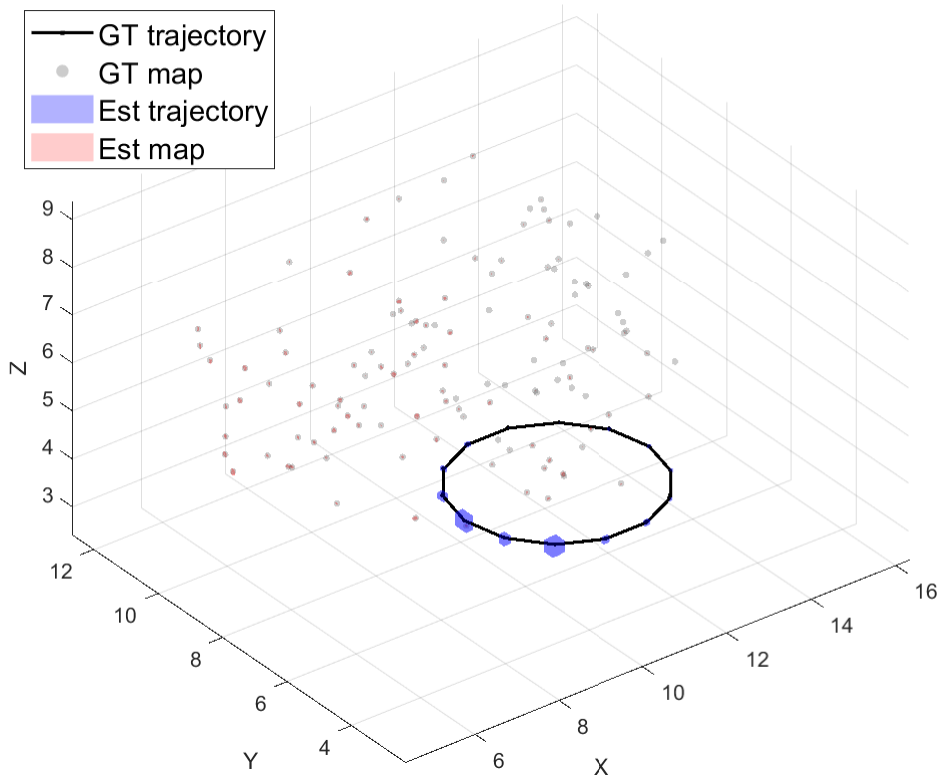}
        \caption{Global framework}
        \label{fig:toy_example_b}
    \end{subfigure}
        \begin{subfigure}[t]{0.235\linewidth}
        \centering
        \includegraphics[width=\linewidth]{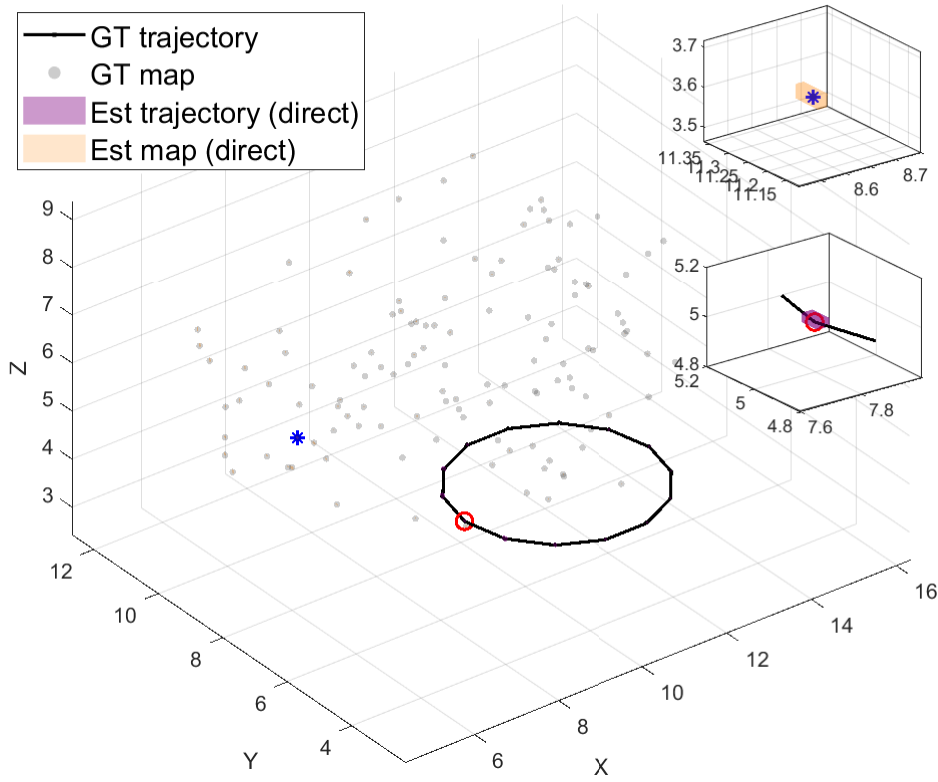}
        \caption{Loop closure}
        \label{fig:toy_example_d} 
    \end{subfigure}  \hspace{0.5mm} \rule{1pt}{0.22\textwidth} \hspace{0.5mm}
            \begin{subfigure}[t]{0.235\linewidth}
        \centering
        \includegraphics[width=\linewidth]{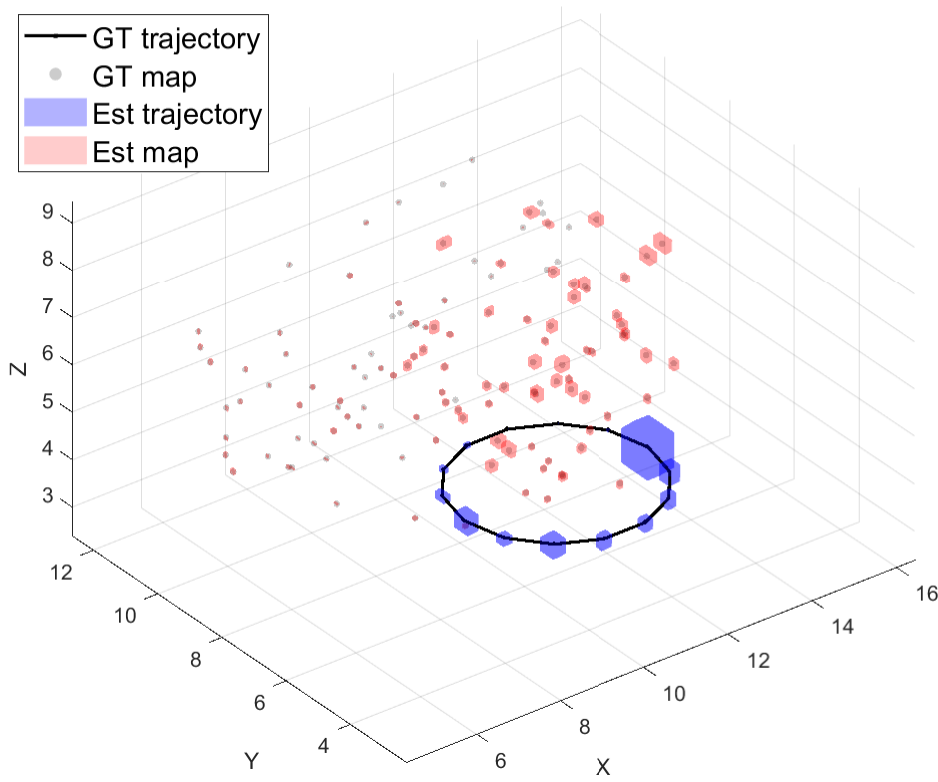}
        \caption{Global framework ablation}
        \label{fig:toy_example_c}
    \end{subfigure}
    \caption{Guaranteed SLAM simulation results. In the relative framework, ``direct'' and ``indirect'' denote the direct pose compound and the indirect pose compound results, respectively, and the map are constructed based on direct pose compound results. The loop closure results are obtained by feeding the ``direct'' results into the smoothing module. The global framework ablation uses the setting that each frame registers all visible (instead of newly observed) landmarks into the global map.}
    \label{fig:sim_toy_example}
\end{figure*}

\subsubsection{Pose compound} We perform six trials. In each trial, we randomly generate poses $T_1$ and $T_2$ and use the same uncertainty-generation procedure as in the forward UQ experiments. To assess conservatism, we draw $1000$ samples from the original set $\mathcal T_3$ and verify that all samples lie within the pose uncertainty sets provided by both the direct method and the indirect method in Sec.\ref{subsec:pose_compound}.
We plot the translation and rotation uncertainties in the first trial in Figure~\ref{fig:sim_pose_compound_pro}. The direct algorithm yields a more conservative translation uncertainty but a tighter rotation uncertainty. Part of the higher rotation conservatism of the indirect method is attributed to its geodesic-ball approximation. 
We then use a sampling-based approach for complete pose uncertainty visualization. We draw $1000$ poses from the estimated set and transform a ball of radius $0.2$m. The ball center is set as $[10~0~0]^\top, [0~10~0]^\top, [0~0~10]^\top, [10~10~10]^\top$, respectively, and the results in the first trial are shown in Figure~\ref{fig:sim_pose_compound}.  We see that the direct algorithm achieves a tighter pose uncertainty. 
% Our MATLAB implementation is evaluated on a machine with an Intel Core i7-10700 CPU
% for runtime comparison. The direct and indirect methods take \zgy{$14.17$s and $6.60$s} on average, respectively. 

\begin{figure}[!htbp]
    \centering
     \begin{subfigure}[t]{0.49\linewidth}
        \centering
        \includegraphics[width=\linewidth]{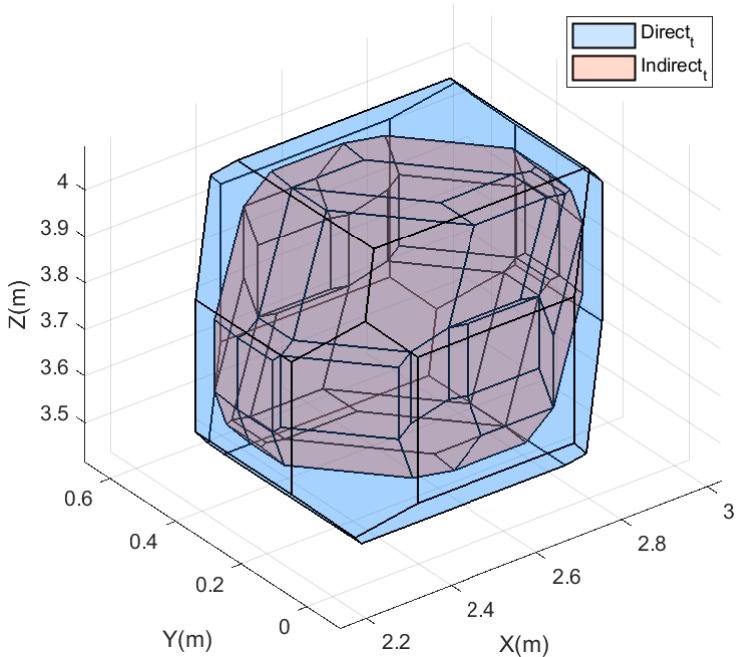}
        \label{fig:sim_pose_compound_t}
    \end{subfigure}
         \begin{subfigure}[t]{0.49\linewidth}
        \centering
        \includegraphics[width=\linewidth]{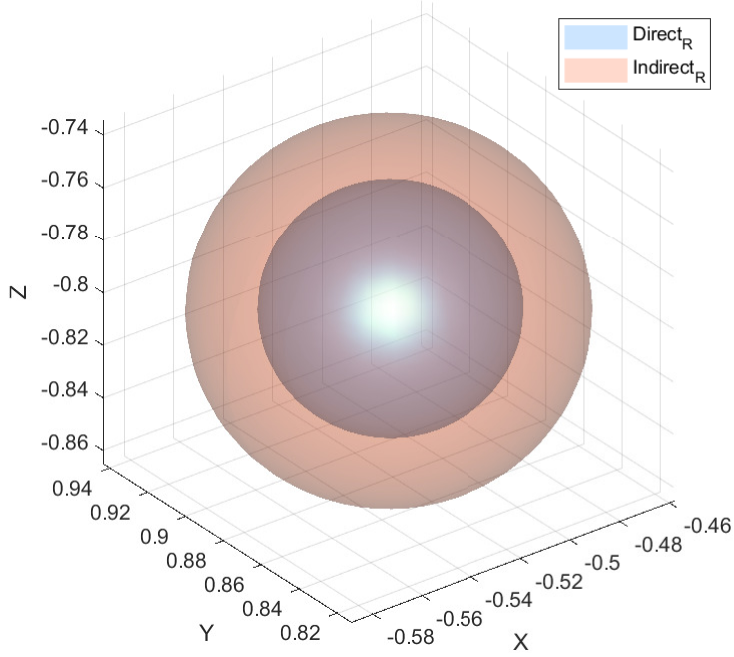}
        \label{fig:sim_pose_compound_R}
    \end{subfigure}
    \caption{Visualization of uncertainties for translation (left) and rotation (right) in pose compound.}
    \label{fig:sim_pose_compound_pro}
\end{figure}

\begin{figure*}[!htbp]
    \centering
     \begin{subfigure}[t]{0.5\linewidth}
        \centering
        \includegraphics[width=\linewidth]{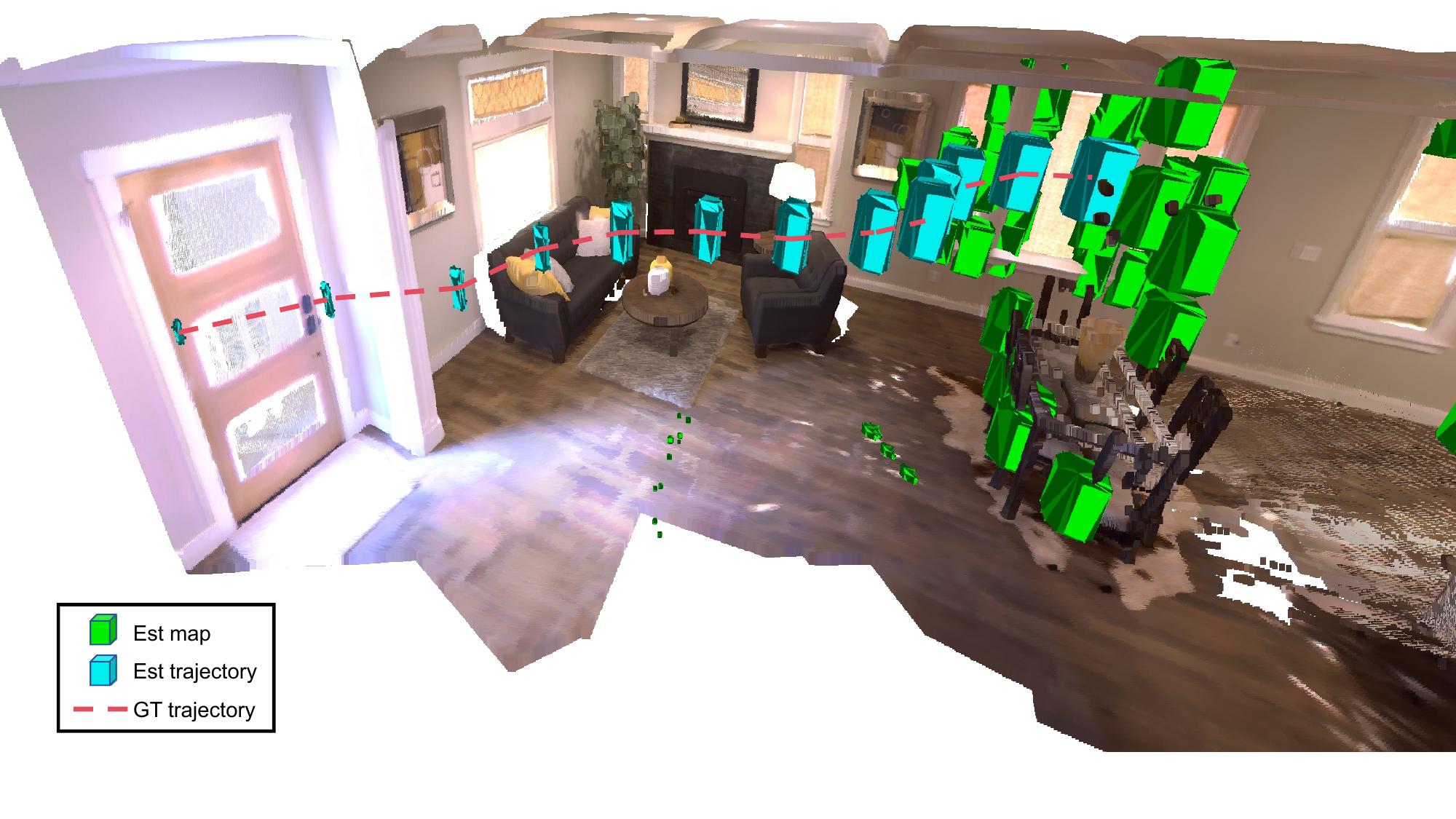}
        \caption{Relative framework (direct)}
        \label{fig:experiment_SLAM1}
    \end{subfigure}
    \begin{subfigure}[t]{0.43\linewidth}
        \centering
        \includegraphics[width=\linewidth]{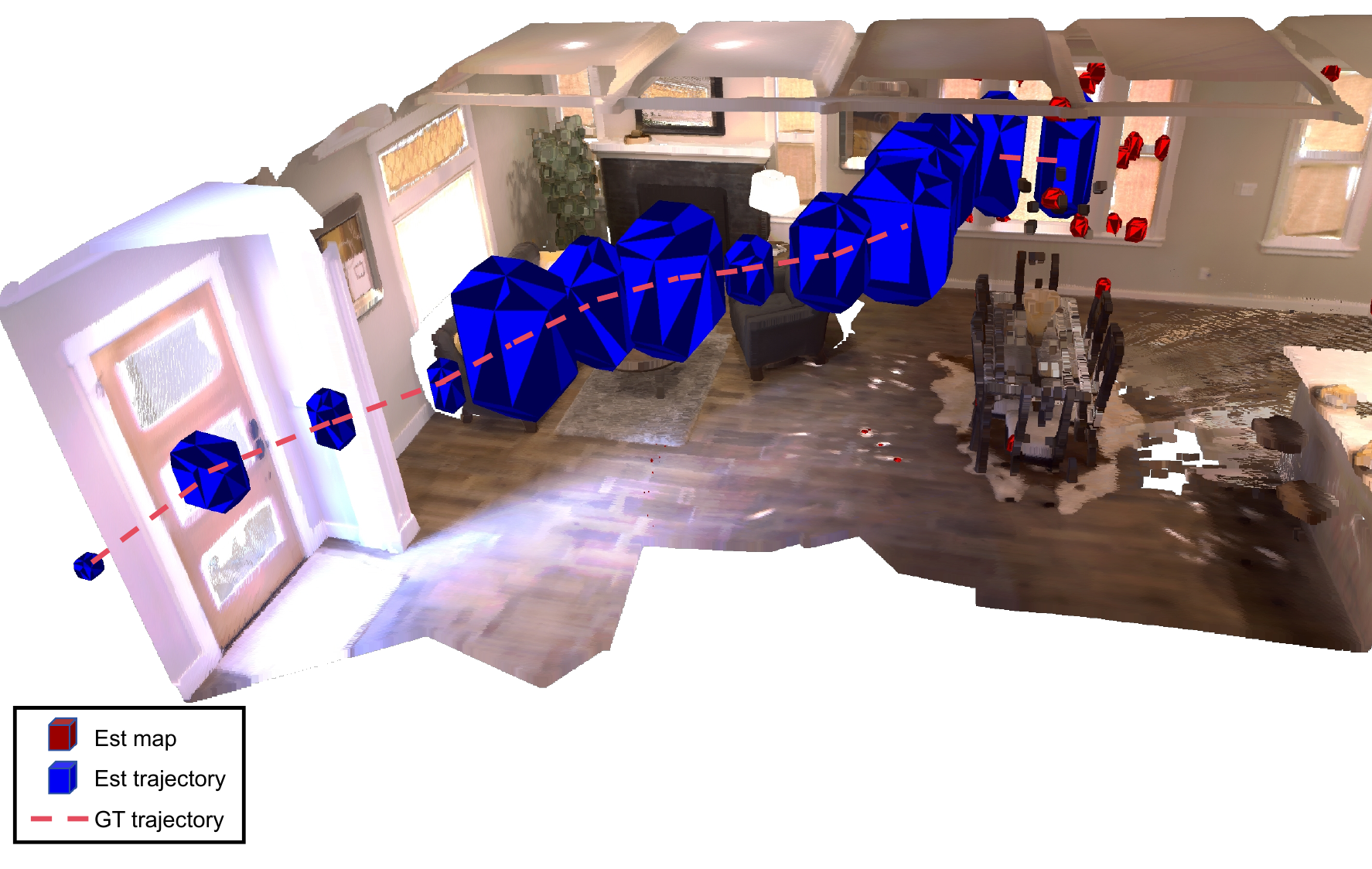}
        \caption{Global framework}
        \label{fig:experiment_SLAM2}
    \end{subfigure}
    \caption{Replica experiment results. The proposed relative framework using the direct pose compound method (left); the proposed global framework (right).}
    \label{fig:experiment_SLAM}
\end{figure*}

\subsection{Numerical SLAM Simulation} \label{subsec:SLAM_simulation}
We evaluate the proposed guaranteed SLAM UQ algorithms using synthetic data. We uniformly scatter 3D landmarks within a $50{\rm m}\times50{\rm m}\times50{\rm m}$ workspace and randomly generate a circular robot trajectory. The robot's field of view is $60^\circ$ horizontally, $60^\circ$ vertically, and $[0.5,5]{\rm m}$ in depth. For each visible landmark, we impose box uncertainties on the visual measurements, where the box size is proportional to the landmark's distance from the robot.
The results are shown in Figure~\ref{fig:sim_toy_example}.
In the relative framework, uncertainties for both robot poses and landmark positions accumulate over time. Notably, the pose uncertainty in the indirect compound algorithm grows more rapidly than in the direct one, aligning with the results from the previous subsection. 
We note that the map and trajectory uncertainties in the global framework are much smaller than those in the relative framework. 
This phenomenon occurs because, during motion, the robot may observe previously registered landmarks with low uncertainty; these observations mitigate current pose uncertainties. 
To substantiate this explanation, we conduct an ablation study in which each frame registers all visible landmarks $p_{k,i}$, rather than only newly observed landmarks $p_{k,i}^{\rm Ne}$, into the global map. Under this modification, previously registered landmarks are re-registered whenever they become visible in a new frame and thus may have a larger uncertainty set.
The ablation results show that the estimated uncertainty increases significantly, especially when the trajectory returns to the starting point, which verifies our claim.
Moreover, when a loop closure is detected and smoothing is performed, the uncertainties of both the trajectory and the map can be significantly reduced. Finally, we observe that all ground-truth poses and landmarks lie within their corresponding uncertainty sets. 

We also evaluate the computational complexity of our algorithm. When the average number of points per frame is 19, the average runtime of each module (MATLAB implementation on a computer with 32 GB RAM and an Intel Core Ultra 7 265 CPU) is as follows: forward UQ ($9.3$s), backward UQ ($0.0036$s), direct pose compound ($12$s), indirect pose compound ($0.5$s), and loop-closure smoothing ($29.55$s per frame, $591$s over $20$ frames). 
The primary computational bottleneck in our framework arises from the repeated SDP optimizations required for forward UQ and direct pose compound. While our current MATLAB implementation is a research prototype not yet intended for real-time use, it remains usable and effective for offline applications like uncertainty-aware mapping.

For comparison with probabilistic counterparts, we implement a probabilistic baseline that computes the exact posterior in a linear-Gaussian setting without ${\rm SE}(3)$ constraints. 
% We conduct two open-loop tests. 
% Unlike our algorithm, this baseline provides no formal containment guarantee.
% We represent its uncertainty sets by covariance-induced ellipsoids at the $99$\% confidence level. With a moderate number of visible points per frame ($\sim 15$), by frame $20$, the baseline fails to cover the ground truth in 
% $9/20$ pose estimates and $26/120$ map points, whereas our method achieves full containment. 
% The probabilistic baseline generally exhibits lower uncertainty than our method. However, since the baseline ignores ${\rm SE}(3)$ constraints, its uncertainty grows faster when only a few points are visible. 
% For $8$ points per frame, by frame $20$, our polytope volume is $10.4$, whereas the translation ellipsoid volume is 
% $18.2$, showing about $1.75\times$ faster growth.
We conduct open-loop comparisons under both global and relative frameworks. 
Unlike our algorithm, this baseline provides no formal containment guarantee.
We represent its uncertainty sets by covariance-induced ellipsoids at the $99$\% confidence level. In the global framework, with a moderate number of points per frame ($\sim 15$), by frame $20$, the baseline fails to cover the ground truth in 
$9/20$ pose estimates and $26/120$ map points, whereas our method achieves full containment. 
The probabilistic baseline generally exhibits lower uncertainty than our method. However, since it ignores ${\rm SE}(3)$ constraints, in the relative framework, its uncertainty grows faster when only a few points are visible. 
For $\sim 5$ points per frame, by frame $30$, the translation ellipsoid volume is 
$3.12$, whereas our polytope volume is $1.2$, showing about $3\times$ slower growth.

% \begin{figure*}[!htbp]
%     \centering
%      \begin{subfigure}[t]{0.49\linewidth}
%         \centering
%         \includegraphics[width=\linewidth]{figures/experiment_SLAM.pdf}
%         \label{fig:experiment_SLAM1}
%     \end{subfigure}
%          \begin{subfigure}[t]{0.49\linewidth}
%         \centering
%         \includegraphics[width=\linewidth]{figures/experiment_SLAM.pdf}
%         \label{fig:experiment_SLAM2}
%     \end{subfigure}
%     \caption{Replica experiment result obtained with the proposed relative framework using the direct pose compounding method.}
%     \label{fig:experiment_SLAM}
% \end{figure*}

% \begin{figure}[!htbp]
%     \centering
%     \includegraphics[width=0.92\linewidth]{figures/experiment_SLAM.pdf}
%     \caption{Replica experiment result obtained with the proposed relative framework using the direct pose compounding method.}
%     \label{fig:experiment_SLAM}
% \end{figure}

\begin{figure*}[!htbp]
    \centering
    \begin{subfigure}[t]{0.195\linewidth}
        \centering
        \includegraphics[width=\linewidth]{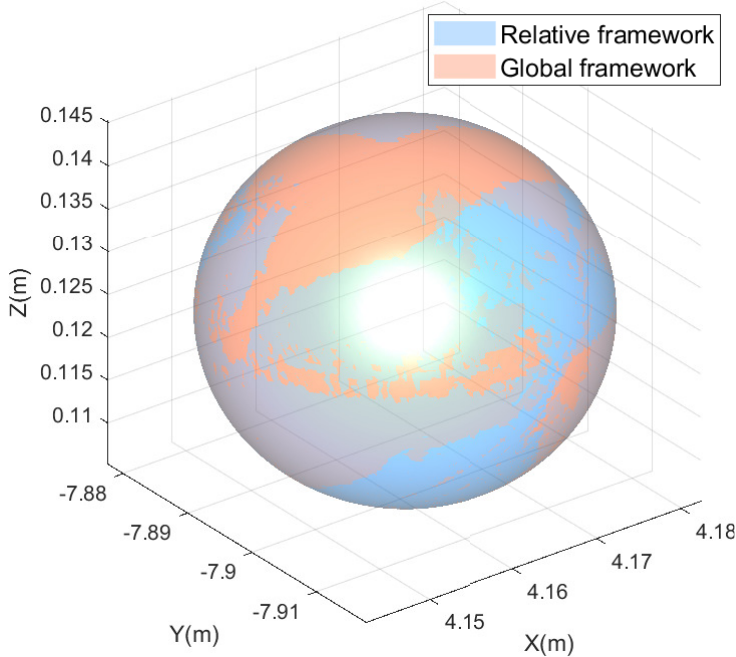}
        \caption{Pose $1$}
        \label{fig:experiment_suppl1}
    \end{subfigure}
        \begin{subfigure}[t]{0.195\linewidth}
        \centering
        \includegraphics[width=\linewidth]{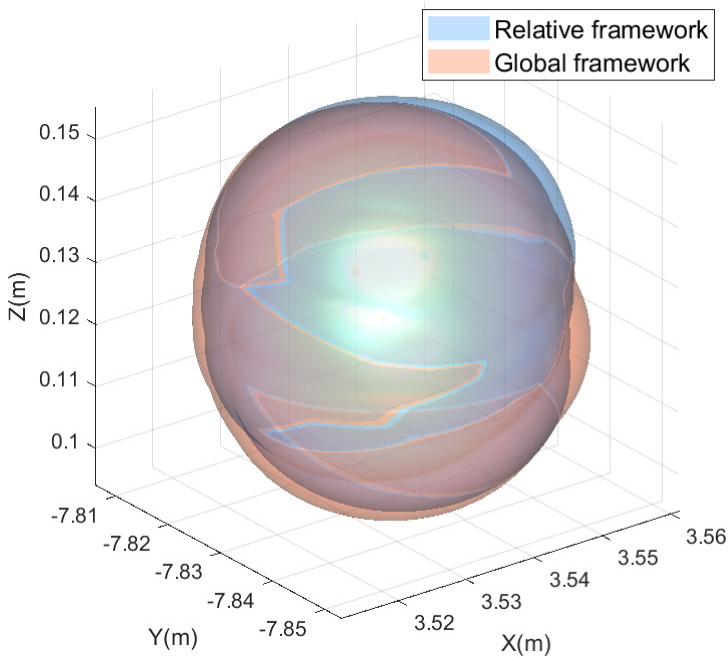}
        \caption{Pose $2$}
        \label{fig:experiment_suppl2}
    \end{subfigure}
       \begin{subfigure}[t]{0.195\linewidth}
        \centering
        \includegraphics[width=\linewidth]{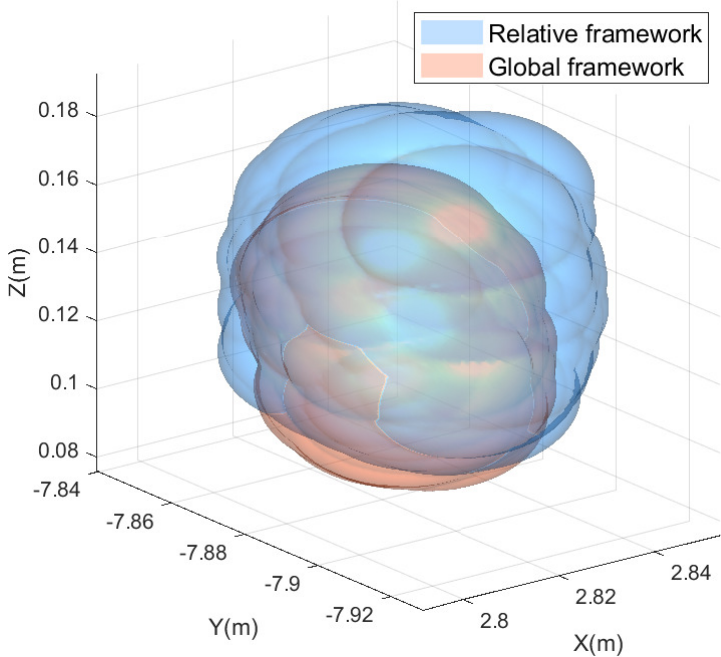}
        \caption{Pose $3$}
        \label{fig:experiment_suppl3}
    \end{subfigure}
        \begin{subfigure}[t]{0.195\linewidth}
        \centering
        \includegraphics[width=\linewidth]{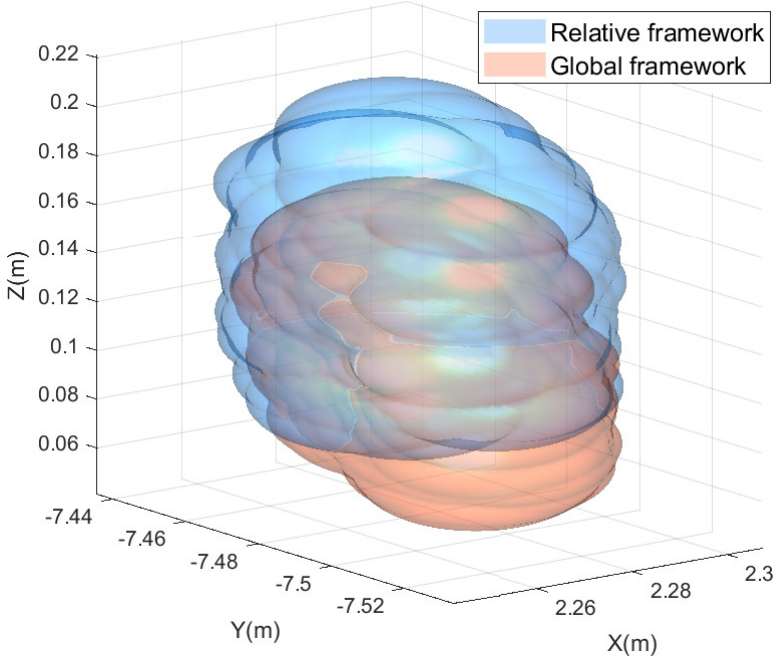}
        \caption{Pose $4$}
        \label{fig:experiment_suppl4}
    \end{subfigure}
        \begin{subfigure}[t]{0.195\linewidth}
        \centering
        \includegraphics[width=\linewidth]{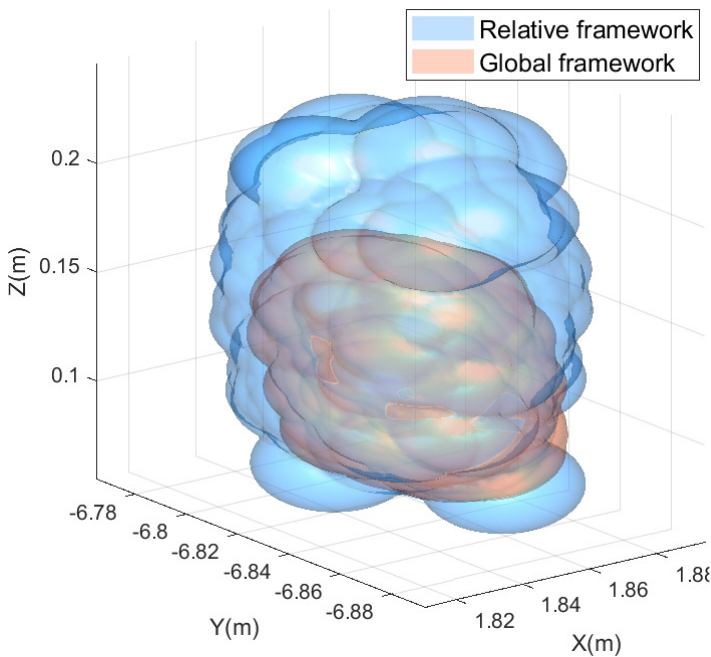}
        \caption{Pose $5$}
        \label{fig:experiment_suppl5}
    \end{subfigure}
        \begin{subfigure}[t]{0.195\linewidth}
        \centering
        \includegraphics[width=\linewidth]{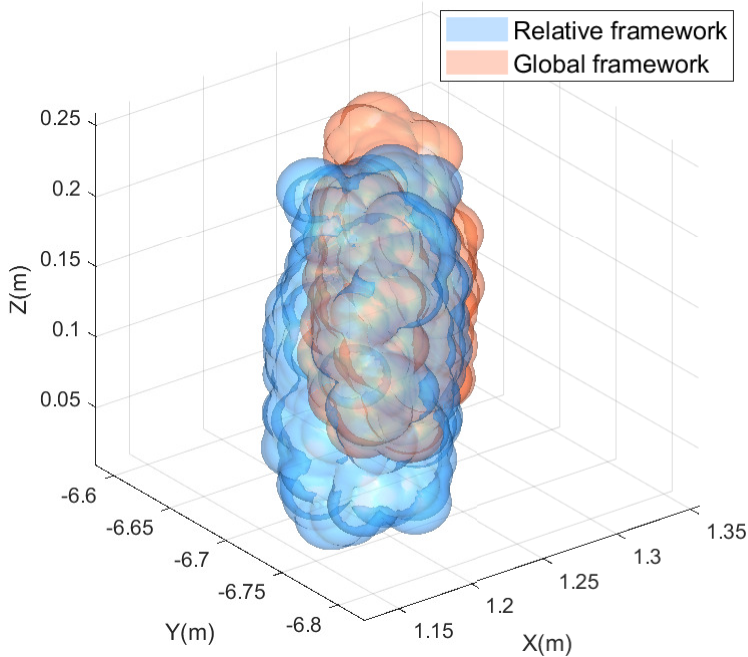}
        \caption{Pose $6$}
        \label{fig:experiment_suppl6}
    \end{subfigure}
       \begin{subfigure}[t]{0.195\linewidth}
        \centering
        \includegraphics[width=\linewidth]{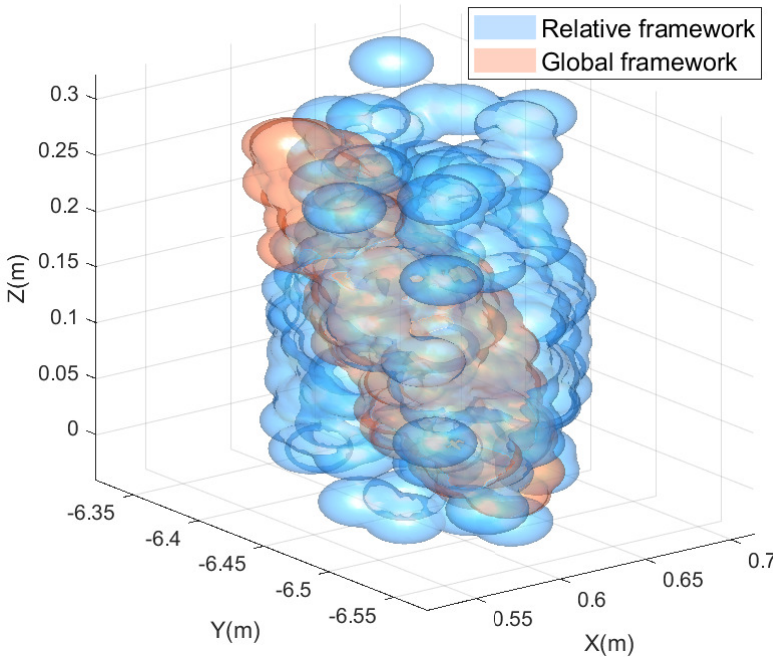}
        \caption{Pose $7$}
        \label{fig:experiment_suppl7}
    \end{subfigure}
        \begin{subfigure}[t]{0.195\linewidth}
        \centering
        \includegraphics[width=\linewidth]{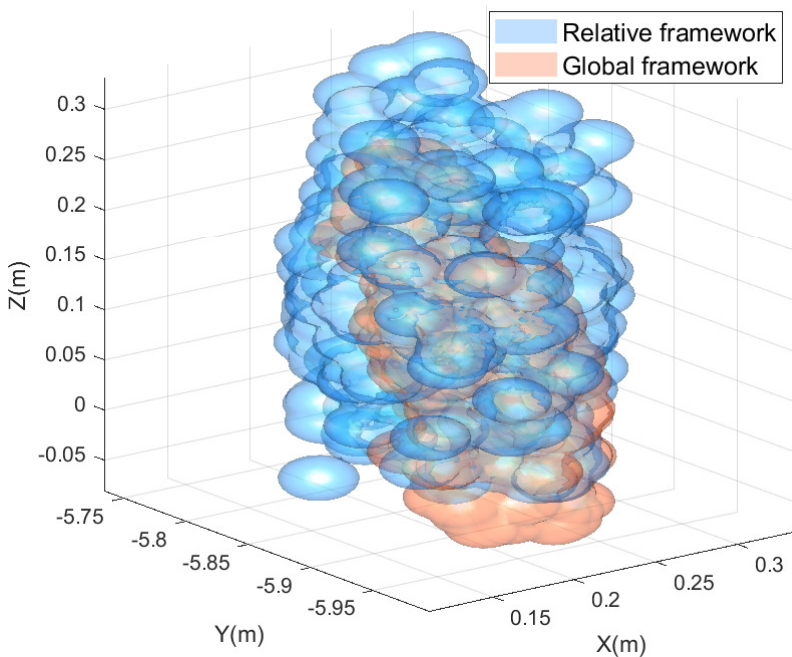}
        \caption{Pose $8$}
        \label{fig:experiment_suppl8}
    \end{subfigure}
      \begin{subfigure}[t]{0.195\linewidth}
        \centering
        \includegraphics[width=\linewidth]{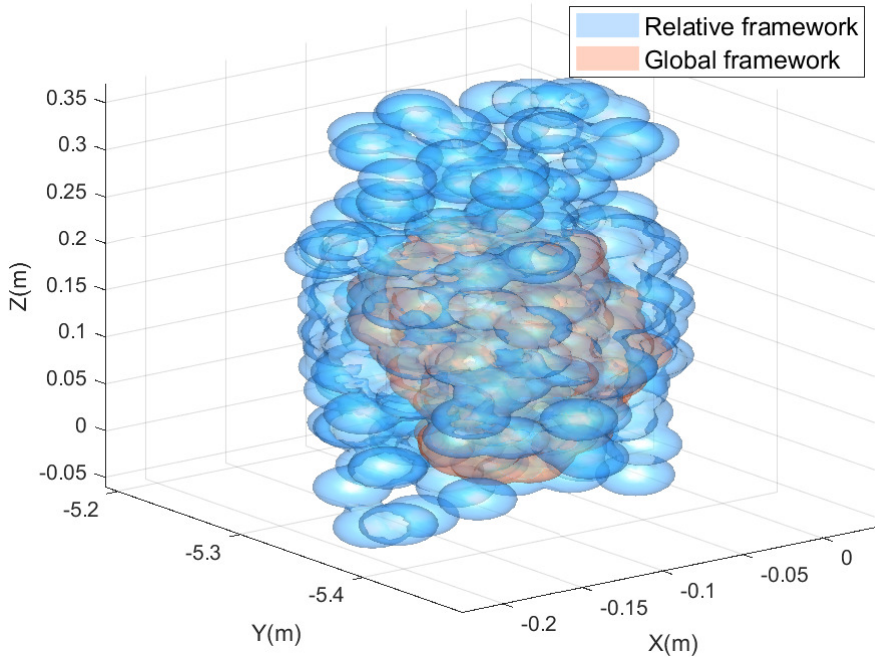}
        \caption{Pose $9$}
        \label{fig:experiment_suppl9}
    \end{subfigure}
       \begin{subfigure}[t]{0.195\linewidth}
        \centering
        \includegraphics[width=\linewidth]{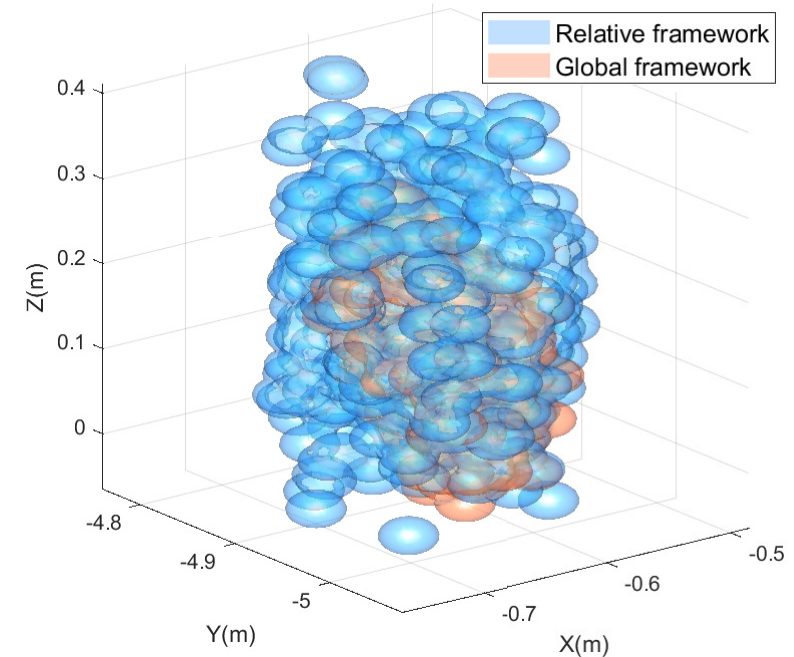}
        \caption{Pose $10$}
        \label{fig:experiment_suppl10}
    \end{subfigure}
        \begin{subfigure}[t]{0.195\linewidth}
        \centering
        \includegraphics[width=\linewidth]{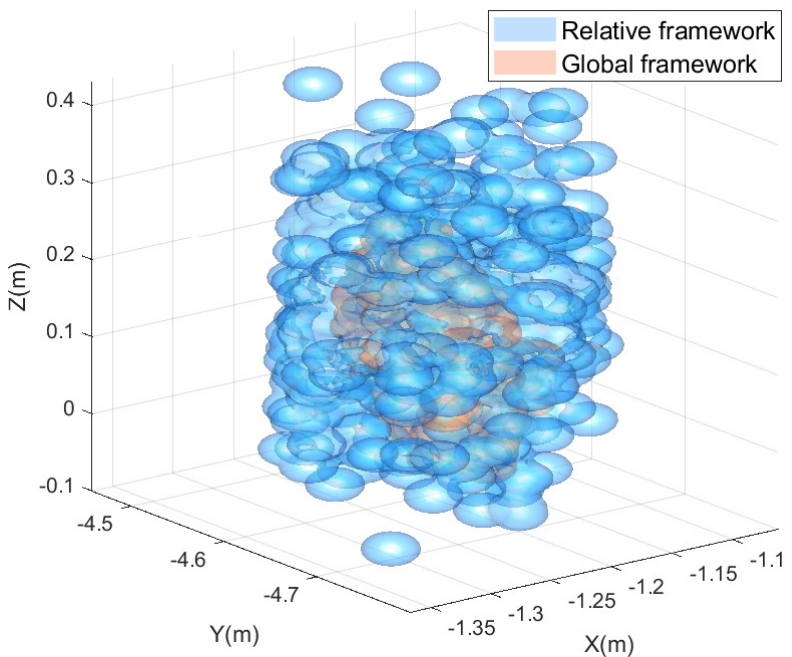}
        \caption{Pose $11$}
        \label{fig:experiment_suppl11}
    \end{subfigure}
        \begin{subfigure}[t]{0.195\linewidth}
        \centering
        \includegraphics[width=\linewidth]{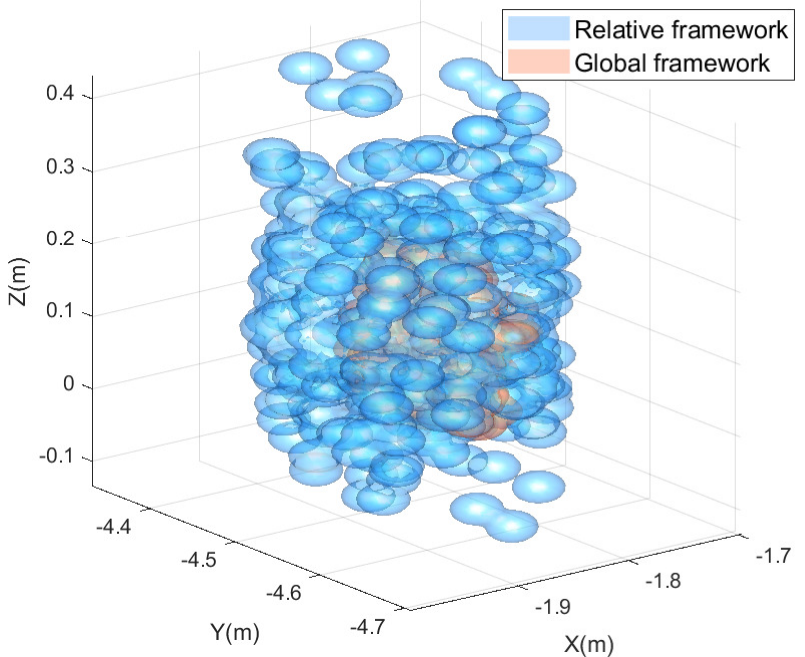}
        \caption{Pose $12$}
        \label{fig:experiment_suppl12}
    \end{subfigure}
        \begin{subfigure}[t]{0.195\linewidth}
        \centering
        \includegraphics[width=\linewidth]{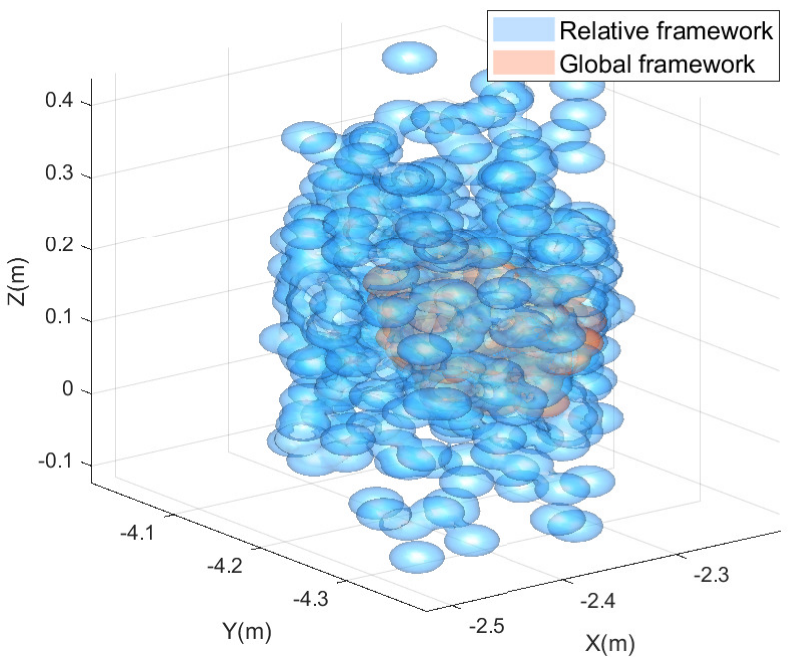}
        \caption{Pose $13$}
        \label{fig:experiment_suppl13}
    \end{subfigure}
       \begin{subfigure}[t]{0.195\linewidth}
        \centering
        \includegraphics[width=\linewidth]{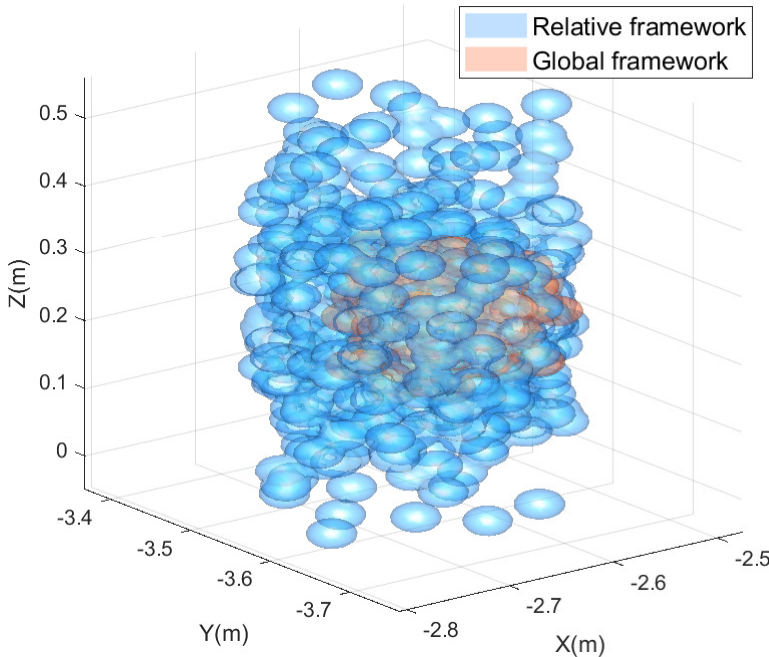}
        \caption{Pose $14$}
        \label{fig:experiment_suppl14}
    \end{subfigure}
        \begin{subfigure}[t]{0.195\linewidth}
        \centering
        \includegraphics[width=\linewidth]{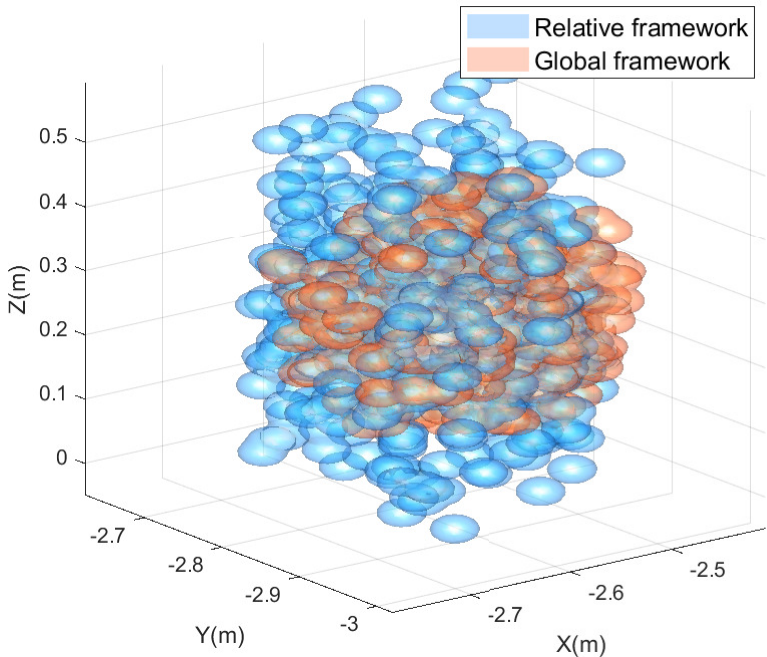}
        \caption{Pose $15$}
        \label{fig:experiment_suppl15}
    \end{subfigure}

    \caption{Sampling-based pose uncertainty visualization in SLAM experiment. A ball centered at $[0~0~1]^\top$ of radius $0.02$m is transformed by $1000$ poses sampled from the estimated set, and the union of the transformed balls are plotted. In the relative framework, direct pose compound is adopted.}
    \label{fig:experiment_suppl}
\end{figure*}

\subsection{SLAM Experiment} \label{subsec:SLAM_experiment}
We evaluate our method on the Replica dataset~\cite{straub2019replica}, which provides photo-realistic indoor environments. For any camera pose, it renders an RGB image and a dense depth map, which we back-project to generate a point cloud.
To calibrate CP, we randomly generate three trajectories, each containing hundreds of frames. When a new frame arrives, we extract ORB features~\cite{campos2021orb} from the current and previous frames and perform feature matching. To ensure tighter uncertainty bounds, we apply RANSAC-based outlier rejection. We reject outliers using the feature depths and the ground-truth relative pose provided by the dataset.
For each 3D correspondence $(p_l,q_l)$, we transform $p_l$ using the relative pose $(R,t)$ and define the nonconformity score as $s_l=\|q_l-Rp_l-t\|$. We set the target miscoverage level in CP to $\delta=0.01$ to calibrate local point uncertainties.
During testing, we use the same feature extraction and matching pipeline. Loop closure detection is not involved. To balance runtime and accuracy, we randomly retain $15$ correspondences for each pair of consecutive frames.
In our feature-matching setting, the global framework works in the same way as in the ablation study in Figure~\ref{fig:sim_toy_example}. As shown in Figure~\ref{fig:experiment_SLAM}, the estimated trajectories and maps enclose their ground truths well. An interesting phenomenon is that the global framework achieves tighter map uncertainty but more conservative trajectory (translation) uncertainty than the relative framework. 
The reason is that although the marginally projected translation uncertainty of the global framework is larger, its joint pose uncertainty in the form $\{T \in {\rm SE}(3) \mid H x(T) \leq d\}$ is smaller due to the coupling between translation and rotation. To verify this, we conduct a sampling-based pose uncertainty test. As shown in Figure~\ref{fig:experiment_suppl}, the poses sampled from the relative framework-generated uncertainty sets span a ball to a larger volume, indicating its larger pose uncertainty.
Overall, the uncertainty sets of both frameworks expand relatively rapidly in the absence of loop-closure smoothing.

\section{Conclusion and future work}
In this paper, we proposed SME-based UQ for 3D-3D SLAM pipelines. We identified three UQ primitives---forward UQ, backward UQ, and pose compound---and derived uncertainty sets with guarantees for each. By integrating them with CP, we obtain a complete SLAM UQ algorithm. Representing uncertainties as polytopes further enables tractable computation and a unified treatment of pose uncertainty.

%Future work focuses on three directions: (i) Tightness: mitigating rapid uncertainty growth by incorporating sliding-window bundle adjustment and tighter SDP formulations; (ii) Performance: accelerating execution, potentially via GPU-based parallelization; and (iii) Usability: extending the framework to additional modalities (e.g., IMUs, range sensors) and validating it in complex realistic environments.

Future work will focus on three directions. (i) Tightness: In our tests, the uncertainty sets expand relatively quickly. We plan to incorporate sliding-window bundle adjustment and develop tighter SDP formulations to slow this growth. We are also interested in providing a rigorous theoretical analysis of the uncertainty set growth rate. (ii) Real-time performance: While tractable, the current algorithms are still in an early stage; future iterations will target real-time performance by transitioning to C++ and incorporating parallel SDP computation, keyframe mechanisms, and customized solver designs. (iii) Practical utility: Our current implementation is limited to 3D-3D landmark-based SLAM. Extending the approach to more sensor modalities (e.g., IMUs and range sensors) and validating it in complex, large-scale (loop-closure) scenarios remain promising directions.  

\section*{Acknowledgement}
This work was supported in part by NSFC under Grants 62273288 and 62336005, in part by the Shenzhen Science and Technology Program under Grants JCYJ20241202124010014 and JCYJ20240813113609013, and in part by Australian Research Council (ARC) under Grants DP190103615, LP210200473, and DP230101014.

\bibliographystyle{unsrtnat}
\bibliography{sj_reference}

@inproceedings{gao2024closure,
  title={Closure: Fast quantification of pose uncertainty sets},
  author={Gao, Yihuai and Tang, Yukai and Qi, Han and Yang, Heng},
  booktitle={Proceedings of Robotics: Science and Systems},
pages={1--13},
  year={2024}
}

@book{szeliski2022computer,
  title={Computer {V}ision: {A}lgorithms and {A}pplications},
  author={Szeliski, Richard},
  year={2022},
  publisher={Springer Nature}
}

@article{lindemann2024formal,
  title={Formal verification and control with conformal prediction: Practical safety guarantees for autonomous systems},
  author={Lindemann, Lars and Zhao, Yiqi and Yu, Xinyi and Pappas, George J and Deshmukh, Jyotirmoy V},
  journal={IEEE Control Systems},
  volume={45},
  number={6},
  pages={72--122},
  year={2025},
  publisher={IEEE}
}

@book{lee2018introduction,
  title={Introduction to {R}iemannian {M}anifolds},
  author={Lee, John M},
  volume={2},
  year={2018},
  publisher={Springer}
}

@article{shafer2008tutorial,
  title={A tutorial on conformal prediction.},
  author={Shafer, Glenn and Vovk, Vladimir},
  journal={Journal of Machine Learning Research},
  volume={9},
  number={3},
pages={371--421},
  year={2008}
}

@article{angelopoulos2023conformal,
  title={Conformal prediction: A gentle introduction},
  author={Angelopoulos, Anastasios N and Bates, Stephen and others},
  journal={Foundations and Trends{\textregistered} in Machine Learning},
  volume={16},
  number={4},
  pages={494--591},
  year={2023},
  publisher={Now Publishers, Inc.}
}

@article{fontana2023conformal,
  title={Conformal prediction: a unified review of theory and new challenges},
  author={Fontana, Matteo and Zeni, Gianluca and Vantini, Simone},
  journal={Bernoulli},
  volume={29},
  number={1},
  pages={1--23},
  year={2023},
  publisher={Bernoulli Society for Mathematical Statistics and Probability}
}

@book{boyd2004convex,
  title={Convex {O}ptimization},
  author={Boyd, Stephen and Vandenberghe, Lieven},
  year={2004},
  publisher={Cambridge University Press}
}

@article{zeng2025bias,
  title={Bias-Eliminated {PnP} for Stereo Visual Odometry: Provably Consistent and Large-Scale Localization},
  author={Zeng, Guangyang and Shen, Yuan and Hong, Ziyang and Hong, Yuze and Ila, Viorela and Shi, Guodong and Wu, Junfeng},
 volume={10},
  number={11},
  pages={11840--11847},
  journal={IEEE Robotics and Automation Letters},
  year={2025}
}

@inproceedings{yang2023object,
  title={Object pose estimation with statistical guarantees: Conformal keypoint detection and geometric uncertainty propagation},
  author={Yang, Heng and Pavone, Marco},
  booktitle={Proceedings of IEEE/CVF Conference on Computer Vision and Pattern Recognition},
  pages={8947--8958},
  year={2023}
}

@inproceedings{tang2024uncertainty,
  title={Uncertainty quantification of set-membership estimation in control and perception: Revisiting the minimum enclosing ellipsoid},
  author={Tang, Yukai and Lasserre, Jean-Bernard and Yang, Heng},
  booktitle={Proceedings of Annual Learning for Dynamics \& Control Conference},
  pages={286--298},
  year={2024}
}

@article{campos2021orb,
  title={Orb-slam3: An accurate open-source library for visual, visual--inertial, and multimap {SLAM}}}

@article{zhen2018adjustable,
  title={Adjustable robust optimization via {Fourier--Motzkin} elimination},
  author={Zhen, Jianzhe and Den Hertog, Dick and Sim, Melvyn},
  journal={Operations Research},
  volume={66},
  number={4},
  pages={1086--1100},
  year={2018},
  publisher={INFORMS}
}

@incollection{avis2002canonical,
  title={On canonical representations of convex polyhedra},
  author={Avis, David and Fukuda, Komei and Picozzi, Stefano},
  booktitle={Mathematical Software},
  pages={350--360},
  year={2002},
  publisher={World Scientific}
}

@book{nesterov1994interior,
  title={Interior-point polynomial algorithms in convex programming},
  author={Nesterov, Yurii and Nemirovskii, Arkadii},
  year={1994},
  publisher={SIAM}
}

@inproceedings{zhang2014loam,
  title={{LOAM}: Lidar odometry and mapping in real-time.},
  author={Zhang, Ji and Singh, Sanjiv and others},
  booktitle={Proceedings of Robotics: Science and Systems},
  pages={1--9},
  year={2014}
}

@article{dumbgen2024toward,
  title={Toward globally optimal state estimation using automatically tightened semidefinite relaxations},
  author={D{\"u}mbgen, Frederike and Holmes, Connor and Agro, Ben and Barfoot, Timothy},
  journal={IEEE Transactions on Robotics},
  volume={40},
  pages={4338--4358},
  year={2024},
  publisher={IEEE}
}

@article{cvivsic2022soft2,
  title={Soft2: Stereo visual odometry for road vehicles based on a point-to-epipolar-line metric},
  author={Cvi{\v{s}}i{\'c}, Igor and Markovi{\'c}, Ivan and Petrovi{\'c}, Ivan},
  journal={IEEE Transactions on Robotics},
  volume={39},
  number={1},
  pages={273--288},
  year={2022},
  publisher={IEEE}
}

@article{yang2020teaser,
  title={Teaser: Fast and certifiable point cloud registration},
  author={Yang, Heng and Shi, Jingnan and Carlone, Luca},
  journal={IEEE Transactions on Robotics},
  volume={37},
  number={2},
  pages={314--333},
  year={2020},
  publisher={IEEE}
}

@inproceedings{han2025building,
  title={Building rome with convex optimization},
  author={Han, Haoyu and Yang, Heng},
  booktitle={Proceedings of Robotics: Science and Systems},
  pages={1--19},
  year={2025}
}

@inproceedings{engel2014lsd,
  title={{LSD-SLAM}: Large-scale direct monocular {SLAM}},
  author={Engel, Jakob and Sch{\"o}ps, Thomas and Cremers, Daniel},
  booktitle={Proceedings of European Conference on Computer Vision},
  pages={834--849},
  year={2014}
}

@article{rosen2019se,
  title={{SE-Sync}: A certifiably correct algorithm for synchronization over the special Euclidean group},
  author={Rosen, David M and Carlone, Luca and Bandeira, Afonso S and Leonard, John J},
  journal={The International Journal of Robotics Research},
  volume={38},
  number={2-3},
  pages={95--125},
  year={2019},
  publisher={Sage Publications Sage UK: London, England}
}

@article{mangelson2020characterizing,
  title={Characterizing the uncertainty of jointly distributed poses in the {L}ie algebra},
  author={Mangelson, Joshua G and Ghaffari, Maani and Vasudevan, Ram and Eustice, Ryan M},
  journal={IEEE Transactions on Robotics},
  volume={36},
  number={5},
  pages={1371--1388},
  year={2020},
  publisher={IEEE}
}

@article{agrawal2024gatekeeper,
  title={gatekeeper: Online safety verification and control for nonlinear systems in dynamic environments},
  author={Agrawal, Devansh Ramgopal and Chen, Ruichang and Panagou, Dimitra},
  journal={IEEE Transactions on Robotics},
volume={40},
  pages={4358--4375},
  year={2024},
  publisher={IEEE}
}

@article{de2024set,
  title={Set-based state estimation for discrete-time constrained nonlinear systems: An approach based on constrained zonotopes and {DC} programming},
  author={de Paula, Alesi A and Raimondo, Davide M and Raffo, Guilherme V and Teixeira, Bruno OS},
  journal={Automatica},
  volume={159},
  pages={111401},
  year={2024},
  publisher={Elsevier}
}

@article{wang2018zonotopic,
  title={Zonotopic set-membership state estimation for discrete-time descriptor {LPV} systems},
  author={Wang, Ye and Wang, Zhenhua and Puig, Vicen{\c{c}} and Cembrano, Gabriela},
  journal={IEEE Transactions on Automatic Control},
  volume={64},
  number={5},
  pages={2092--2099},
  year={2018},
  publisher={IEEE}
}

@article{mustafa2018guaranteed,
  title={Guaranteed {SLAM}---An interval approach},
  author={Mustafa, Mohamed and Stancu, Alexandru and Delanoue, Nicolas and Codres, Eduard},
  journal={Robotics and Autonomous Systems},
  volume={100},
  pages={160--170},
  year={2018},
  publisher={Elsevier}
}

@article{voges2021interval,
  title={Interval-based visual-lidar sensor fusion},
  author={Voges, Raphael and Wagner, Bernardo},
  journal={IEEE Robotics and Automation Letters},
  volume={6},
  number={2},
  pages={1304--1311},
  year={2021},
  publisher={IEEE}
}

@article{shaikewitz2025uncertainty,
  title={Uncertainty Quantification for Visual Object Pose Estimation},
  author={Shaikewitz, Lorenzo and Georgiou, Charis and Carlone, Luca},
  journal={arXiv preprint arXiv:2511.21666},
  year={2025}
}

@article{wang2018set,
  title={Set-membership approach and Kalman observer based on zonotopes for discrete-time descriptor systems},
  author={Wang, Ye and Puig, Vicen{\c{c}} and Cembrano, Gabriela},
  journal={Automatica},
  volume={93},
  pages={435--443},
  year={2018},
  publisher={Elsevier}
}

@article{belforte1990parameter,
  title={Parameter estimation algorithms for a set-membership description of uncertainty},
  author={Belforte, Gustavo and Bona, Basilio and Cerone, Vito},
  journal={Automatica},
  volume={26},
  number={5},
  pages={887--898},
  year={1990},
  publisher={Elsevier}
}

@article{jaulin1993set,
  title={Set inversion via interval analysis for nonlinear bounded-error estimation},
  author={Jaulin, Luc and Walter, Eric},
  journal={Automatica},
  volume={29},
  number={4},
  pages={1053--1064},
  year={1993},
  publisher={Elsevier}
}

@inproceedings{ehambram2022interval,
  title={Interval-based visual-inertial {LiDAR} {SLAM} with anchoring poses},
  author={Ehambram, Aaronkumar and Voges, Raphael and Brenner, Claus and Wagner, Bernardo},
  booktitle={Proceedings of International Conference on Robotics and Automation},
  pages={7589--7596},
  year={2022}
}

@article{straub2019replica,
  title={The replica dataset: A digital replica of indoor spaces},
  author={Straub, Julian and Whelan, Thomas and Ma, Lingni and Chen, Yufan and Wijmans, Erik and Green, Simon and Engel, Jakob J and Mur-Artal, Raul and Ren, Carl and Verma, Shobhit and others},
  journal={arXiv preprint arXiv:1906.05797},
  year={2019}
}

\end{document}